%% file: neurips_2026.tex
\documentclass{article}
\PassOptionsToPackage{numbers, compress}{natbib}
\usepackage[main, final]{neurips_2026}
\usepackage[utf8]{inputenc} 
\usepackage[T1]{fontenc}    
\usepackage{hyperref}       
\usepackage{url}            
\usepackage{booktabs}       
\usepackage{amsfonts}       
\usepackage{nicefrac}       
\usepackage{microtype}      
\usepackage{xcolor}         

\usepackage{rotating}
\usepackage{multirow}
\usepackage{caption}
\usepackage{subcaption}
\usepackage{hyperref}
\usepackage{amsmath}
\usepackage{amssymb}
\usepackage{mathtools}
\usepackage{amsthm}
\usepackage{xfrac}
\usepackage{mathtools}
\usepackage{wrapfig}
\usepackage{tikz}
\usepackage{algorithm}
\usepackage{algorithmic}

\DeclarePairedDelimiterX{\infdivx}[2]{(}{)}{%
  #1\parallel #2%
}

\usepackage{cancel}
\usepackage{adjustbox}

\theoremstyle{plain}
\newtheorem{theorem}{Theorem}[section]
\newtheorem{proposition}[theorem]{Proposition}
\newtheorem{lemma}[theorem]{Lemma}

\theoremstyle{definition}
\newtheorem{definition}[theorem]{Definition}

\theoremstyle{remark}
\newtheorem{remark}[theorem]{Remark}

\usepackage[american]{babel}
\usepackage{amsfonts}
\usepackage{graphicx}
\usepackage{url}
\usepackage{bm}
\usepackage{cancel}
\usepackage{comment}
\usepackage{natbib}
\usepackage{enumitem}
\usepackage{etoolbox}
\usepackage{subcaption}
\usepackage{pifont}
\newcommand{\cmark}{\textcolor{green!60!black}{\ding{51}}}
\newcommand{\xmark}{\textcolor{red!60!black}{\ding{55}}}
\BeforeBeginEnvironment{wrapfigure}{\setlength{\intextsep}{0pt}}

\DeclareMathOperator{\dep}{\not\! \perp\!\!\!\perp}

\newcommand{\E}{\mathbb{E}}
\newcommand{\Prob}{\mathbb{P}}
\newcommand{\R}{\mathbb{R}}
\newcommand{\1}{\mathbf{1}}

\newcommand{\cE}{\mathcal{E}}
\newcommand{\cF}{\mathcal{F}}

\newcommand{\cX}{\mathcal{X}}
\newcommand{\KL}{\mathrm{KL}}
\newcommand{\tr}{\operatorname{tr}}

\newcommand{\norm}[1]{\lVert #1 \rVert}
\newcommand{\ip}[2]{\langle #1, #2 \rangle}
\newcommand{\mtry}{m_{\mathrm{try}}}
\newcommand{\Nstr}{N_{\mathrm{str}}}
\newcommand{\Nest}{N_{\mathrm{est}}}

\providecommand{\E}{}\renewcommand{\E}{\mathbb{E}}
\providecommand{\Prob}{}\renewcommand{\Prob}{\mathbb{P}}
\providecommand{\R}{}\renewcommand{\R}{\mathbb{R}}
\providecommand{\1}{}\renewcommand{\1}{\mathbf{1}}
\providecommand{\cE}{}\renewcommand{\cE}{\mathcal{E}}
\providecommand{\cF}{}\renewcommand{\cF}{\mathcal{F}}
\providecommand{\cG}{}\renewcommand{\cG}{\mathcal{G}}
\providecommand{\cX}{}\renewcommand{\cX}{\mathcal{X}}
\providecommand{\KL}{}\renewcommand{\KL}{\mathrm{KL}}
\providecommand{\tr}{}\renewcommand{\tr}{\operatorname{tr}}
\providecommand{\Var}{}\renewcommand{\Var}{\operatorname{Var}}
\providecommand{\norm}[1]{}\renewcommand{\norm}[1]{\lVert #1 \rVert}
\providecommand{\ip}[2]{}\renewcommand{\ip}[2]{\langle #1, #2 \rangle}
\providecommand{\mtry}{}\renewcommand{\mtry}{m_{\mathrm{try}}}
\providecommand{\Nstr}{}\renewcommand{\Nstr}{N_{\mathrm{str}}}
\providecommand{\Nest}{}\renewcommand{\Nest}{N_{\mathrm{est}}}
\providecommand{\Creach}{}\renewcommand{\Creach}{C_{\mathrm{reach}}}
\providecommand{\Cblock}{}\renewcommand{\Cblock}{C_{\mathrm{block}}}
\providecommand{\Ctail}{}\renewcommand{\Ctail}{C_{\mathrm{tail}}}

\title{Minimax Rates and\\Spectral Distillation for Tree Ensembles}

\author{%
  Binh Duc Vu\\
  King's College London\\
  \texttt{binh.vu@kcl.ac.uk}
  \And
  David S. Watson\\
  King's College London\\
  \texttt{david.watson@kcl.ac.uk}
}

\begin{document}

\maketitle
\input{main}

\bibliographystyle{plainnat}
\small{
\bibliography{references} 
}
\input{appx2}

\input{checklist.tex}

\end{document}

%% file: main.tex
\begin{abstract}
  Tree ensembles such as random forests (RFs) and gradient boosting machines (GBMs) are among the most widely used supervised learners, yet their theoretical properties remain incompletely understood. We adopt a spectral perspective on these algorithms, with two main contributions. First, we establish minimax-optimal convergence for RF regression over a source class defined by a mixed partition kernel. The rate is governed by the kernel's eigenvalue decay and the partition complexity of its leading eigenspaces. 
  This novel result sheds new light on the inductive bias of RFs.
  Second, we exploit this spectral viewpoint to develop compression schemes for tree ensembles. For RFs, leading eigenfunctions of the kernel operator capture the dominant predictive directions; for GBMs, leading singular vectors of the smoother matrix play an analogous role. Learning nonlinear maps for these spectral representations yields distilled models that are orders of magnitude smaller than the originals while maintaining competitive predictive performance. Our methods compare favorably to state of the art algorithms for forest pruning and rule extraction, with applications to resource constrained computing.
\end{abstract}

\section{Introduction}\label{sec:intro}


Tree ensembles such as random forests (RFs) and gradient boosting machines (GBMs) are among the most popular and versatile supervised learning algorithms. Fast and easy to use, trees regularly attain state of the art results in tabular data tasks \citep{grinsztajn2022why, shwartz_tabular}, often beating out deep learning models that have proven more effective for image and language processing. 
Trees are especially good at adapting to sparse signals, making them a common choice in high-dimensional domains where parametric methods may struggle, such as bioinformatics \citep{li2022}, econometrics \citep{borup2023}, and engineering \citep{shoar2022}.

Despite their prevalence, tree ensembles remain poorly understood at a theoretical level. For RFs, most authors focus on idealized variants such as purely random forests \citep{breiman2000, genuer2012, biau_analysis, klusowski2021}, which make analysis more tractable but sacrifice the adaptive splitting that drives empirical performance. 
Finite sample generalization guarantees remain relatively rare. For GBMs, the situation is somewhat different. While boosting theory is relatively well developed \citep{schapire2012, Buhlmann01062003, bartlett2004, zhang_yu_2005}, the specific inductive bias of tree-based gradient boosting---and in particular, how ensemble size interacts with tree complexity---lacks a comparably sharp characterization.

We adopt a spectral perspective that addresses both algorithms. For RFs, we work in the kernel tradition \citep{breiman2000, davies_rf_kernel, scornet_kernel}, showing that the statistical rate is governed by the eigenvalue decay of the integral operator and how well spectral energy can be resolved via splits in the ambient feature space. Our main theoretical result is minimax optimality for an algorithm that closely approximates the original method of \citet{Breiman2001} over a source class of mixed partition kernels. The proof sheds light on how greedy splits inherently adapt to sparse signals.

Taking this spectral viewpoint, we observe that the principal eigenvectors of the RF kernel operator capture the most informative dimensions of the model. 
GBMs, by contrast, do not admit a kernel representation, but an analogous compression scheme is achievable via the singular value decomposition of its implicit smoother matrix \citep{curth2023_double_descent}. In both cases, models can be distilled by learning the nonlinear mapping from the input space to an appropriate spectral representation. Since a few leading modes capture most of the problem's variance, the resulting models maintain competitive performance while being significantly smaller and faster than the originals, thus enabling deployment on resource-constrained devices where full ensembles are infeasible. Crucially, while previous methods require users to set an accuracy tolerance---leaving the final model size unpredictable---our algorithm allows the size to be set explicitly. This provides a practical solution for deployment within TinyML settings with strict memory budgets \citep{lin2023tiny}, while remaining competitive on the empirical rate-distortion frontier.


The remainder of this paper is structured as follows. We review related work in Sect. \ref{sec:related}. Following a brief summary of our setup and assumptions in Sect. \ref{sec:setup}, we state our main theoretical result in Sect. \ref{sec:main}. 
This motivates a distillation scheme for tree ensembles, which we describe in Sect. \ref{sec:compression}.
Empirical results are presented in Sect. \ref{sec:exp}, including extensive benchmarks against state of the art methods. A brief discussion follows in Sect. \ref{sec:discussion}. Sect. \ref{sec:conclusion} concludes.

\section{Related Work}\label{sec:related}

\paragraph{RF Theory}
For good overviews of RF theory, see \citep{Biau2016, scornet_hooker_2025}. Early work in this area tended to prioritize asymptotic consistency results \citep{breiman2004, Meinshausen2006, biau_consistency, biau2010, genuer2012, Scornet2015}, often under simplifying assumptions. 
Convergence rates have only gained attention more recently, with risk bounds established for RF classification \citep{lorenzen_PACbayes, gao2020, gao2022} and individual regression trees \citep{klusowski_cart, tan2022, zheng2023, xu2026}.
RF regression has typically been studied in the sparse additive regime, which is difficult for trees to resolve, though promising results have been derived under a local gain assumption \citep{syrg2020, chi2022, mazumder2023}. 
Minimax convergence, a common optimality target in nonparametric regression, has proven challenging.
Suboptimal rates have been derived for non-adaptive \citep{klusowski2021} and adaptive RFs \citep{Klusowski2024}.
An online variant of the RF algorithm known as Mondrian forests \citep{mondrian_forest} are minimax optimal over a H\"older class \citep{mourtada2017, mourtada2020}, as are tessellation forests \citep{oreill2022, oreilly2024}, which use oblique splits.
However, the split mechanism in both cases is non-adaptive, and so these results do little to explain the performance of standard RFs.
Globally optimal trees attain the minimax rate over a modified Besov class \citep{xu2026}, but this has little bearing on greedy procedures.
The only minimax result we are aware of for any adaptive RF variant relies on a modified split criterion and leaf-wise regression \citep{Zhang2024}, two major deviations from Breiman's algorithm.

Our theoretical approach differs markedly from previous work in this area. Classical minimax results assume that the sampling distribution is dominated by some reference measure over the input space, giving rates that scale poorly with data dimensionality \citep{stone1982, tsybakov2008}. 
However, in high-dimensional settings, data are often concentrated on a low-dimensional manifold \citep{bengio2013}. Algorithms that can adapt to such sparse structures may converge far faster than theory would suggest. 
Building on this insight, several authors have taken a spectral approach to kernel regression, deriving novel minimax rates under a source condition and assumptions on eigenvalue decay \citep{caponetto2007, blanchard2020, rastogi2023, zhang2025}. We adapt this approach to the RF context, establishing the minimax optimality of an adaptive RF algorithm.

\paragraph{Ensemble compression}

Pruning methods for tree ensembles selectively remove splits or entire trees from the forest. 
This tradition goes back to the original CART algorithm, where truncating depth was proposed as a form of tree regularization \citep{breiman1984}. 
Numerous ensemble variants have been developed since \citep{Tsoumakas2009}, notably including Leaf Refinement with $L_1$ pruning (LRL1) \citep{busch2023}, which aims to drop redundant trees; and
ForestPrune (FP) \citep{liu2023}, which optimizes tree depth on a per-tree basis. 
Though FP generally compresses better than LRL1, it produces symmetric trees by default, which may be suboptimal in some settings.
Building on these methods, \citet{devos2025} propose Level-wise Optimization and Pruning (LOP), which traverses the ensemble by tree depth, optimizing a global $L_1$-regularized objective to simultaneously prune subtrees and update leaf predictions. These methods generally fix an error tolerance upfront and minimize model size subject to this budget.

Distillation approaches, by contrast, typically fix a target size upfront and minimize error subject to a compression rate budget. 
The method originated in deep learning, where knowledge distillation involves training a ``student'' network to mimic the soft labels of a larger ``teacher'' network \citep{hinton_distilling, hinton_distilling2}. 
The approach has been adapted to tree ensembles.
\citet{zhou2016sat} use the parent ensemble as an oracle to generate pseudo-samples and stabilize the decision boundaries of a single surrogate tree, while \citet{vidal2020born} develop a dynamic program to construct a minimal decision tree that perfectly reproduces the original ensemble's predictions. Such exact distillation of a large forest into a single tree is NP-hard, and quickly becomes infeasible as dimensionality increases.  

\begin{figure*}[t]
    \centering
    \includegraphics[width=\textwidth]{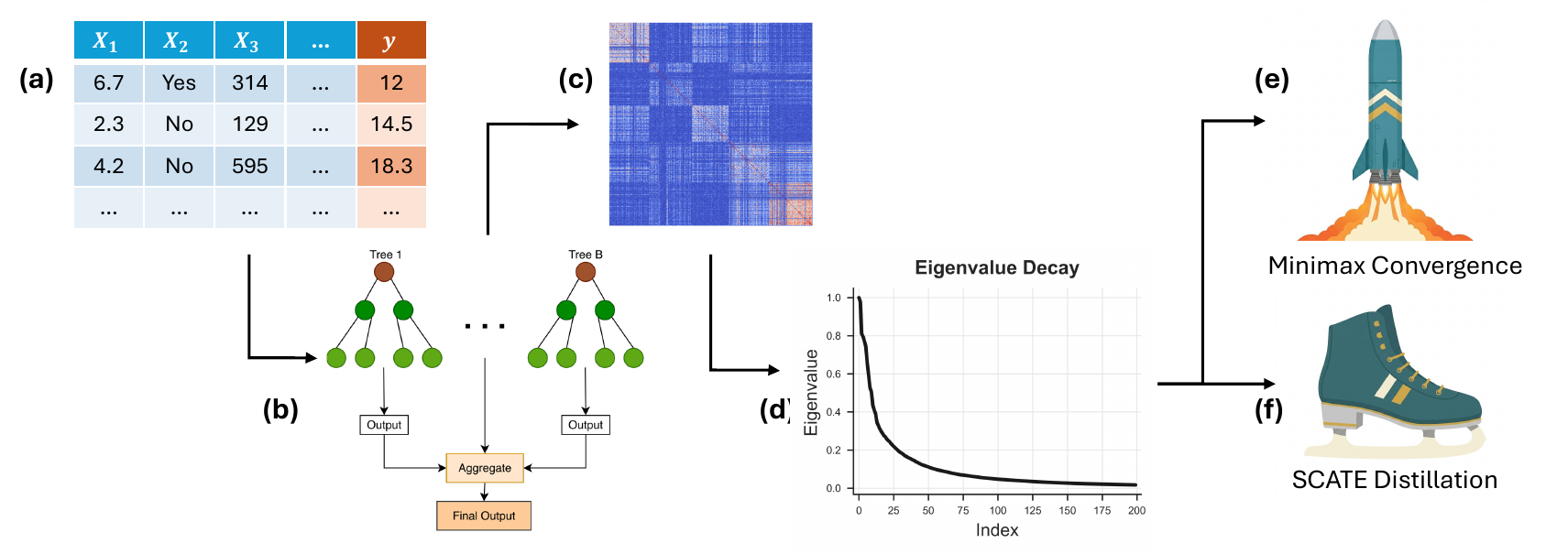}
    \caption{Summary of our pipeline. (a) Start with tabular training data for supervised learning (regression or binary classification). (b) Train a RF or GBM. (c) A kernel matrix $\mathbf{K}$ or smoother matrix $\mathbf{S}$ is extracted from the ensemble. (d) Compute the spectrum of the operator. (e) Under fast decay, RFs converge at the minimax-optimal rate. (f) The base ensemble is distilled by SCATE, which approximates model predictions by learning the top spectral directions.}
    \label{fig:pipeline}
\end{figure*}
A related but distinct line of work aims to produce smaller, more interpretable models via rule extraction. 
Prominent early examples include RuleFit \citep{friedman2008} and Node Harvest \citep{meinshausen2010}, which select the most informative splits via lasso penalties and constrained quadratic programming, respectively. 
More recent methods, such as SIRUS \citep{benard2021, benard2} and FIRE \citep{liu2023fire, liu2025extracting}, generally produce stabler results. 
The SSF algorithm extracts nodes into isolated decision stumps \citep{alkhoury2024splitting}, while MicroForest \citep{yoo2025microforest} and MILP-based alternatives \citep{bonasera2026} reassemble high information-gain nodes from multiple source forests into a new architecture.
Although these methods enable fast, memory-efficient prediction on constrained devices, they are generally optimized for interpretability over performance, as we confirm in our experiments.

\section{Background}\label{sec:setup}

Consider the standard regression setting, with data $\mathcal D_N = \{\bm x_i, y_i\}_{i=1}^N$ sampled i.i.d. from the joint probability measure $\mu$ over the feature space $\mathcal X \subseteq \mathbb R^d$ and label space $\mathcal Y \subseteq \mathbb R$. The goal is to estimate the conditional expectation $f^*(\bm x) = \mathbb E_\mu[Y \mid \bm x]$.

The CART algorithm \citep{breiman1984} approximates $f^*$ by learning a recursive partition of $\mathcal X$. Axis-aligned splits are selected on a greedy basis to minimize a loss function over the training set $\mathcal D_N$, typically mean squared error (MSE) for continuous $Y$. Splits continue until either samples are exhausted or some stopping criterion is met, e.g. target depth. Terminal nodes are labeled by taking leaf-wise averages of $Y$. This procedure is universally consistent under common regularity assumptions \citep{devroye1996, gyorfi2002}.
We distinguish between \textit{oracle} CART, which always picks the optimal split point at a given node, and \textit{empirical} CART, which may fail to do so due to sample noise.
Note that oracle CART is still a greedy algorithm, and will therefore generally fail to converge on the globally optimal binary tree. 

Despite their attractive asymptotic properties, trees can be high-variance predictors in practice. 
RFs overcome this by taking an ensemble of independently randomized CART basis functions, where each tree is trained on a bootstrap sample of the data and candidate splits are evaluated for only a subset of features. 
GBMs, on the other hand, are generally not randomized and grown to relatively shallow depth.
Most importantly, they are trained sequentially, with each tree improving on the performance of the last. 
Both models are \textit{smoothers} \citep{gams1990}, i.e. supervised learners of the form 
\begin{align}\label{eq:smooth}
    f(\bm x_0) = \sum_{i=1}^N ~w_i(\bm x_0) ~y_i,
\end{align}
where $\bm x_0 \in \mathcal X$ is an arbitrary point and $i \in [N]$ indexes the training data. Smoothers make predictions by taking linear combinations of training labels, with weights given by a potentially nonlinear smoothing function $w: \mathcal X \mapsto \mathbb R^N$. The differences between RFs and GBMs can be fully explained by their respective smoothing functions.

\paragraph{Random Forests}
In the case of RFs, the smoother is a kernel function. 
Each tree induces a partition of $\mathcal X$ into hyperrectangular cells $\{\mathcal X_\ell\}_{\ell=1}^M$, one for each leaf. 
Let $A(\cdot; b): \mathcal X \mapsto \{\mathcal X_\ell\}_{\ell=1}^M$ be a map from samples to cells for a given tree $b$.
The partition defines an indicator kernel that reports whether a pair of feature vectors are routed to the same leaf: $\hat k^{(b)}(\bm x, \bm x') = \bm 1 \big\{ A(\bm x; b) = A(\bm x'; b) \big\}$.
Normalizing over training data and averaging across trees, we derive the empirical RF kernel:
\begin{align*} 
    \hat k_N^B(\bm x, \bm x') = \frac{1}{B} \sum_{b=1}^B \Big( \frac{\hat k^{(b)}(\bm x, \bm x')}{\sum_{i=1}^N \hat k^{(b)}(\bm x, \bm x_i)} \Big).
\end{align*}
This kernel provides the weights of Eq. \ref{eq:smooth} for RFs. 
The resulting matrix $\hat{\bm K} \in [0,1]^{N \times N}$ is PSD, symmetric, and doubly stochastic. 
Because splits are adaptively selected to minimize a loss function, the kernel is not independent of outcomes $Y$. Still, the resulting RKHS is asymptotically universal and characteristic \citep{vu_RFAE}.
The empirical kernel $\hat{\bm K}$ approximates the population operator $\bm K$, defined w.r.t. infinite data and trees. 
We write $\nu$ to summarize the randomness of the tree-generating mechanism. 

\paragraph{Gradient Boosting}  

GBMs construct predictions by sequentially fitting trees to pseudo-residuals under MSE loss. 
Let $\bm f_b \in \mathbb R^N$ denote the vector of model predictions on training data at step $b$. Beginning with $\bm f_0 = \bm 0$, the $b$-th step fits a CART tree to the residual vector $\bm \varepsilon^{(b-1)} = \bm y - \bm f_{b-1}$, resulting in the normalized indicator kernel $\hat r^{(b)}(\bm x, \bm x')$.
Evaluating this kernel across all training pairs gives the single-tree smoother matrix $\hat{\bm R}^{(b)} \in [0,1]^{N \times N}$. 
The model updates according to $\bm f_b \gets \bm f_{b-1} + \eta \hat{\bm R}^{(b)} \bm \varepsilon^{(b-1)}$, with learning rate $\eta \in (0, 1]$.
Thus, beginning with $\hat{\bm{S}}^{(0)} = \bm{0}$, each round's prediction can be written as $\bm f_b= \hat{\bm S}^{(b)} \bm y$ with the accumulated linear operator given by $\hat{\bm S}^{(b)} \gets \hat{\bm S}^{(b-1)} + \eta \hat{\bm{R}}^{(b)}(\bm I - \hat{\bm S}^{(b-1)})$. Solving this recursion yields the final GBM smoother:
\begin{align*}
    \hat{\bm{S}} = \hat{\bm{S}}^{(B)} = \bm{I} - \prod_{b=1}^{B} (\bm{I} - \eta \hat{\bm{R}}^{(b)}).
\end{align*}
The corresponding smoothing function $\hat s^B_N: \mathcal X \mapsto \mathbb R^{N}$ provides the weights of Eq. \ref{eq:smooth} for GBMs.

Unlike the RF operator $\bm K$, the GBM smoother $\bm S$ is generally neither symmetric nor PSD. This precludes a direct eigendecomposition of the kind used for RFs. However, the singular value decomposition (SVD) $\bm S = \bm U \bm \Sigma \bm V^\top$ provides a natural spectral representation: the leading right singular vectors of $\bm S$ identify the directions in label space that the ensemble amplifies most, while the corresponding left singular vectors describe the associated structure in prediction space. Equivalently, we could take the eigendecomposition of the Gram matrix $\bm S \bm S^{\top}$, though there are computational advantages to using SVD instead. We detail pseudocode to construct $\bm S$ and $\bm K$ in Appx. \ref{appx:exp}


\section{Minimax Convergence}\label{sec:main}


In this section, we present a novel minimax convergence result for RF regression. To the best of our knowledge, this constitutes the first such result for a RF based on greedy CART splits. We start with a simple boundedness assumption.
\begin{itemize}
    \item[(A1)] \textit{Bounds}. The feature space is compact on the unit hypercube $\mathcal X = [0,1]^d$, with marginal density bounded away from $0$ and $\infty$, and outcomes on the unit interval $\mathcal Y = [0,1]$.
\end{itemize}
We make no assumption about the structural form of $Y \mid \bm X$, allowing for arbitrary heteroskedasticity.

Our next assumptions characterize the spectral source class. 
All $L^2$ norms and inner products from here on out are taken w.r.t. $\mu_X$ unless otherwise specified.
Let $\bm P_T$ be the leaf averaging operator, $\bm P_T f := \mathbb E_\mu[f(\bm x) \mid \sigma_T]$, where $\sigma_T$ denotes the $\sigma$-algebra associated with tree $T$, a recursive partition of axis-aligned splits.
Let $\Pi$ be a probability distribution over partitions.
(We use $b \in [B]$ to index trees grown by the RF, and $T \sim \Pi$ to denote a generic partition of $\mathcal X$ with $|T|$ leaves.)
\begin{itemize}[noitemsep]
    \item[(A2)] \textit{Reference operator}. 
    For fixed $\bm K_\Pi := \mathbb E_{T \sim \Pi} \big[\bm P_T\big]$, there exists some
    $\alpha \in (0,1]$, $\zeta \in (0, \infty)$, and $g \in L^2$ with $\lVert g \rVert_{L^2} \leq \zeta$ such that $f^* = \bm K^\alpha_\Pi g$ and $f^*(\bm x) \in [0,1] ~\mu_X$-a.e. 
\end{itemize}
Though the basic form of the ellipsoid defined by (A2) is standard for inverse problems \citep{caponetto2007, rastogi2023}, the definition of $\bm K_\Pi$ as a convex combination of partition projections is tailored to the RF setting. 
This operator is positive, self-adjoint, and contractive. 
Next, we impose additional spectral structure. Let $(\psi_i, \lambda_i)_{i \geq 1}$ denote the nonconstant eigenpairs of $\bm K_\Pi$, ordered by decreasing eigenvalue. 
\begin{itemize}[noitemsep]
    \item[(A3)] \textit{Polynomial decay}. The operator $\bm K_\Pi$ has eigenvalues $\lambda_i \asymp i^{-\beta}$, for some $\beta > 1$.
    \item[(A4)] \textit{Spectral localization}. For all $m \geq 1, \sup_{\omega \in \{-1,1\}^m} ~\Big\lVert \sum_{j=m+1}^{2m} \omega_{j-m} \psi_j \Big\rVert_\infty \lesssim m^{\alpha \beta + 1/2}$.
\end{itemize}
(A3) specifies a spectral decay rate widely assumed in the analysis of kernel ridge regression \citep{blanchard2020, zhang2025} and adaptive neural kernels \citep{li2024eigenvalue, bietti2019inductive, belfer2024spectral}.
(A4), which is only relevant for the minimax lower bound, ensures that bounded spectral perturbations remain valid regression functions. 
We write $\mathcal F := \mathcal F(\bm K_\Pi, \alpha, \zeta, \beta)$ to denote the class of functions defined by (A2)-(A4), and parameterize smoothness henceforth via $s := 2 \alpha \beta$.

We make the following regularity assumptions on tree growth:
\begin{itemize}[noitemsep]
    \item[(A5)] \textit{Subsampling}. Trees are trained on subsamples of size $N^{(b)} = aN$ for some $a \in (0, 1]$.
    \item[(A6)] \textit{Budget}. Splits are selected on a best-first basis until a predetermined leaf budget is spent.
    \item[(A7)] \textit{Tiebreak}. If every candidate split has zero gain, the algorithm splits the largest-mass leaf at the midpoint of a coordinate drawn uniformly at random.
\end{itemize}
(A5)-(A7) represent slight deviations from \citet{Breiman2001}'s classic algorithm, which trains on bootstraps rather than subsamples and grows to full depth by default. 
The subsampling assumption is standard in the literature \citep{mentch_hooker, Wager2018, Athey2019}, while growing to a leaf budget is common in practice, and widely implemented in standard RF software \citep{scikit-learn, ranger}. Zero-gain events are rare for generic signals, so (A7) is primarily a theoretical convention.

Our most substantive assumption links the quality of tree splits, optimized for MSE in feature space, to spectral gain. Let the operator $\bm G_m:= \sum_{i \leq m} \psi_i \otimes \psi_i$ project onto the leading $m$ eigenfunctions of $\bm K_\Pi$. 
Then spectral tree complexity can be quantified by:
\begin{align*}
    \mathfrak C_m(\varepsilon) := \sup_{\bm u \in \mathrm{ran}(\bm G_m), \lVert \bm u \rVert_2=1}~\inf \big\{|T|: \lVert (\bm I - \bm P_T) \bm u \rVert_{L_2}^2 \leq \varepsilon \big\}.
\end{align*}
This value represents the number of leaves that an optimal tree requires in the worst case to resolve spectral energy of order $m$ up to tolerance $\varepsilon$.
Let $V(T;f):=\norm{\bm (\bm I - \bm P_T) f}_{L_2}^2$ and $U_m(T;f):=\norm{\bm (\bm I - \bm P_T) \bm G_m f}_{L_2}^2$ represent unresolved energy of tree $T$ for function $f \in \mathcal F$ at full and $m$-truncated resolutions, respectively.
\begin{itemize}
    \item[(A8)] \textit{Amortized spectral gain}. Fix $\varepsilon\in (0,1)$.
    There exist constants $c \in (0,1)$ and $C > 0$ such that, on any linear reachability window $|T| \lesssim m$, for every $m \ge 1$ and $f \in \mathcal F$, some $h \lesssim \mathfrak C_m(\varepsilon)$ further CART splits satisfy
    \begin{equation}
        V(T;f)-\E_\nu\bigl[V(T_{+h}; f) \mid T \bigr]
        \geq c \,U_m(T;f) - C\,\norm{(\bm I-\bm G_m)f}_{L^2}^2.
        \tag{ASG}\label{eq:ASG}
    \end{equation}
\end{itemize}
(A8) says that, within $\mathfrak C_m(\varepsilon)$ splits, the forest captures a constant fraction of whatever leading $m$-dimensional spectral energy remains, up to the irreducible tail.
The condition may appear somewhat ad hoc at first glance, but ASG is in fact weaker than the structural impurity decrease assumption, which has recently become popular in RF theory \citep{chi2022, syrg2020, klusowski_cart, mazumder2023}; see Appx. \ref{sec:ex-sid} for a comparison.

With these elements in place, we state our main result. The proof is deferred to Appx. \ref{appx:proofs}. 
\begin{theorem}[Minimax optimality]
\label{thm:minimax}
Assume (A1)-(A8), with $s := 2\alpha\beta$ and $\mathfrak C_m(\varepsilon) = \mathcal O(m)$ for some $\varepsilon \in (0,1)$.
Then a RF $\hat f_{\mathrm{RF}}$ with oracle CART splits, $B=\infty$ trees, and leaf budget $M\asymp N^{1/(s+1)}$ satisfies
\[
  \sup_{f^\ast \in \mathcal F} ~\E \Big[\norm{\hat f_{\mathrm{RF}}-f^\ast}^2_{L^2} \Big]
  \ \lesssim\ N^{-\frac{s}{s+1}}.
\]
This rate is minimax optimal over the source class $\mathcal F$. Replacing oracle splits with empirical splits preserves the rate up to a logarithmic factor $\log(Nd)^{\frac{s}{s+1}}$ under an empirical variant of ASG with $M \asymp \big(N/\log(Nd)\big)^{1/(s+1)}$ leaves.
\end{theorem}

For a precise statement of our empirical ASG condition, see Appx. \ref{appx:empirical}. The infinite-tree idealization is benign, and can be removed at the cost of an additive $1/B$ term, a negligible offset for $B$ on the order of a few hundred. The more substantive condition is that ASG applies with $\mathfrak C_m = \mathcal O(m)$. This assumption is violated by some well-studied function classes, including the sparse additive model that has been the focus of much RF theory for over a decade \citep{Scornet2015, syrg2020, chi2022, Klusowski2024}. This accounts for the negative results of \citet{tan2022}, who show that even an optimal tree attains suboptimal rates in such settings (though their own experiments suggest faster rates for RFs).

However, there is a large class of natural multivariate regression models that fall within our minimax regime. 
In Appx.~\ref{appx:scope}, we study a hierarchically separated tree-Besov class, in which the target function may exhibit different sparse and anisotropic structure across cells of a recursive partition. 
This class is minimax-hard, with exponentially many independently varying fine-scale constrasts, but better suited to the inductive bias of RFs than additive models. 
Our construction can be viewed as a greedy-accessible analogue of the PSHAB spaces studied by \citet{xu2026}, whose recursive-tree approximation theory likewise falls in the linear-complexity regime. For a detailed discussion of related works and differing rates derived under varying assumptions, see Appx. \ref{appx:related}.


Thm.~\ref{thm:minimax} provides a conceptual bridge to spectral distillation. 
Under (A3), the regression signal is concentrated in the leading eigenspaces of a reference operator with rapidly decaying spectrum, while (A8) posits that greedy RFs can efficiently resolve this spectral geometry. 
This immediately suggests a compression scheme: if the induced kernel inherits the spectral profile of $\bm K_\Pi$, then $\lambda_i(\hat{\bm K})$ decays rapidly and a low-rank approximation retaining only its leading eigenfunctions should preserve most of the predictive signal. 
Our present theory does not explicitly establish such convergence, so we treat this as a theoretically motivated hypothesis rather than a consequence of Thm.~\ref{thm:minimax}. 
The distillation procedure developed next tests the hypothesis directly, using the empirical spectrum of the fitted RF to identify and retain its dominant predictive modes.

\section{Spectral Compression of Adaptive Tree Ensembles (SCATE)}\label{sec:compression}


We introduce \textbf{SCATE} (\textbf{S}pectral \textbf{C}ompression of \textbf{A}daptive \textbf{T}ree \textbf{E}nsembles). 
SCATE uses a (small) neural network to learn the leading spectral directions of an implicit linear operator---the kernel matrix $\hat{\bm{K}}$ for RFs, or the smoother matrix $\hat{\bm{S}}$ for GBMs. 
For an RF $\hat{f}_{\mathrm{RF}}$ trained on $N$ samples, let $(\bm{\Psi}, \bm{\Lambda} = \text{diag}(\bm{\lambda}))$ be the eigenpairs of the kernel matrix.\footnote{Though we generally omit hats on eigenvectors and eigenvalues for notational convenience, it is worth clarifying that $(\bm{\Psi}, \bm \Lambda)$ represent finite sample estimates, not population parameters. The same goes for the SVD of $\hat{\bm S}$.} 
Let $\psi: \mathcal{X} \mapsto \mathbb{R}^N$ be a map from input features to $\hat{\bm{K}}$'s bundled eigenfunctions, treated as a column vector.
Model predictions over the training data can be written $\hat{f}_{\mathrm{RF}}(\bm X) = \hat{\bm{K}}\bm{y}$, where $\bm{y}$ is a vector of training labels. For a test point $\bm{x}_0$, predictions rely on localized kernel weights $\hat{\bm{k}}_0 = \bm{\Psi\Lambda} \psi(\bm{x}_0)$, yielding:
\begin{equation*}
    \hat{f}_{\mathrm{RF}}(\bm{x}_0) = \hat{\bm{k}}_0^\top \bm{y} = (\bm{\Psi\Lambda}\psi(\bm{x}_0))^\top \bm{y} = \psi(\bm{x}_0)^\top \bm{\Lambda\Psi}^\top \bm{y}.
\end{equation*}
This derivation reveals our compression strategy: we store the eigendecomposition of $\hat{\bm{K}}$, truncated to rank $m \ll N$, and train a neural network to approximate the top $m$ eigenfunctions $\{\psi_i\}_{i=1}^m$. By defining a fixed coefficient vector $\bm{c} := \bm{\Lambda}_m \bm{\Psi}_m^\top \bm{y} \in \mathbb{R}^m$, test time predictions amount to evaluating $\hat{f}_{\mathrm{RF}}(\bm{x}_0) \approx \sum_{i=1}^m c_i ~\psi_i(\bm{x}_0)$.
By the Eckart-Young-Mirsky theorem \citep{eckart1936approximation}, this rank-$m$ truncation provides the optimal low-rank matrix approximation under the Frobenius norm, ensuring we retain maximal information from the ensemble's smoother matrix. 
An identical strategy works for GBMs, substituting $\hat{\bm S} = \bm{U \Sigma V}^\top$ for $\hat{\bm K} = \bm{\Psi \Lambda \Psi}^\top$ and learning the left singular vector function $u: \mathcal X \mapsto \mathbb R^m$. Note that, because we can compress any ensemble that generates a valid smoother matrix in accordance with Eq. \ref{eq:smooth}, SCATE can be applied to a model already pruned by a different algorithm, with an exact smoother when the leaf values are not altered by the process (e.g., FP) , and approximate otherwise (e.g., LOP).

\paragraph{Learning spectral directions with SCATE}
Extracting eigenfunction mappings from a given matrix is a well-studied problem. Existing neural approaches, such as SpIN \citep{pfau2018spectral} and NeuralEF \citep{deng2022neuralef}, rely heavily on constrained optimization, explicitly enforcing eigenvector orthonormality within the loss function to guide the network. However, constrained optimization methods are computationally heavy, making them difficult to tune and design in low-size settings. Alternatively, methods like NeuralSVD \citep{ryu2024operator} utilize unconstrained optimization by assigning greater weight to eigenfunctions with stronger signals. However, in practice, we find that these methods struggle to learn efficiently in constrained settings due to their complex objectives---see Appx. \ref{appx:neural_exp} for a small comparison.

To overcome these limitations, SCATE focuses on a simple architecture with a targeted loss formulation. The network is parameterized as a simple Multi-Layer Perceptron (MLP) with SiLU activations \citep{ELFWING20183}. SCATE is trained via a multi-task objective \citep{caruana97}, which minimizes a spectrum-weighted MSE alongside an orthogonality penalty, $\mathcal{L}_{\mathrm{ortho}}=m^{-2}\|\operatorname{offdiag}(b^{-1}\widehat{\Psi}^{\top}\widehat{\Psi})\|_F^2$, where $\widehat{\Psi}\in\mathbb{R}^{b\times m}$ contains the network's $m$ outputs on a minibatch of $b$ points.
\begin{equation*}
    \min_g \frac{1}{b}
    \sum_{i \in \mathcal{B}} \sum_{j=1}^m
    \frac{\lambda_j}{\lambda_1}
    \big( g_j(\bm x_i) - \psi_j(\bm x_i) \big)^2
    + \gamma \mathcal{L}_{\text{ortho}}.
\end{equation*}
We explicitly weigh the MSE by the corresponding squared eigenvalue $\lambda_j^2$---or singular value $\sigma_j^2$ for SVD---to encode greater interest in the most informative, high-variance spectral directions, and set $\gamma = 0.001$, which we found balanced mode diversity and reconstruction error well.

\paragraph{Oracle Benchmark}
To isolate our network's approximation error from the inherent loss of a low-rank bottleneck, we benchmark against an eigenfunction oracle. Given direct access to the train-test cross-kernel matrix $\bm{K}_{test}$, the oracle computes the optimal least-squares mapping $\bm{W}$ from the scaled training eigenvectors $\bm{C} = \bm{\Psi\Lambda}$, solving $\bm{W}\bm{C}^\top = \bm{K}_{test}$. Truncating to the top $m$ dimensions, the reconstructed test kernel is $\tilde{\bm{K}} = \bm{W}_m \bm{C}_m^\top$. The oracle then computes predictions via $\tilde{\bm{K}}\bm{y}$, and establishes a strict lower bound on the Frobenius reconstruction error $\|\bm{K}_{test} - \tilde{\bm{K}}\|_F$ achievable by any rank-$m$ distillation on the same spectral subspace.
Interestingly, this does not guarantee that $\tilde{\bm K}$ attains optimal performance in terms of MSE on the associated prediction task, although it is rarely far off in practice.

\section{Experiments}\label{sec:exp}
In this section, we evaluate SCATE's performance for both RFs and GBMs. Comprehensive details describing the experimental setup can be found in Appx. \ref{appx:exp}; code is available online.\footnote{\url{https://github.com/binhducvu/SCATE}}
We use consistent hyperparameters throughout, with \texttt{num\_trees=250}, \texttt{max\_depth=15} for RFs and \texttt{num\_trees=100}, \texttt{max\_depth=6} for GBMs. 

\begin{figure}[t]
    \centering
    \includegraphics[width=0.9\linewidth]{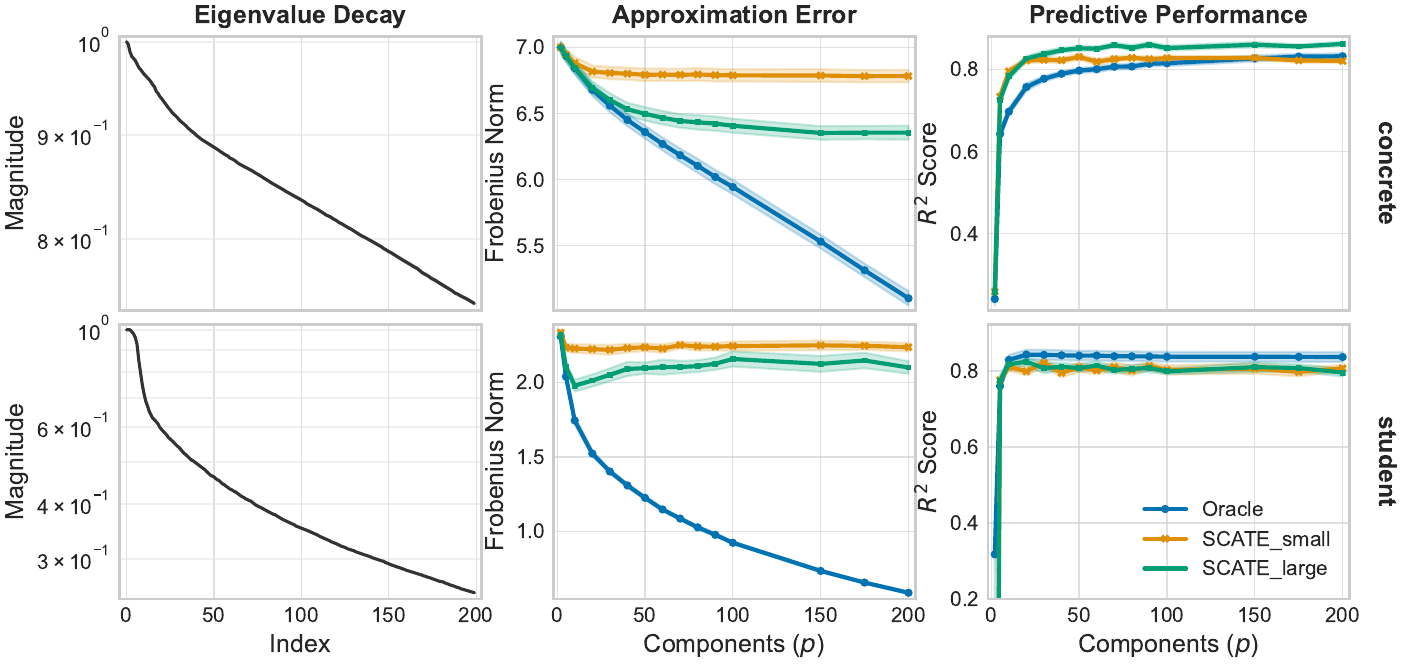}
    \caption{Spectral decay, kernel approximation error, and predictive performance across two datasets.}
    \label{fig:visualisation}
\end{figure}

\paragraph{Spectra and Generalization} Our strategy is motivated by the idea that the strong performance of RFs (and, by analogy, GBMs) is driven by their capacity to learn rich, informative embeddings encoded in the principal directions of their respective smoothing functions.
Spectral decay has implications not just for the convergence rate of an ensemble, as in Thm. \ref{thm:minimax}, but for distillation as well. To demonstrate this, Fig. \ref{fig:visualisation} illustrates the predictive performance of SCATE and the spectral oracle as a function of the number of eigenfunctions ($m$) for the RF case. We highlight the \texttt{student} and \texttt{concrete} datasets as representative examples of ``easy'' (rapid eigendecay) and ``hard'' (slower eigendecay) scenarios, respectively. Results for additional datasets are provided in the Appx \ref{appx:extra_vis}. We evaluate two different architectures for comparison---a small MLP of \texttt{width=16,depth=2} and a larger MLP of \texttt{width=128,depth=4}. While the larger network’s superior performance across both datasets confirms that increased expressivity is useful, the results also demonstrate a connection between compressibility and difficulty, where the model performance converges quicker for the dataset with faster decay.
Because SCATE's resource-bound architecture is limited to learning top components, its success depends on how much of the model’s total predictive power is concentrated within these leading eigenfunctions.

\paragraph{Na\"{i}ve Benchmark} To further validate our approach, we benchmark SCATE against three na\"{i}ve tree compression baselines: (1) a vanilla MLP mirroring SCATE's architecture, but trained to directly mimic the individual tree predictions of the RF (Na\"{i}ve Distillation); (2) an RF artificially constrained to various sizes via hyperparemeter tuning (Na\"{i}ve RF); and (3) the base performance of the original RF. Whereas in Fig. \ref{fig:visualisation}, we fixed model size and varied the number of eigenfunctions $m$, here we do the inverse. With fixed $m = 50$---a value we found to be generally robust, 
striking a reasonable balance between complexity (of the base ensemble) and learnability (for the compressed network)---we vary the model size to create a Pareto frontier of rate-distortion curves for various methods and datasets. The base RF provides a reference point, indicating an optimal performance threshold (albeit at high model size). 

We visualize results for three datasets in Fig. \ref{fig:naive} (additional results provided in Appx. \ref{appx:extra_naive}).
SCATE consistently dominates in this experiment, effectively matching the predictive performance of the base RF while using orders of magnitude less memory.
Na\"{i}ve RF approaches the same performance levels but at a far slower rate.
SCATE also significantly outperforms the Na\"{i}ve Distillation scheme, which struggles to approximate RF predictions using the comparatively impoverished signal provided by direct RF outputs. This aligns with a notable finding of \citet{hinton_distilling}, who initiated research into neural network distillation. They showed that small student networks can quickly attain strong performance in classification tasks when trained on softmax weights from a larger teacher network because these so-called ``soft labels'' provide far more information than the hard labels used to train the original model.
We find that eigenfunctions perform a similar role in forest distillation for regression, accelerating distillation with rich spectral signals. 


\paragraph{Compression Benchmark} 
For this experiment, we evaluate the size-performance trade-off of SCATE against a suite of established compression techniques, including pruning (FP and LOP), rule extraction (TreeExtract and FIRE), and approximate distillation (SSF). We also introduce a hybrid approach, FP-SCATE, which applies ForestPrune to a base model before further refinement with SCATE. We excluded other techniques like SIRUS, RuleFit, and LRL1, as prior studies suggest they are dominated by the methods we use here. 
We were unable to find an implementation of MicroForest. 
For all SCATE variants, we maintain $m=50$ eigenfunctions, consistent with our previous analysis.

\begin{figure}[t]
    \centering
    \includegraphics[width=0.9\linewidth]{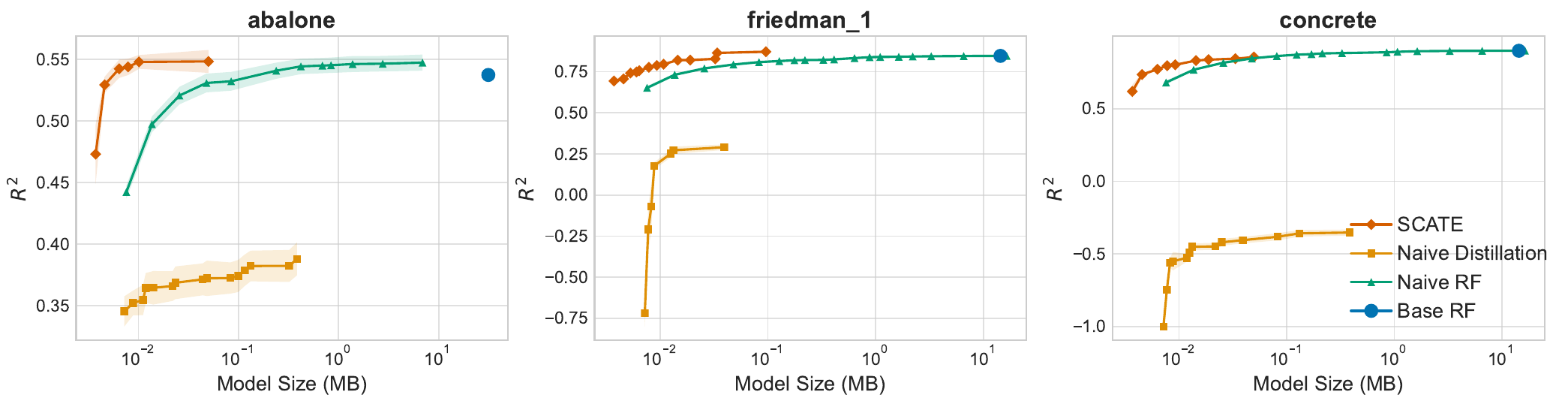}
    \caption{Rate-performance Pareto curves across three benchmark datasets. Shaded regions indicate standard errors across $10$ trials.}
    \label{fig:naive}
\end{figure}

We evaluate performance across several regression and classification datasets using $R^2$ and accuracy, respectively; further dataset details are provided in Appx. \ref{appx:exp}. 
Because the benchmarked algorithms utilize different hyperparameter mechanisms to control compression rates,
we adopt a size-constrained evaluation strategy. For each method and run, we sweep across a broad hyperparameter range to generate a performance-vs-size profile. 
We then report the best performance achieved at two fixed model sizes: 10KB and 100KB. These thresholds reflect the limited flash memory typical of microcontrollers, which generally operate with a budget under 1MB \citep{branco2019micro}. Original model sizes vary, but at 10KB, compression factors are on the order of $10^3$ for RFs and $10^2$ for GBMs --- see Appx. \ref{appx:big_table}. This experimental design is applied to both RF and GBM models, with the exception of SSF, which is restricted to RFs. 
Notably, while the RF kernel $\bm{K}$ generalizes naturally to both tasks, the GBM smoother $\bm{S}$ is defined specifically for MSE loss. To extend SCATE to GBM classification, we train the base model using MSE loss—effectively treating classification as a regression problem—which limits our evaluation to the binary case. 
As shown in Tables \ref{tab:rf} and \ref{tab:gbm}, SCATE yields highly competitive results for both RF and GBM compression, securing the highest average rank across all evaluated model sizes, with FP-SCATE coming in second for the RF benchmark. 
In contrast, this hybrid method declines markedly when applied to GBMs. We attribute this decline to the weaker baseline performance of FP as a preprocessing step, compounded by the fact that FP's pruning mechanism alters the ensemble structure in a way that violates our recursive formula for computing $\bm S$. For a more comprehensive results of algorithm performance, see Appx. \ref{appx:extra_results}. 

\begin{table}[t]
\centering
\caption{RF compression benchmark, averaged over up to $10$ trials. Best results per row and results within one standard error of the best are in bold, with values being $R^2$ for regression (top half) and accuracy for classification (bottom half). NA indicates no within-budget candidate; TO indicates no candidate result within time limit. For an expanded table with standard errors and run counts, see Table~\ref{tab:rf_results}.}
\label{tab:rf}
\setlength{\tabcolsep}{3pt}
\renewcommand{\arraystretch}{1.1}
\resizebox{\textwidth}{!}{
\begin{tabular}{lc|ccccccc|ccccccc}
\toprule
 & & \multicolumn{7}{c|}{\textbf{Size = 10KB}} & \multicolumn{7}{c}{\textbf{Size = 100KB}} \\
\cmidrule{3-16}
\textbf{Dataset} & \textbf{RF} & \textbf{FP} & \textbf{LOP} & \textbf{FIRE} & \textbf{TE} & \textbf{SSF} & \textbf{SCATE} & \textbf{FP-SCATE} & \textbf{FP} & \textbf{LOP} & \textbf{FIRE} & \textbf{TE} & \textbf{SSF} & \textbf{SCATE} & \textbf{FP-SCATE} \\
\midrule
\texttt{abalone} & 0.55 & 0.35 & 0.50 & 0.28 & NA & NA & \textbf{0.56} & 0.54 & 0.53 & 0.51 & 0.28 & 0.55 & NA & \textbf{0.57} & 0.55 \\
\texttt{auto\_mpg} & 0.87 & 0.63 & 0.81 & 0.68 & NA & NA & 0.85 & \textbf{0.87} & 0.82 & 0.82 & 0.68 & 0.84 & 0.81 & 0.87 & \textbf{0.88} \\
\texttt{bike} & 0.94 & 0.43 & NA & TO & TO & \textbf{0.70} & 0.68 & 0.68 & 0.72 & NA & TO & TO & 0.70 & \textbf{0.86} & 0.84 \\
\texttt{boston} & 0.85 & 0.60 & \textbf{0.80} & 0.68 & NA & NA & 0.78 & \textbf{0.79} & 0.80 & 0.82 & 0.69 & 0.81 & 0.72 & 0.81 & \textbf{0.83} \\
\texttt{california} & 0.80 & 0.45 & 0.68 & TO & TO & NA & 0.70 & \textbf{0.71} & 0.68 & \textbf{0.74} & TO & TO & NA & \textbf{0.75} & \textbf{0.74} \\
\texttt{concrete} & 0.91 & 0.50 & 0.74 & 0.73 & NA & NA & 0.80 & \textbf{0.82} & 0.76 & 0.85 & 0.75 & 0.89 & \textbf{0.90} & 0.86 & 0.86 \\
\texttt{diamonds} & 0.98 & 0.78 & NA & TO & TO & NA & \textbf{0.94} & 0.59 & 0.93 & NA & TO & TO & 0.93 & \textbf{0.96} & 0.60 \\
\texttt{friedman\_1} & 0.85 & 0.47 & 0.76 & 0.48 & NA & NA & \textbf{0.82} & 0.81 & 0.74 & 0.82 & 0.49 & 0.79 & NA & \textbf{0.86} & 0.86 \\
\texttt{student} & 0.85 & 0.64 & 0.80 & 0.71 & NA & 0.82 & \textbf{0.85} & 0.84 & 0.84 & 0.81 & 0.72 & 0.79 & 0.82 & \textbf{0.86} & 0.84 \\
\texttt{wq} & 0.50 & 0.24 & 0.31 & 0.31 & NA & NA & \textbf{0.34} & \textbf{0.33} & 0.35 & \textbf{0.36} & 0.34 & 0.28 & 0.24 & \textbf{0.36} & \textbf{0.36} \\
\midrule
\texttt{adult} & 0.85 & 0.79 & 0.83 & TO & TO & NA & \textbf{0.84} & 0.83 & 0.84 & 0.83 & TO & TO & \textbf{0.85} & \textbf{0.85} & 0.85 \\
\texttt{banknote} & 0.99 & 0.87 & 0.99 & 0.99 & NA & 1.00 & \textbf{1.00} & \textbf{1.00} & 0.92 & 0.99 & 0.99 & 0.89 & \textbf{1.00} & \textbf{1.00} & \textbf{1.00} \\
\texttt{bc} & 0.96 & 0.95 & 0.92 & 0.96 & NA & 0.98 & \textbf{0.98} & 0.98 & 0.95 & 0.92 & 0.96 & 0.92 & 0.98 & \textbf{0.99} & \textbf{0.99} \\
\texttt{churn} & 0.86 & 0.82 & 0.84 & 0.85 & NA & NA & \textbf{0.86} & 0.85 & 0.85 & 0.85 & 0.85 & 0.85 & 0.83 & \textbf{0.86} & 0.86 \\
\texttt{diabetes} & 0.76 & 0.76 & 0.74 & 0.77 & NA & NA & \textbf{0.79} & \textbf{0.79} & 0.76 & 0.74 & 0.77 & 0.66 & 0.73 & \textbf{0.79} & \textbf{0.80} \\
\texttt{magic} & 0.88 & 0.79 & 0.82 & \textbf{0.86} & NA & NA & \textbf{0.86} & \textbf{0.86} & 0.85 & 0.86 & 0.87 & 0.84 & NA & \textbf{0.87} & 0.87 \\
\texttt{spambase} & 0.95 & 0.87 & 0.92 & \textbf{0.94} & NA & NA & 0.93 & 0.93 & 0.91 & 0.93 & \textbf{0.94} & 0.89 & \textbf{0.94} & 0.94 & 0.94 \\
\texttt{telco} & 0.79 & 0.75 & 0.77 & 0.79 & NA & NA & \textbf{0.80} & 0.80 & 0.80 & 0.77 & 0.79 & 0.79 & 0.78 & \textbf{0.81} & 0.80 \\
\midrule
\textbf{Average Rank} &  & 4.67 & 3.86 & 4.08 & 6.44 & 5.50 & \textbf{1.53} & 1.92 & 4.22 & 4.72 & 5.31 & 5.19 & 4.67 & \textbf{1.56} & 2.33 \\
\bottomrule
\end{tabular}
}
\vspace{-3mm}
\end{table}
\begin{table}[htbp]
\centering
\caption{GBM compression benchmark, averaged over up to $10$ trials. Best results per row and results within one standard error of the best are in bold, with values being $R^2$ for regression (top half) and accuracy for classification (bottom half). NA indicates no within-budget candidate; TO indicates no candidate result within time limit. For an expanded table with standard errors and run counts, see Table~\ref{tab:gbm_results}.}
\label{tab:gbm}
\setlength{\tabcolsep}{3pt}
\renewcommand{\arraystretch}{1.1}
\resizebox{\textwidth}{!}{
\begin{tabular}{lc|cccccc|cccccc}
\toprule
& & \multicolumn{6}{c|}{\textbf{Size = 10KB}} & \multicolumn{6}{c}{\textbf{Size = 100KB}} \\
\cmidrule{3-14}
\textbf{Dataset} & \textbf{GBM} & \textbf{FP} & \textbf{LOP} & \textbf{TE} & \textbf{FIRE} & \textbf{SCATE} & \textbf{FP-SCATE} & \textbf{FP} & \textbf{LOP} & \textbf{TE} & \textbf{FIRE} & \textbf{SCATE} & \textbf{FP-SCATE} \\
\midrule
\texttt{abalone} & 0.53 & 0.44 & 0.50 & 0.54 & 0.46 & \textbf{0.56} & 0.55 & 0.53 & 0.52 & 0.55 & 0.46 & \textbf{0.57} & 0.55 \\
\texttt{auto\_mpg} & 0.86 & 0.80 & 0.82 & \textbf{0.85} & 0.82 & \textbf{0.85} & 0.83 & 0.86 & 0.83 & 0.85 & 0.82 & \textbf{0.87} & 0.85 \\
\texttt{bike} & 0.95 & 0.57 & NA & TO & \textbf{0.88} & 0.62 & 0.62 & 0.92 & \textbf{0.94} & TO & 0.88 & 0.92 & 0.87 \\
\texttt{boston} & 0.86 & 0.78 & \textbf{0.81} & \textbf{0.79} & 0.78 & \textbf{0.79} & 0.71 & \textbf{0.87} & 0.83 & 0.82 & 0.78 & 0.84 & 0.76 \\
\texttt{california} & 0.83 & 0.58 & 0.76 & \textbf{0.79} & 0.75 & 0.71 & 0.70 & 0.80 & \textbf{0.82} & 0.79 & 0.75 & 0.79 & 0.76 \\
\texttt{concrete} & 0.92 & 0.71 & 0.83 & \textbf{0.87} & 0.84 & 0.79 & 0.81 & 0.90 & 0.91 & \textbf{0.92} & 0.84 & 0.88 & 0.87 \\
\texttt{diamonds} & 0.98 & 0.88 & TO & NA & \textbf{0.94} & \textbf{0.94} & 0.93 & 0.97 & TO & \textbf{0.98} & 0.94 & 0.97 & 0.96 \\
\texttt{friedman\_1} & 0.91 & 0.68 & 0.87 & \textbf{0.92} & 0.79 & 0.89 & 0.87 & 0.90 & 0.90 & 0.92 & 0.79 & \textbf{0.97} & 0.93 \\
\texttt{student} & 0.83 & 0.81 & 0.79 & 0.79 & 0.75 & \textbf{0.84} & 0.70 & \textbf{0.86} & 0.79 & 0.79 & 0.75 & 0.84 & 0.70 \\
\texttt{wq} & 0.45 & 0.29 & 0.32 & 0.30 & \textbf{0.35} & 0.33 & 0.33 & \textbf{0.39} & \textbf{0.38} & 0.30 & 0.35 & 0.36 & 0.36 \\
\midrule
\texttt{adult} & 0.85 & 0.82 & 0.84 & \textbf{0.85} & 0.81 & 0.85 & 0.84 & 0.85 & 0.85 & \textbf{0.86} & 0.81 & 0.85 & 0.85 \\
\texttt{banknote} & 0.98 & 0.95 & 0.98 & 0.99 & 0.98 & \textbf{1.00} & 1.00 & 0.99 & 0.98 & 1.00 & 0.98 & \textbf{1.00} & 1.00 \\
\texttt{bc} & 0.94 & 0.94 & 0.94 & 0.94 & 0.95 & \textbf{0.98} & 0.95 & 0.95 & 0.94 & 0.95 & 0.95 & \textbf{0.99} & 0.95 \\
\texttt{churn} & 0.86 & 0.85 & 0.86 & \textbf{0.86} & 0.85 & 0.86 & 0.85 & \textbf{0.86} & 0.86 & \textbf{0.87} & 0.85 & \textbf{0.86} & 0.86 \\
\texttt{diabetes} & 0.74 & 0.77 & 0.76 & 0.78 & 0.74 & \textbf{0.79} & 0.73 & 0.78 & 0.76 & 0.78 & 0.74 & \textbf{0.80} & 0.73 \\
\texttt{magic} & 0.88 & 0.81 & 0.85 & \textbf{0.87} & 0.85 & 0.86 & 0.86 & 0.87 & 0.87 & \textbf{0.88} & 0.85 & 0.87 & 0.87 \\
\texttt{spambase} & 0.95 & 0.90 & 0.92 & \textbf{0.94} & 0.92 & 0.93 & 0.93 & 0.95 & 0.93 & \textbf{0.95} & 0.92 & 0.94 & 0.94 \\
\texttt{telco} & 0.79 & 0.79 & 0.79 & \textbf{0.80} & 0.79 & \textbf{0.81} & 0.79 & 0.80 & 0.79 & 0.81 & 0.79 & \textbf{0.81} & 0.80 \\
\midrule
\textbf{Average Rank} &  & 5.28 & 3.67 & 2.50 & 3.83 & \textbf{2.06} & 3.67 & 2.64 & 3.94 & 2.67 & 5.50 & \textbf{2.06} & 4.19 \\
\bottomrule
\end{tabular}
}
\vspace{-3mm}
\end{table}

\section{Discussion}\label{sec:discussion}

Our theoretical results suggest that generalization in tree ensembles is strongly linked to the compressibility of the model. The more information is encoded in the leading eigenfunctions of the RF or singular vectors of the GBM, the better the model learns, and the more easily it can be distilled. This accords with the minimum description length principle \citep{rissanen1989, barron1998mdl, hutter2004universal, grunwald2007} and related information theoretic insights \citep{tishby1999information, hellstrom2025} suggesting a profound connection between the dual tasks of prediction and compression. 

Notably, our minimax optimality theorem is limited to RF regression. A promising direction for future work would be to pursue similar results for GBMs. This poses substantial technical challenges, as kernel methods do not apply. The smoother matrix is neither symmetric nor PSD, and boosting induces nontrivial dependencies between basis functions. Still, recent work suggests some theoretical overlap in the mechanics of tree ensembles that may be exploited for future results \citep{zhou}.

Our compression method differs fundamentally from traditional techniques that learn directly from the base model. Instead, SCATE reproduces the leading spectral directions of a linear operator. 
This approach offers distinct advantages over classical pruning methods, which typically optimize the tree structure itself, making their computational cost heavily dependent on the base model's size. By contrast, SCATE relies solely on the kernel, which can be computed with a single forward pass of the training data. 
Scaling up the complexity of the base model therefore has lower impact on distillation time. 
Furthermore, SCATE compiles into a small network where inference amounts to computing fast matrix operations, maintaining the parallelization capabilities of its teacher model.

SCATE is not without its limitations. First, training imposes nontrivial computational demands: storing operators requires $\mathcal{O}(N^2)$ space, and truncated decomposition takes $\mathcal{O}(N^2m)$ time. However, alongside standard mitigations like randomized SVD \citep{Halko2011Finding} and Nystr\"om approximations \citep{WilliamsSeeger2001}, we find empirically that subsampling the training set to 10,000 observations produces no noticeable loss in distillation accuracy, effectively removing this bottleneck for large datasets. Second, deployment constraints restrict SCATE to using just the top $20$ to $50$ eigenfunctions. This may be insufficient for complex problems, leading SCATE's approximation quality to suffer—although our strong empirical results indicate this is not often the case in practice. 
While our framework allows us to stipulate the final size of the model (see Appx. \ref{appx:scate_size} for a basic derivation of model size given architecture), choosing the optimal network architecture is nontrivial. In our experiments, a simple MLP proved effective in the resource-constrained setting, but more sophisticated alternatives may provide improvements.
Finally, our distilled model sacrifices the purported interpretability benefits of trees, which naturally lend themselves to rule extraction and formal verification \citep{izza2021, ignatiev2022using}. 

Finally, we remark that the RF and GBM operators share a conceptual symmetry to two adaptive kernel interpretations of neural networks: the conjugate kernel (CK) \citep{lee2018deep} and the neural tangent kernel (NTK) \citep{jacot2018neural}. In particular, the RF kernel and the CK both construct a symmetric, static matrix by averaging over random projections. Conversely, the GBM smoother and NTK both instead track the learning trajectory of residual fitting during optimization. Adapting our results to deep learning architectures is a natural direction for future work.

\section{Conclusion} \label{sec:conclusion}

We have shown that tree ensembles can be compressed into significantly smaller models of comparable performance using spectral distillation. These results are motivated by a novel theoretical insight: that RFs are minimax-optimal w.r.t. a source class of mixed partition kernels when eigenvalues decay at a polynomial rate.
To our knowledge, this is the first minimax result for an RF algorithm composed of ordinary CART basis functions.
Our key assumption is amortized spectral gain, a weaker variant of the popular structural impurity decrease assumption.
We conjecture that a similar result holds for GBMs, when singular values of the smoother matrix decay quickly. Extensive benchmarks show that our methods are competitive with state of the art pruning and rule extraction techniques for ensemble compression, providing new methods for resource constrained computing. 

\section*{Acknowledgements}

We thank Erwan Scornet and several anonymous reviewers for their helpful comments on an earlier draft of this manuscript.
This research was supported by the UK Engineering and Physical Sciences Research Council
(EPSRC) [Grant reference number EP/Y035216/1] Centre for Doctoral Training in Data-Driven
Health (DRIVE-Health) at King’s College London. 

%% file: appx2.tex
\appendix
\normalsize

\providecommand{\E}{}\renewcommand{\E}{\mathbb{E}}
\providecommand{\Prob}{}\renewcommand{\Prob}{\mathbb{P}}
\providecommand{\R}{}\renewcommand{\R}{\mathbb{R}}
\providecommand{\1}{}\renewcommand{\1}{\mathbf{1}}
\providecommand{\cE}{}\renewcommand{\cE}{\mathcal{E}}
\providecommand{\cF}{}\renewcommand{\cF}{\mathcal{F}}
\providecommand{\cG}{}\renewcommand{\cG}{\mathcal{G}}
\providecommand{\cX}{}\renewcommand{\cX}{\mathcal{X}}
\providecommand{\KL}{}\renewcommand{\KL}{\mathrm{KL}}
\providecommand{\tr}{}\renewcommand{\tr}{\operatorname{tr}}
\providecommand{\Var}{}\renewcommand{\Var}{\operatorname{Var}}
\providecommand{\norm}[1]{}\renewcommand{\norm}[1]{\lVert #1 \rVert}
\providecommand{\ip}[2]{}\renewcommand{\ip}[2]{\langle #1, #2 \rangle}
\providecommand{\mtry}{}\renewcommand{\mtry}{m_{\mathrm{try}}}
\providecommand{\Nstr}{}\renewcommand{\Nstr}{N_{\mathrm{str}}}
\providecommand{\Nest}{}\renewcommand{\Nest}{N_{\mathrm{est}}}
\providecommand{\Creach}{}\renewcommand{\Creach}{C_{\mathrm{reach}}}
\providecommand{\Cblock}{}\renewcommand{\Cblock}{C_{\mathrm{block}}}
\providecommand{\Ctail}{}\renewcommand{\Ctail}{C_{\mathrm{tail}}}
\providecommand{\cH}{}\renewcommand{\cH}{\mathcal{H}}
\providecommand{\cX}{}\renewcommand{\cX}{\mathcal{X}}
\providecommand{\cC}{}\renewcommand{\cC}{\mathfrak{C}}
\providecommand{\norm}[1]{}\renewcommand{\norm}[1]{\lVert #1 \rVert}
\providecommand{\ip}[2]{}\renewcommand{\ip}[2]{\langle #1, #2 \rangle}
\providecommand{\dep}{}\renewcommand{\dep}{\operatorname{dep}}
\providecommand{\Creach}{}\renewcommand{\Creach}{C_{\mathrm{reach}}}

\section{Proofs}\label{appx:proofs}

We restate our assumptions and Thm.~\ref{thm:minimax} for completeness. Throughout, let $\bm G_m$ denote the orthogonal projector in $L^2(\mu_X)$ onto the constant function $1$ together with the first $m$ nonconstant eigenfunctions $\{\psi_1,\dots,\psi_m\}$ of $\bm K_\Pi$, ordered by decreasing eigenvalue.

\begin{itemize}
    \item[(A1)] \textit{Bounds}. The feature space is compact on the unit hypercube $\mathcal X = [0,1]^d$, with marginal density $\mu_X$ bounded away from $0$ and $\infty$, and outcomes on the unit interval $\mathcal Y = [0,1]$.

    \item[(A2)] \textit{Reference operator}. 
    For fixed $\bm K_\Pi := \mathbb E_{T \sim \Pi} \big[\bm P_T\big]$, there exist
    $\alpha \in (0,1]$, $\zeta \in (0, \infty)$, and $g \in L^2(\mu_X)$ with $\lVert g \rVert_{L^2} \leq \zeta$ such that $f^\ast = \bm K^\alpha_\Pi g$ and $f^\ast(\bm x) \in [0,1] ~\mu_X$-a.e. 

    \item[(A3)] \textit{Polynomial decay}. The operator $\bm K_\Pi$ has nonconstant eigenvalues $\lambda_i \asymp i^{-\beta}$, for some $\beta > 1$. We write $s := 2\alpha\beta$ for the effective spectral smoothness exponent.

    \item[(A4)] \textit{Spectral localization}. For all $m \geq 1$,
    \[
        \sup_{\omega \in \{-1,1\}^m} ~\Big\lVert \sum_{j=m+1}^{2m} \omega_{j-m} \psi_j \Big\rVert_\infty \lesssim m^{\alpha \beta + 1/2}.
    \]

    \item[(A5)] \textit{Subsampling}. Trees are trained on subsamples of size $N^{(b)} = aN$ for some fixed $a \in (0, 1]$.

    \item[(A6)] \textit{Budget}. Splits are selected on a best-first basis until a predetermined leaf budget is spent.

    \item[(A7)] \textit{Tiebreak}. If every candidate split has zero gain, the algorithm splits the largest-mass leaf at the midpoint of a coordinate drawn uniformly at random.

    \item[(A8)] \textit{Amortized spectral gain}. Fix $\varepsilon\in (0,1)$.
    There exist constants $c \in (0,1)$ and $C > 0$ such that, on any linear reachability window $|T| \lesssim m$, for every $m \ge 1$ and $f \in \mathcal F$, some $h \lesssim \mathfrak C_m(\varepsilon)$ further CART splits satisfy
    \begin{equation}
        V(T;f)-\E_\nu\bigl[V(T_{+h}; f) \mid T \bigr]
        \geq c \,U_m(T;f) - C\,\norm{(\bm I-\bm G_m)f}_{L^2}^2.
        \tag{ASG}
    \end{equation}

\end{itemize}

\textbf{Theorem 4.1} (Minimax optimality). Assume (A1)-(A8), with $s := 2\alpha\beta$ and $\mathfrak C_m(\varepsilon) = \mathcal O(m)$ for some $\varepsilon \in (0,1)$.
Then a RF $\hat f_{\mathrm{RF}}$ with oracle CART splits, $B=\infty$ trees, and leaf budget $M\asymp N^{1/(s+1)}$ satisfies
\[
  \sup_{f^\ast \in \mathcal F} ~\E \Big[\norm{\hat f_{\mathrm{RF}}-f^\ast}^2_{L^2} \Big]
  \ \lesssim\ N^{-\frac{s}{s+1}}.
\]
This rate is minimax optimal over the source class $\mathcal F$. Replacing oracle splits with empirical splits preserves the rate up to a logarithmic factor $\log(Nd)^{\frac{s}{s+1}}$ under an empirical variant of ASG with $M \asymp \big(N/\log(Nd)\big)^{1/(s+1)}$ leaves.

Thm.~\ref{thm:minimax} makes three separate claims: (1) RFs grown with oracle CART splits attain the stated upper bound $N^{-s/(s+1)}$; (2) the minimax risk over $\mathcal F$ is of exact order $N^{-s/(s+1)}$; and (3) replacing oracle splits with empirical splits preserves the rate up to a logarithmic factor under \eqref{eq:eASG} and an adjusted leaf budget.
We prove each claim in turn.

\subsection{Population RF upper bound}\label{appx:population-upper}

The population argument has two ingredients: the source condition gives a
spectral approximation floor of order $m^{-s}$, while ASG shows that oracle
CART reaches that floor using order $m$ leaves in the linear-complexity regime.

\begin{lemma}[Spectral tail]
\label{lem:tail}
For every $f^\ast=\bm K_\Pi^\alpha g\in\mathcal F$ and every $m\ge1$,
\[
  \norm{(\bm I-\bm G_m)f^\ast}_{L^2}^2
  \le
  \zeta^2\lambda_{m+1}^{2\alpha}
  \lesssim
  m^{-s}.
\]
\end{lemma}

\begin{proof}
Expand the nonconstant component of $g$ in the eigenbasis of $\bm K_\Pi$.
Then
\[
  (\bm I-\bm G_m)f^\ast
  =
  \sum_{j>m}\lambda_j^\alpha
  \ip{g}{\psi_j}\psi_j,
\]
so
\[
  \norm{(\bm I-\bm G_m)f^\ast}_{L^2}^2
  \le
  \lambda_{m+1}^{2\alpha}\norm{g}_{L^2}^2
  \le
  \zeta^2\lambda_{m+1}^{2\alpha}.
\]
Apply (A3) with $s=2\alpha\beta$.
\end{proof}

\begin{lemma}[Block recurrence with slack]
\label{lem:block}
Fix $m\ge1$ and let $T$ be any reachable state. Suppose that, for some
$\sigma_m\ge0$, a CART block taking $T$ to $T'=T_{+h}$ in
$h\lesssim \mathfrak C_m(\varepsilon)$ further splits satisfies
\begin{equation}
  V(T;f^\ast)-\E_\nu[V(T';f^\ast)\mid T]
  \ge
  c\,U_m(T;f^\ast)
  -C\norm{(\bm I-\bm G_m)f^\ast}_{L^2}^2
  -\sigma_m.
  \label{eq:asg-slack}
\end{equation}
Then
\begin{equation}
  \E_\nu[V(T';f^\ast)\mid T]
  \le
  q\,V(T;f^\ast)
  +C_1m^{-s}
  +\sigma_m,
  \qquad
  q := 1-\frac{c}{2} < 1,
  \label{eq:block-slack}
\end{equation}
where $C_1<\infty$ is uniform in $m, f^\ast, T$.
In particular, under population ASG \eqref{eq:ASG} this holds with $\sigma_m=0$.
\end{lemma}

\begin{proposition}[Multiscale CART approximation]
\label{prop:population-approx}
Assume $\mathfrak C_m(\varepsilon) = \mathcal O(m)$.
Suppose that the block recurrence \eqref{eq:block-slack} holds at every
resolution $m\ge 1$ on any linear reachability window $|T|\lesssim m$, with a nondecreasing sequence $\sigma_m\ge0$ and $h\lesssim\mathfrak C_m(\varepsilon)$ further splits per block.
Then for every $m\ge1$ there is a deterministic leaf budget $M_m\lesssim m$
such that
\[
    \sup_{f^\ast\in\mathcal F}
    ~\E_\nu\!\Big[
        \norm{(\bm I-\bm P_{T_{M_m}})f^\ast}_{L^2}^2
    \Big]
    \lesssim
    m^{-s}+\sigma_m.
\]
In particular, population ASG gives $\sigma_m=0$ and hence the approximation
bound $\mathcal O(m^{-s})$.
\end{proposition}

\begin{proof}
Let $m_\ell=2^\ell$, $\ell=0,\ldots,L$, with $m_L\asymp m$, and write
$q:=1-c/2<1$. We show inductively that, at the end of stage $\ell$,
\begin{equation}
    \sup_{f^\ast\in\mathcal F}~
    \E_\nu[V(T;f^\ast)]
    \le
    A\bigl(m_\ell^{-s}+\sigma_{m_\ell}\bigr)
    \label{eq:multiscale-induction}
\end{equation}
for a constant $A$ independent of $\ell$.

Since $0\le f^\ast\le1$, we have $V(T; f^\ast) \leq \lVert f^\ast \rVert_{L^2}^2 \leq 1$, so for $m_0=1$ the claim holds after enlarging $A$ if necessary.
Suppose \eqref{eq:multiscale-induction} holds at stage $\ell-1$.
Since $m_{\ell-1}^{-s}=2^sm_\ell^{-s}$ and $\sigma_m$ is nondecreasing,
\[
    \E_\nu[V(T;f^\ast)]
    \le
    A\bigl(2^sm_\ell^{-s}+\sigma_{m_\ell}\bigr)
\]
at the beginning of stage $\ell$.
Fix the integer $k=k(s,c)$ so that $q^k2^s<1$.
At stage $\ell$, we apply Lemma~\ref{lem:block} $k$ consecutive times at resolution $m_\ell$.
By induction on the stages and the $k$ blocks within stage $\ell$, each invocation starts from a tree of size $|T| \lesssim \sum_{j\le\ell} m_j \lesssim m_\ell$ and uses $h \lesssim \mathfrak C_{m_\ell}(\varepsilon) \lesssim m_\ell$ splits, so every block remains within a fixed linear reachability window $|T|\lesssim m_\ell$ and the total leaf budget satisfies
\[
    M_m
    \le
    1 + \sum_{\ell\le L} \mathcal O(m_\ell)
    \lesssim m.
\]
After the $k$ blocks at stage $\ell$, Lemma~\ref{lem:block} gives
\begin{align*}
    \E_\nu[V(T;f^\ast)]
    \le
    q^kA\bigl(2^sm_\ell^{-s}+\sigma_{m_\ell}\bigr)
    +
    \bigl(C_1m_\ell^{-s}+\sigma_{m_\ell}\bigr)
    \sum_{j=0}^{k-1}q^j.
\end{align*}
Since $q^k2^s<1$ and $q^k<1$, choosing $A$ sufficiently large absorbs both terms on the
right-hand side and closes the induction.

Finally, if the target approximation is reached before the deterministic budget
is exhausted, continue splitting arbitrarily. Recursive refinement cannot
increase $V(T;f^\ast)$, so the same bound holds at the predetermined budget
$M_m$.
\end{proof}

\begin{lemma}[Leaf estimation on an oracle partition]
\label{lem:oracle-regressogram}
Let $T$ be a recursive partition independent of the sample used to estimate
its leaf values, with $|T|\le M$. Let $\hat f_T$ be the usual leaf-mean
regressogram based on $n$ observations with bounded response $Y\in[0,1]$.
Then
\[
  \E \Big[\norm{\hat f_T-f^\ast}_{L^2}^2 \Big]
  \lesssim
  \norm{(\bm I-\bm P_T)f^\ast}_{L^2}^2+\frac{M}{n}.
\]
\end{lemma}

\begin{proof}
Conditional expectation onto the leaf $\sigma$-algebra gives the orthogonal bias term \mbox{$\norm{(\bm I-\bm P_T)f^\ast}_{L^2}^2$}. 
For a leaf $A\in T$, write $p_A:=\mu_X(A)$ and let $N_A$ denote its
sample count. Conditional on $N_A>0$, boundedness of $Y$ gives
\[
  \E\!\left[(\hat\mu_A-\mu_A)^2\mid N_A\right]
  \le \frac{1}{4N_A}.
\]
Since $N_A\sim{\rm Binomial}(n,p_A)$,
\[
  p_A\,\E\!\left[\frac{\mathbf 1\{N_A>0\}}{N_A}\right]
  \lesssim \frac1n.
\]
If $N_A=0$, any fixed fallback in $[0,1]$ has squared error at most one,
and its integrated contribution is bounded by
\[
  p_A\Prob(N_A=0)
  =
  p_A(1-p_A)^n
  \le \frac{1}{en}.
\]
Thus each leaf contributes $\mathcal O(n^{-1})$ to the estimation error.
Summing over at most $M$ leaves and adding the orthogonal projection bias proves the result.
\end{proof}

We now prove our first claim.

\textbf{Claim 1}. Fix $\varepsilon \in (0,1)$ and assume (A1)--(A8) with $\mathfrak C_m(\varepsilon) = \mathcal O(m)$. Then a RF $\hat f_{\mathrm{RF}}$ with oracle CART splits, $B=\infty$ trees, and leaf budget $M\asymp N^{1/(s+1)}$ satisfies
\[
  \sup_{f^\ast \in \mathcal F} ~\E \Big[\norm{\hat f_{\mathrm{RF}}-f^\ast}^2_{L^2} \Big]
  \ \lesssim\ N^{-\frac{s}{s+1}}.
\]

\begin{proof}
Take $m\asymp N^{1/(s+1)}$ and the corresponding deterministic budget
$M_m\lesssim m$ from Prop.~\ref{prop:population-approx}. Because recursive refinement cannot increase the population projection error $V(T;f^\ast)$, we may pad the tree after the multiscale blocks finish by arbitrary recursive splits to reach any predetermined leaf budget $M \asymp m \asymp N^{1/(s+1)}$. Under (A5), each tree uses $N^{(b)}=aN$ observations, so
Lemma~\ref{lem:oracle-regressogram} and Prop.~\ref{prop:population-approx}
give
\[
  \sup_{f^\ast\in\mathcal F}
  \E\Big[\norm{\hat f_{T_{M}}-f^\ast}_{L^2}^2\Big]
  \lesssim
  m^{-s}+\frac{M}{N}
  \asymp
  m^{-s}+\frac{m}{N}.
\]
For $B=\infty$, Jensen's inequality shows that averaging over tree
randomization cannot increase squared $L^2$ risk. Therefore
\[
  \sup_{f^\ast\in\mathcal F}
  \E\Big[\norm{\hat f_{\mathrm{RF}}-f^\ast}_{L^2}^2\Big]
  \lesssim
  m^{-s}+\frac{m}{N}.
\]
Balancing the two terms at $m\asymp N^{1/(s+1)}$ yields
\[
  \sup_{f^\ast\in\mathcal F}
  \E\Big[\norm{\hat f_{\mathrm{RF}}-f^\ast}_{L^2}^2\Big]
  \lesssim
  N^{-s/(s+1)}.
\]
\end{proof}

\begin{remark}[Relation to classical tree rates]
\label{rem:classical-tree-rates}
The dimension-free exponent in Claim~1 is a consequence of the
linear spectral tree-complexity regime. More generally, if resolution $m$
incurs approximation error $m^{-s}$ but requires
$\mathfrak C_m(\varepsilon)\asymp m^\gamma$ leaves, the same bias--variance
calculation (formalized in Prop.~\ref{prop:gamma}) gives
\[
    \E\Big[\norm{\hat f_{\mathrm{RF}}-f^\ast}_{L^2}^2\Big]\lesssim m^{-s}+\frac{m^\gamma}{N},
\]
and hence the rate $N^{-s/(s+\gamma)}$.
This recovers a familiar non-spectral phenomenon for additive regression.
For a model with $d_0 \leq d$ active $C^1$ univariate components, resolution $r$
gives squared approximation error of order $r^{-2}$, while a single
leaf-average axis-aligned partition may require order $r^{d_0}$ cells to
resolve all coordinate directions simultaneously. Balancing
$r^{-2}+r^{d_0}/N$ yields $N^{-2/(d_0+2)}$,
the sharp tree rate obtained by \citet{tan2022} for honest leaf-average trees
on $C^1$ additive models (corresponding to $s=2$ and $\gamma=d_0$; see Appx.~\ref{sec:ex-general} and Appx.~\ref{appx:related} for further discussion).
\end{remark}

\subsection{Minimax optimality}
\label{appx:minimax-lower}

Because population CART splits depend on the unknown target $f^\ast$, the oracle forest $\hat f_{\mathrm{RF}}$ in Claim~1 is not an ordinary data-only estimator. To establish that $N^{-s/(s+1)}$ is the exact minimax rate over $\mathcal F$, we combine a standard Assouad hypercube lower bound \citep{tsybakov2008} using (A2)--(A4) with a simple data-driven spectral projection upper bound.

\begin{proposition}[Spectral Assouad packing]
\label{prop:assouad-packing}
Under (A2)--(A4), there are constants $a_0>0$ and $\theta\in(0,1/2]$,
depending only on the fixed class parameters, such that for every $m$ the
functions
\[
  f_\omega
  :=
  \theta
  +
  a_m\sum_{j=m+1}^{2m}\omega_{j-m}\psi_j,
  \qquad
  \omega\in\{-1,1\}^m,
  \qquad
  a_m:=a_0m^{-(\alpha\beta+1/2)},
\]
belong to $\mathcal F$ and take values in a fixed interval
$[\delta,1-\delta]\subset(0,1)$ independent of $m$.
\end{proposition}

\begin{proof}
Because $\bm P_T1=1$ for every $T$, one has $\bm K_\Pi1=1$ and hence
$\bm K_\Pi^\alpha1=1$. Choose
\[
  \theta:=\min\left\{\frac14,\frac{\zeta}{4}\right\}.
\]
The source preimage of $f_\omega$ is
\[
  g_\omega
  =
  \theta
  +
  a_m\sum_{j=m+1}^{2m}
  \lambda_j^{-\alpha}\omega_{j-m}\psi_j.
\]
By orthogonality and (A3),
\[
\begin{aligned}
  \norm{g_\omega}_{L^2}^2
  &=
  \theta^2
  +
  a_m^2\sum_{j=m+1}^{2m}\lambda_j^{-2\alpha}\\
  &\lesssim
  \theta^2
  +
  a_0^2
  m^{-(2\alpha\beta+1)}
  \cdot m\cdot m^{2\alpha\beta}\\
  &\le
  \theta^2+C a_0^2.
\end{aligned}
\]
Choose $a_0$ sufficiently small that this is at most $\zeta^2$.
By (A4),
\[
  \left\|
  a_m\sum_{j=m+1}^{2m}\omega_{j-m}\psi_j
  \right\|_\infty
  \lesssim
  a_0.
\]
Reducing $a_0$ once more makes this quantity at most $\theta/2$. Hence every
$f_\omega$ lies in $[\theta/2,3\theta/2]\subset(0,1)$, uniformly in $m$.
Take $\delta:=\theta/2$.
\end{proof}

We may now prove our second claim.

\textbf{Claim 2}. Assume (A1)--(A4). Then the minimax rate over the source class $\mathcal F$ is given by
\begin{align*}
    \inf_{\tilde f}
    ~\sup_{f^\ast\in\mathcal F}
    ~\E\Big[\norm{\tilde f-f^\ast}_{L^2}^2\Big]
    \asymp
    N^{-s/(s+1)}.
\end{align*}

\begin{proof}
For the lower bound, restrict attention to the subexperiment
\[
  Y\mid \bm X= \bm x\sim\mathrm{Bernoulli}(f_\omega(\bm x)),
  \qquad
  \omega\in\{-1,1\}^m,
\]
with $f_\omega$ as in Prop.~\ref{prop:assouad-packing}. For neighboring
vertices $\omega,\omega'$ differing in one coordinate,
\[
  \norm{f_\omega-f_{\omega'}}_{L^2}^2
  =
  4a_m^2.
\]
Since all success probabilities lie in the fixed interval
$[\delta,1-\delta]$, the Bernoulli KL divergence is quadratically bounded, so
\[
  \KL\!\left(
  P_\omega^{\otimes N}\,\Vert\,P_{\omega'}^{\otimes N}
  \right)
  \lesssim
  N a_m^2.
\]
Choose $m\asymp N^{1/(s+1)}$. Because
$a_m^2\asymp m^{-(2\alpha\beta+1)}=m^{-(s+1)}$, neighboring KL divergences are
then bounded by an absolute constant. Assouad's lemma therefore gives
\[
  \inf_{\tilde f}
  \sup_{f^\ast\in\mathcal F}
  \E \Big[\norm{\tilde f-f^\ast}_{L^2}^2\Big]
  \gtrsim
  m a_m^2
  \asymp
  m^{-s}
  \asymp
  N^{-s/(s+1)}.
\]

For the matching minimax upper bound, consider the truncated orthogonal-series estimator based on $\mathcal D_N = \{(\bm X_i, Y_i)\}_{i=1}^N$. Writing $\psi_0 \equiv 1$ and $\theta_j := \ip{f^\ast}{\psi_j}_{L^2(\mu_X)}$ for $j \ge 0$, define
\[
  \hat\theta_j := \frac{1}{N}\sum_{i=1}^N Y_i \psi_j(\bm X_i),
  \qquad
  \hat f_m := \sum_{j=0}^m \hat\theta_j \psi_j.
\]
Since $\E[Y_i \psi_j(\bm X_i)] = \theta_j$, $Y_i \in [0,1]$, and $\norm{\psi_j}_{L^2(\mu_X)} = 1$, we have
\[
  \E\big[(\hat\theta_j - \theta_j)^2\big]
  = \frac{1}{N}\Var\big(Y_1 \psi_j(\bm X_1)\big)
  \le \frac{1}{N}\E\big[\psi_j(\bm X_1)^2\big]
  = \frac{1}{N}.
\]
By Parseval's identity and Lemma~\ref{lem:tail},
\[
  \sup_{f^\ast \in \mathcal F} \E\Big[\norm{\hat f_m - f^\ast}_{L^2}^2\Big]
  = \sup_{f^\ast \in \mathcal F} \left\{ \sum_{j=0}^m \E\big[(\hat\theta_j - \theta_j)^2\big] + \norm{(\bm I - \bm G_m)f^\ast}_{L^2}^2 \right\}
  \lesssim \frac{m}{N} + m^{-s}.
\]
Choosing $m \asymp N^{1/(s+1)}$ yields a genuine data-based estimator achieving risk $\mathcal O(N^{-s/(s+1)})$, completing the proof of Claim~2.
\end{proof}

\subsection{Empirical splits}
\label{appx:empirical}

For empirical CART, each tree is trained on a subsample $\mathcal D_n$ of size
$n := N^{(b)} = aN$ drawn from $\mathcal D_N$ under (A5), and thus marginally
$\mathcal D_n \sim \mu^{\otimes n}$. Because the selected partition depends on
the same subsample $\mathcal D_n$ used to fit its leaf values, and empirical
gain maximization need not follow the population CART trajectory, we require a
high-probability analogue of ASG that controls the \emph{population} progress
made by the empirically chosen splits.

For $u\ge1$, write
\[
  \Lambda_N(u):=1+\log(Nd)+u,
  \qquad
  \Lambda_N:=1+\log(Nd).
\]
For empirical CART, we assume the following high-probability analogue of
(A8) on a generic tree's training subsample $\mathcal D_n$: for every
$f^\ast \in \mathcal F$ and $u \ge 1$, with probability at least $1-\exp(-u)$
over $\mathcal D_n$, simultaneously for every $m\ge 1$ and reachable $T$ with
$|T|\lesssim m$, some $h\lesssim\mathfrak C_m(\varepsilon)$ further empirical
CART splits satisfy
\begin{align}
    V(T;f^\ast)
    -
    \E_\nu\!\left[V(T_{+h};f^\ast)\mid T,\mathcal D_n\right]
    &\ge
    c\,U_m(T;f^\ast)
    -
    C\norm{(\bm I-\bm G_m)f^\ast}_{L^2}^2 \nonumber\\
    &\quad-
    C_{\rm emp}
    \frac{\mathfrak C_m(\varepsilon)\Lambda_N(u)}{N},
    \tag{eASG}\label{eq:eASG}
\end{align}
where $\nu$ here denotes the tree-level randomization conditional on
$\mathcal D_n$. The final term in \eqref{eq:eASG} is a block-level stochastic
penalty of the order expected from uniform concentration over a class of
recursive trees with $\mathcal O\big(\mathfrak C_m(\varepsilon)\big)$ leaves,
rather than a per-split error accumulated across the block.

\begin{lemma}[Uniform adaptive-tree oracle bound]
\label{lem:adaptiveoracle}
Assume (A1), with $\mathcal D_n \sim \mu^{\otimes n}$ of size $n = aN$.
For a recursive partition $T$, let $\hat f_T$
denote the empirical MSE minimizer on $\mathcal D_n$ over functions constant on the leaves of
$T$, clipped to $[0,1]$ and assigned any fixed value in $[0,1]$ on an
empirically empty leaf. There is a universal constant $C_0<\infty$ such that,
for every $M\le n$ and $u\ge1$, with probability at least $1-\exp(-u)$ over $\mathcal D_n$,
simultaneously for every recursive partition $T$ with $|T|\le M$,
\[
  \norm{\hat f_T-f^\ast}_{L^2}^2
  \le
  3\norm{(\bm I-\bm P_T)f^\ast}_{L^2}^2
  +
  C_0\frac{M\log(Nd)+u}{N}.
\]
Consequently, the same inequality holds when $T$ is selected measurably from
$\mathcal D_n$.
\end{lemma}

\begin{proof}
Write
\[
    \widehat R(f)
    :=
    \frac1n\sum_{i=1}^n (y_i-f(\bm x_i))^2,
    \qquad
    \mathcal E(f)
    :=
    \norm{f-f^\ast}_{L^2}^2,
\]
and define the empirical excess risk on $\mathcal D_n$ by
\[
    \widehat{\mathcal E}(f)
    :=
    \widehat R(f)-\widehat R(f^\ast).
\]
Let $\varepsilon_i:=y_i-f^\ast(\bm x_i)$. Since $y_i,f^\ast(\bm x_i)\in[0,1]$,
the variables $\varepsilon_i$ are conditionally mean-zero and uniformly
sub-Gaussian.

Set
\[
    \delta_{M,u}
    :=
    \frac{M\log(nd)+u}{n}
    \lesssim
    \frac{M\log(Nd)+u}{N},
\]
where the second relation uses $n = aN$ with fixed $a \in (0,1]$.
We use the self-normalized uniform concentration bounds of
\citet{xu2026}. Their Lemma~7.2, together with the empirical partition
count of Lemma~2.1 and the empirical-equivalence argument of their
Lemma~E.3, implies that with probability at least $1-\exp(-u)$ over $\mathcal D_n$,
simultaneously for every tree function $f$ with at most $M$ leaves and
values in $[0,1]$,
\begin{equation}
    \left|
    \widehat{\mathcal E}(f)-\mathcal E(f)
    \right|
    \le
    C\left\{
        \norm{f-f^\ast}_{L^2}\sqrt{\delta_{M,u}}
        +\delta_{M,u}
    \right\},
    \label{eq:uniform-excess}
\end{equation}
for a universal constant $C$.
Indeed,
\[
    \widehat{\mathcal E}(f)
    =
    \norm{f-f^\ast}_{n}^2
    -
    2\langle \varepsilon,f-f^\ast\rangle_n,
\]
and \eqref{eq:uniform-excess} follows by applying their self-normalized
quadratic- and multiplier-process bounds to $f-f^\ast$; note that
$\norm{f-f^\ast}_\infty\le1$.

Applying Young's inequality with the fixed choice $\gamma=1/2$ gives,
after adjusting constants,
\begin{equation}
    \left|
    \widehat{\mathcal E}(f)-\mathcal E(f)
    \right|
    \le
    \frac12\,\mathcal E(f)
    +
    C_1\delta_{M,u}
    \label{eq:uniform-excess-young}
\end{equation}
simultaneously over the same class.

Fix now any recursive partition $T$ with $|T|\le M$ and let
\[
    f_T^\ast:=\bm P_Tf^\ast.
\]
Since $f^\ast\in[0,1]$, also $f_T^\ast\in[0,1]$, and $f_T^\ast$ belongs to the
leaf-constant function class associated with $T$. By definition of
$\hat f_T$,
\[
    \widehat{\mathcal E}(\hat f_T)
    \le
    \widehat{\mathcal E}(f_T^\ast).
\]
Applying \eqref{eq:uniform-excess-young} to both functions yields
\[
    \frac12\mathcal E(\hat f_T)-C_1\delta_{M,u}
    \le
    \widehat{\mathcal E}(\hat f_T)
    \le
    \widehat{\mathcal E}(f_T^\ast)
    \le
    \frac32\mathcal E(f_T^\ast)+C_1\delta_{M,u}.
\]
Therefore
\[
    \mathcal E(\hat f_T)
    \le
    3\mathcal E(f_T^\ast)
    +
    4C_1\delta_{M,u}.
\]
Finally,
\[
    \mathcal E(f_T^\ast)
    =
    \norm{(\bm I-\bm P_T)f^\ast}_{L^2}^2,
\]
which proves the displayed bound after absorbing $a^{-1}$ and universal constants into $C_0$.
The concentration event is uniform over all recursive partitions with at
most $M$ leaves. Hence the result remains valid if $T$ is chosen
measurably from the same subsample $\mathcal D_n$.
\end{proof}

\textbf{Claim 3}. Under the assumptions of Thm.~\ref{thm:minimax},
replacing oracle splits with empirical splits preserves the minimax polynomial
rate up to a logarithmic factor under \eqref{eq:eASG}.

\begin{proof}
Fix any $f^\ast \in \mathcal F$, $m \ge 1$, and $u \ge 1$, and consider a generic tree trained on its subsample $\mathcal D_n$ of size $n = aN$. On the event where \eqref{eq:eASG} holds for $\mathcal D_n$, applying Lemma~\ref{lem:block} and Proposition~\ref{prop:population-approx} conditionally on $\mathcal D_n$ with
\[
    \sigma_m
    =
    C_{\rm emp}
    \frac{\mathfrak C_m(\varepsilon)\Lambda_N(u)}{N}
\]
and $\mathfrak C_m(\varepsilon)=\mathcal O(m)$ gives a multiscale budget
$M_m\lesssim m$. Choosing a predetermined leaf budget $M \asymp m$ with $M \ge M_m$ and padding $T_{M_m}$ to $T_M$ by arbitrary recursive refinements cannot increase the population projection error $\norm{(\bm I-\bm P_T)f^\ast}_{L^2}^2$, so on the \eqref{eq:eASG} event we have
\[
    \E_\nu \Big[
    \norm{(\bm I-\bm P_{T_{M}})f^\ast}_{L^2}^2 \,\Big|\, \mathcal D_n\Big]
    \lesssim
    m^{-s}
    +
    \frac{m\Lambda_N(u)}{N}.
\]
Intersect this event with the high-probability event of Lemma~\ref{lem:adaptiveoracle} applied at leaf budget $M \asymp m$, which holds simultaneously over all $M$-leaf partitions. By a union bound, the combined event $\mathcal G_u$ has probability at least $1 - 2\exp(-u)$ over $\mathcal D_n$. On $\mathcal G_u$, taking conditional expectation $\E_\nu[\,\cdot \mid \mathcal D_n]$ over the internal tree randomization gives
\[
    \E_\nu\Big[
    \norm{\hat f_{T_{M}}-f^\ast}_{L^2}^2 \,\Big|\, \mathcal D_n\Big]
    \lesssim
    m^{-s}
    +
    \frac{m\Lambda_N(u)}{N}.
\]
Since $\hat f_{T_M}$ and $f^\ast$ take values in $[0,1]$, the squared error is at most $1$ on $\mathcal G_u^c$. Integrating the exponential tail in $u \ge 1$ therefore yields the unconditional single-tree bound
\[
    \E\Big[
    \norm{\hat f_{T_M}-f^\ast}_{L^2}^2\Big]
    \lesssim
    m^{-s}
    +
    \frac{m\Lambda_N}{N}.
\]
Finally, by Jensen's inequality, averaging across the $B=\infty$ trees in the forest (over both subsample selection and internal tree randomization) cannot increase squared $L^2$ risk:
\[
    \sup_{f^\ast\in\mathcal F}~
    \E\Big[
    \norm{\hat f_{\rm RF}-f^\ast}_{L^2}^2\Big]
    \le
    \sup_{f^\ast\in\mathcal F}~
    \E\Big[
    \norm{\hat f_{T_M}-f^\ast}_{L^2}^2\Big]
    \lesssim
    m^{-s}
    +
    \frac{m\Lambda_N}{N}.
\]
Balancing the two terms gives
\[
    M \asymp m
    \asymp
    \left(\frac{N}{\Lambda_N}\right)^{1/(s+1)}
    \asymp
    \left(\frac{N}{\log(Nd)}\right)^{1/(s+1)},
\]
and hence
\[
    \sup_{f^\ast\in\mathcal F}~
    \E\Big[
    \norm{\hat f_{\rm RF}-f^\ast}_{L^2}^2\Big]
    \lesssim
    \left(\frac{\log(Nd)}{N}\right)^{s/(s+1)}.
\]
This proves the third claim, thereby completing our proof of Thm.~\ref{thm:minimax}.
\end{proof}


\section{Theorem scope}\label{appx:scope}

In this section, we present auxiliary results that delimit the scope of Thm.~\ref{thm:minimax}. We consider a $d$-dimensional reference forest, an explicit function class, and a greedy tree algorithm for which every hypothesis of
our main theorem holds, and for which the minimax risk is exactly
$N^{-s/(s+1)}$. Appx.~\ref{sec:ex-operator} constructs the operator and
verifies (A3)--(A4) together with $\mathfrak C_m(\varepsilon)=\mathcal O(m)$. Appx.~\ref{sec:ex-gating} shows by
counterexample that these spectral conditions do \emph{not} imply ASG, which
is why an algorithmic hypothesis is needed. Appx.~\ref{sec:ex-hv}
gives a structural condition on the target under which ASG follows, and
Appx.~\ref{sec:ex-packing} shows that the resulting class still carries the
full minimax difficulty. 
Appx.~\ref{sec:ex-sid} establishes that the popular structural impurity decrease (SID) assumption implies ASG pointwise.
Appx.~\ref{sec:ex-general} treats superlinear
$\mathfrak C_m$ and connects to the approximation-theoretic literature. Throughout this section, $\mu_X$ is presumed uniform on $\cX=[0,1]^d$.

\subsection{A canonical \texorpdfstring{$d$}{d}-dimensional reference forest}
\label{sec:ex-operator}

\begin{definition}[Cyclic dyadic hierarchy]
\label{def:hierarchy}
Let $\cH_d$ be the infinite rooted binary tree of cells obtained by bisecting
every cell at depth $\ell$ at its midpoint along coordinate
$j(\ell)=1+(\ell\bmod d)$. We write $v$ for a node, $A_v\subseteq\cX$ for its
cell, $\dep(v)$ for its depth, $w\succ v$ for strict descendants, and
$w\succeq v$ for $\{v\}\cup\{w\succ v\}$. There are $2^\ell$ nodes at depth
$\ell$ and $|A_v|=2^{-\dep(v)}$.

For each $v$ let $\psi_v$ be the associated Haar contrast: $\psi_v$ is
supported on $A_v$, takes the values $\pm|A_v|^{-1/2}$ on the two children of
$A_v$, and satisfies $\norm{\psi_v}_{L^2}=1$, $\int\psi_v\,d\mu_X=0$. The
system $\{1\}\cup\{\psi_v\}_{v\in\cH_d}$ is an orthonormal basis of
$L^2(\mu_X)$.
\end{definition}

\begin{definition}[Pruned reference forest]
\label{def:pruned}
Fix $\beta>1$ and set $p:=2^{-\beta}$. Let $\Pi_\beta$ be the law of the random
recursive partition $T$ obtained by splitting the root with probability $p$
and, independently for every node whose parent has been split, splitting that
node with probability $p$. Put $\bm K_\Pi:=\E_{T\sim\Pi_\beta}[\bm P_T]$.
\end{definition}

Although every partition is a pruning of the same underlying dyadic
hierarchy, $\Pi_\beta$ is supported on mutually incomparable partitions,
since different branches may terminate at different depths. Thus the reference
forest is not merely a mixture over truncation depths of a single nested grid
sequence.

Because the nonconstant eigenvalues below have multiplicity within each
depth, throughout this example we use the canonical Haar eigenbasis
$\{\psi_v\}$ ordered by depth (with arbitrary ordering within a depth).
Let $\bm G_m$ denote the orthogonal projector onto the
constant eigenfunction together with the first $m$ nonconstant Haar modes.
Since $\bm P_T1=1$ for every tree, including the constant changes neither CART
gains nor approximation error, but ensures that
$\norm{(\bm I-\bm G_m)f}_{L^2}$ is the genuine nonconstant spectral tail.

\begin{proposition}[Spectrum, trace, and localization]
\label{prop:spectrum}
Under Defs.~\ref{def:hierarchy}--\ref{def:pruned}:
\begin{enumerate}[label=(\roman*),nosep]
\item $\bm K_\Pi\psi_v=p^{\,\dep(v)+1}\psi_v$ for every $v$, and
      $\bm K_\Pi 1=1$;
\item ordering the nonconstant eigenvalues decreasingly,
      $\lambda_j\asymp j^{-\beta}$, so \textup{(A3)} holds;
\item $\tr\bm K_\Pi=\E_{\Pi_\beta}|T|<\infty$ if and only if $\beta>1$;
\item \textup{(A4)} holds: for every $m$,
      $\sup_{\omega\in\{-1,1\}^m}\norm{\sum_{j=m+1}^{2m}\omega_{j-m}\psi_j}_\infty
      \lesssim\sqrt m \le m^{\alpha\beta+1/2}$.
\end{enumerate}
\end{proposition}

\begin{proof}
(i) Fix $v$ and $T\sim\Pi_\beta$. If $v$ is split in $T$ then both children of
$A_v$ are unions of cells of $T$, so $\psi_v$ is $\sigma_T$-measurable and
$\bm P_T\psi_v=\psi_v$. If $v$ is not split in $T$ then $A_v$ is contained in a
cell $A$ of $T$ on which $\psi_v$ has mean zero, whence $\bm P_T\psi_v=0$.
Splitting $v$ requires splitting $v$ and all its $\dep(v)$ ancestors, which by
independence has probability $p^{\dep(v)+1}$. Averaging gives the claim;
$\bm P_T1=1$ for every $T$ gives the second.

(ii) By (i) the eigenvalue at depth $\ell$ is $p^{\ell+1}=2^{-\beta(\ell+1)}$
with multiplicity $2^\ell$. Hence
$\#\{j:\lambda_j\ge2^{-\beta(L+1)}\}=\sum_{\ell\le L}2^\ell=2^{L+1}-1$. If
$\lambda_j=2^{-\beta(\ell+1)}$ then $2^{\ell}\le j\le 2^{\ell+1}$, so
$2^{-\beta}j^{-\beta}\le\lambda_j\le 2^{-\beta}(j/2)^{-\beta}$.

(iii) Since $1$ is an eigenfunction with eigenvalue $1$,
\[
  \tr\bm K_\Pi
  =1+\sum_v\lambda_v
  =1+\sum_{\ell\ge0}2^\ell p^{\ell+1}
  =1+p\sum_{\ell\ge0}2^{(1-\beta)\ell},
\]
which is finite iff $\beta>1$. Likewise
$\E_\Pi|T|=1+p\sum_{\ell\ge0}(2p)^\ell$, and
$\tr\bm K_\Pi=\E_\Pi\tr\bm P_T=\E_\Pi|T|$ because $\bm P_T$ is a
projection of rank $|T|$.

(iv) The indices $j\in(m,2m]$ correspond to nodes at depths $\ell$ and possibly
$\ell+1$, where $2^\ell\asymp m$. Nodes at a common depth have disjoint cells
and $\norm{\psi_v}_\infty=2^{\dep(v)/2}$. A signed sum of functions with
disjoint supports has sup-norm equal to the largest summand, so the sum over
each of the at most two depths is $\lesssim2^{(\ell+1)/2}\lesssim\sqrt m$;
add the two.
\end{proof}

\begin{remark}
Part (iii) shows that $\beta>1$ in (A3) is exactly the subcriticality condition
under which the reference forest has finite expected size. It is a property of
the construction, not an extra restriction.
\end{remark}

\begin{proposition}[Linear spectral tree complexity]
\label{prop:complexity}
For every $m\ge1$, $\mathfrak C_m(0)\le 2m+1$. Moreover, for every fixed
$\varepsilon\in(0,1)$,
\[
  m+1\le \mathfrak C_m(\varepsilon)\le 2m+1,
\]
so $\mathfrak C_m(\varepsilon)\asymp m$ for this example.
\end{proposition}

\begin{proof}
Let $L$ be least with $2^{L+1}-1\ge m$. By Prop.~\ref{prop:spectrum}(i)--(ii),
the leading $m$ eigenfunctions are among $\{\psi_v:\dep(v)\le L\}$. Let $T_L$
be the partition obtained by splitting every node of depth $\le L$; it is
ancestor-closed, has $2^{L+1}-1$ internal nodes and hence $2^{L+1}$ leaves, and
every $\psi_v$ with $\dep(v)\le L$ is $\sigma_{T_L}$-measurable. Therefore
$\bm P_{T_L}u=u$, i.e. $\norm{\bm Q_{T_L}u}^2=0$, for every
$u\in\mathrm{ran}(\bm G_m)$. Minimality of $L$ gives $2^{L}-1<m$, so
$|T_L|=2^{L+1}<2(m+1)$.

For the lower bound, let $T$ have at most $m$ leaves. Since
$\mathrm{ran}(\bm G_m)$ has dimension $m+1$ while $\bm P_T$ has rank at most
$m$, there is a unit $u\in\mathrm{ran}(\bm G_m)$ with $\bm P_Tu=0$. Hence
$\norm{\bm Q_Tu}^2=1$, so no such tree can achieve tolerance
$\varepsilon<1$ uniformly over the leading spectral space. Therefore
$\mathfrak C_m(\varepsilon)\ge m+1$.
\end{proof}

We also record the elementary identity that drives the following results.

\begin{lemma}[Gains are squared Haar coefficients]
\label{lem:gain-coeff}
Let $T$ be a partition all of whose cells are cells of $\cH_d$, let $v$ be a
leaf of $T$, and let the split of $A_v$ be its canonical bisection. Writing
$c_w:=\ip{f}{\psi_w}$,
\[
  \Gamma_T(v;f)=c_v^2,
  \qquad
  \mu_X(A_v)\Var\bigl(f\mid A_v\bigr)=\sum_{w\succeq v}c_w^2=:R_v(f),
  \qquad
  V(T;f)=\sum_{v\in\partial T}R_v(f),
\]
where $\partial T$ denotes the leaves of $T$.
\end{lemma}

\begin{proof}
The one-step refinement operator $\bm Q_{T,a}=\bm P_{T\oplus a}-\bm P_T$ is the
orthogonal projection onto the span of the new stump, which for the canonical
bisection of $A_v$ is exactly $\psi_v$; hence
$\Gamma_T(v;f)=\norm{\bm Q_{T,a}f}^2=\ip{f}{\psi_v}^2$. For the second, the
residual of $f$ on $A_v$ is $\sum_{w\succeq v}c_w\psi_w$ by orthonormality and
the fact that $\{\psi_w\}_{w\succeq v}$ spans the mean-zero functions on $A_v$
measurable with respect to $\cH_d$; take squared norms. The third follows
because leaves are disjoint.
\end{proof}

\begin{definition}[Hierarchy-restricted CART]
\label{def:dyadic-cart}
\emph{Dyadic CART} is the best-first CART variant obtained by setting
$\mtry=d$ and restricting each current leaf $A_v$ to its unique canonical
next split in $\cH_d$, namely the midpoint bisection along coordinate
$j(\dep(v))$. Thus greediness determines which leaf is refined, while the
coordinate and threshold of its refinement are fixed by the hierarchy. If all
available canonical gains are zero, the largest-mass leaf is refined, with any
remaining tie broken at random. Every cell produced by the algorithm is a cell
of $\cH_d$.
\end{definition}

The restriction is part of the hypotheses of the worked example. An
unrestricted split can produce cells misaligned with $\cH_d$, after which
Lemma~\ref{lem:gain-coeff} no longer applies. Dyadic trees are the setting of
much rigorous tree theory \citep{donoho1997, scott2006, blanchard2007}; the
theorem of the main text is not restricted in this way, and only the worked
example below is.

\subsection{Spectral decay does not imply ASG}
\label{sec:ex-gating}

\begin{lemma}[Gating]
\label{lem:gating}
Fix any block constant $D<\infty$. For every sufficiently large $m$ and every
$\eta_0,c_1>0$ there is a regression function $f^\ast$ in the source class of
$\bm K_\Pi$, and a state $T$ reachable by dyadic CART with $|T|=\mathcal O(\log m)$, such
that for all $q\le Dm$,
\[
  V(T;f^\ast)-\E\bigl[V(T_{+q};f^\ast)\mid T\bigr]
  \ <\ \eta_0\,U_m(T;f^\ast)-c_1\norm{(\bm I-\bm G_m)f^\ast}^2 .
\]
Hence no choice of a uniform $\mathcal O(m)$ block budget makes \textup{(A8)} follow
from the source and spectral assumptions alone.
\end{lemma}

\begin{proof}
Let $L$ be large enough that the leading $m$ nonconstant modes are contained
among depths at most $L$. Choose a node $w$ from this leading block and an
unresolved ancestor $v_0$ of $w$. On a sibling branch, first open a path down
to depth $L+1$, and from its terminal node choose a chain
$u_1\succ u_2\succ\cdots\succ u_H$ of descendants, with
$H:=\lceil Dm\rceil$. Thus every $u_r$ lies strictly beyond the leading
spectral block. Define a centered finite Haar expansion $h$ with
\[
  c_{v_0}=0,\qquad c_w=\gamma,\qquad
  c_{u_r}=\delta\quad(1\le r\le H),
\]
and assign positive coefficients of magnitude at most $\delta$ on the already-open sibling path so
that dyadic CART reaches the state $T$ in which $v_0$ and $u_1$ are both
available. The path has length $\mathcal O(L)=\mathcal O(\log m)$; all of its coefficients have
already been resolved at $T$ and therefore do not enter $V(T;h)$, while at most $\mathcal O(\log m)$ of them can lie outside $\mathrm{ran}(\bm G_m)$. Set all
remaining coefficients to zero. Because the expansion is finite, $h$ belongs
to the domain of $\bm K_\Pi^{-\alpha}$.

Choose any fixed $\theta\in(0,\min\{1,\zeta\})$ and set
$f^\ast=\theta+\rho h$, with $\rho>0$ sufficiently small that
$f^\ast\in[0,1]$ and $\norm{\bm K_\Pi^{-\alpha}f^\ast}\le\zeta$. Translation and
positive rescaling do not change any comparison of CART gains.

Take $T$ to be the reachable state just described. By Lemma~\ref{lem:gain-coeff}, the gain that opens
the gate $v_0$ is zero, while the available gain on the distracting branch is
$\rho^2\delta^2>0$. Best-first dyadic CART therefore takes the latter.
Inductively it descends the distracting chain for all $H$ steps and never opens
$v_0$. Consequently, for every $q\le H$,
\[
  V(T;f^\ast)-V(T_{+q};f^\ast)\le q\rho^2\delta^2
  \le H\rho^2\delta^2,
\]
whereas $U_m(T;f^\ast)\ge \rho^2\gamma^2$ and, for $m$ sufficiently large,
\[
  \norm{(\bm I-\bm G_m)f^\ast}^2
  \le (H + C\log m)\rho^2\delta^2
  \le 2H\rho^2\delta^2.
\]
Choosing $\delta$ so that
$H\delta^2(1+2c_1)<\eta_0\gamma^2$ yields the claim.
\end{proof}

The mechanism is that an arbitrarily small \emph{positive} gain is preferred to
a zero gain indefinitely. The tie-breaking rule (A7) does not intervene,
because the candidate gains do not tie. Note also that the total distracting
energy $H\delta^2$ can be made arbitrarily small. The obstruction is not the
amount of deep energy but the fact that it screens coarse structure from a
myopic criterion. Lemma~\ref{lem:gating} is therefore the precise reason (A8) is
an assumption separate from (A2)--(A4).

\subsection{Hierarchical visibility implies ASG}
\label{sec:ex-hv}

\begin{definition}[Hierarchical visibility]
\label{def:hv}
For $\kappa>0$, let $\cF_{\mathrm{HV}}(\kappa)$ be the set of
$f\in L^2(\mu_X)$ whose Haar coefficients $c_v=\ip{f}{\psi_v}$ satisfy
\begin{equation}
  c_v^2\ \ge\ \kappa\sum_{w\succ v}c_w^2
  \qquad\text{for every }v\in\cH_d .
  \tag{HV}\label{eq:HV}
\end{equation}
\end{definition}

Condition \eqref{eq:HV} says that a region's own coarse contrast is never negligible
relative to the structure it contains: large-scale signal cannot be hidden
entirely behind a zero-contrast ancestor. It imposes nothing on nodes with no
descendant energy, and nothing on the \emph{signs} of any coefficient. It is
loosely analogous to self-similarity and polished-tail conditions used in
adaptive nonparametric inference \citep{gine2010, szabo2015, bull2012},
although the necessity results there concern honest confidence sets rather
than minimax estimation, and we claim only the analogy. It excludes pure
interactions such as $f=\mathrm{sgn}(x_1-\tfrac12)\,\mathrm{sgn}(x_2-\tfrac12)$,
for which the root contrast vanishes while descendant energy does not.

\begin{proposition}[HV implies ASG]
\label{prop:hv-asg}
Let $f\in\cF_{\mathrm{HV}}(\kappa)$ and set $\delta:=\kappa/(1+\kappa)$.
Then for dyadic CART, at every state $T$ whose cells are cells of $\cH_d$,
\begin{equation}
  \Gamma^\star(T;f)\ \ge\ \frac{\delta}{|T|}\,V(T;f).
  \tag{$\ast$}\label{eq:hv-sid}
\end{equation}
Consequently, for every $C_{\mathrm{reach}}\ge1$ there is $C_{\mathrm{block}}=C_{\mathrm{block}}(\delta,C_{\mathrm{reach}})$ such that
\textup{(A8)} holds on $\cF_{\mathrm{HV}}(\kappa)$ for every reachable state
with $|T|\le C_{\mathrm{reach}} m$ and some $H\le C_{\mathrm{block}}m\lesssim\mathfrak C_m(\varepsilon)$ further
splits, where the last comparison uses Prop.~\ref{prop:complexity}.
\end{proposition}

\begin{proof}
By \eqref{eq:HV} and Lemma~\ref{lem:gain-coeff}, for any leaf $v$,
$R_v=c_v^2+\sum_{w\succ v}c_w^2\le c_v^2(1+\kappa^{-1})$, i.e.
$c_v^2\ge\delta R_v$. Taking the leaf with largest $R_v$ and using
$V(T;f)=\sum_{v\in\partial T}R_v\le|T|\max_vR_v$ gives \eqref{eq:hv-sid}.

Best-first CART therefore satisfies
$V_{t+1}\le(1-\delta/t)V_t$ for $t=|T|$. Iterating from $t$ to $t+H$,
\[
  \frac{V_{t+H}}{V_t}\le\prod_{r=t}^{t+H-1}\Bigl(1-\frac{\delta}{r}\Bigr)
  \le\exp\Bigl(-\delta\sum_{r=t}^{t+H-1}\frac1r\Bigr)
  \le\Bigl(\frac{t}{t+H}\Bigr)^{\delta}.
\]
For a state reachable with $t\le C_{\mathrm{reach}} m$ and $H=\lceil C_{\mathrm{block}}m\rceil$ this is at most
$$\bigl(C_{\mathrm{reach}}/(C_{\mathrm{reach}}+C_{\mathrm{block}})\bigr)^{\delta},$$ which is below $1-\eta_0$ for any
fixed $\eta_0<1$ once $C_{\mathrm{block}}=C_{\mathrm{block}}(\delta,C_{\mathrm{reach}},\eta_0)$ is large enough.
Finally, $f=\bm G_mf+(\bm I-\bm G_m)f$ and the triangle inequality give
$$V(T;f)\ge\tfrac12U_m(T;f)-\norm{(\bm I-\bm G_m)f}^2,$$ so
\[
  V(T;f)-V(T_{+H};f)\ \ge\ \eta_0V(T;f)
  \ \ge\ \tfrac{\eta_0}{2}U_m(T;f)-\eta_0\norm{(\bm I-\bm G_m)f}^2,
\]
which is \eqref{eq:ASG} with $c=\eta_0/2$ and $C=\eta_0$.
\end{proof}

Inequality \eqref{eq:hv-sid} is a SID condition
expressed in tree coordinates. \eqref{eq:HV} is a convenient coefficient-level
sufficient condition for it. In this worked example, HV therefore yields a
one-step SID-type inequality, which is stronger than the blockwise progress
required by ASG. The important feature for the minimax construction below is
that \eqref{eq:HV} constrains coefficient magnitudes but not their signs, so a
large hypercube of fine-scale alternatives can remain inside the class.

\subsection{A minimax-hard subclass}
\label{sec:ex-packing}

\begin{proposition}[Packing]
\label{prop:packing}
Fix $s>0$ and $1<a<2^s$, let $L\ge1$ and $m:=2^L$. Define centered Haar
perturbations $h_\omega$ by coefficient magnitudes
\[
  |c_v|=\Bigl(\frac{E_{\dep(v)}}{2^{\dep(v)}}\Bigr)^{1/2},
  \qquad
  E_{L-r}:=A\,a^{r}\,2^{-sL}\quad(0\le r\le L),
  \qquad c_v=0\ \text{ for }\dep(v)>L,
\]
with signs fixed arbitrarily at depths $<L$ and free at depth $L$:
\[
  h_\omega:=\sum_{\dep(v)<L}c_v\psi_v
  +a_L\sum_{\dep(v)=L}\omega_v\psi_v,
  \qquad \omega\in\{-1,1\}^{m},
  \qquad a_L^2=E_L/2^L.
\]
Let $f_\omega^\ast:=\tfrac12+h_\omega$. For any fixed source radius
$\zeta>1/2$, a constant $A=A(s,a,d,\zeta)>0$ can be chosen small enough that:
\begin{enumerate}[label=(\roman*),nosep]
\item every $f_\omega^\ast$ satisfies \eqref{eq:HV} with $\kappa=a-1$;
\item $\sup_\omega\norm{\bm K_\Pi^{-\alpha}f_\omega^\ast}^2\le\zeta^2$, so
      every $f_\omega^\ast$ lies in the source class of \textup{(A2)};
\item $\sup_\omega\norm{h_\omega}_\infty\le\tfrac14$, and hence
      $f_\omega^\ast(x)\in[\tfrac14,\tfrac34]$;
\item $a_L^2\asymp m^{-(s+1)}$ and
      $\norm{f_\omega^\ast-f_{\omega'}^\ast}^2
      =4a_L^2\,\mathrm{Ham}(\omega,\omega')$.
\end{enumerate}
Consequently, if $\cF_{\mathrm{HV}}$ denotes the intersection of the source
class with $\cF_{\mathrm{HV}}(a-1)$, then
\[
  \inf_{\hat f}\ \sup_{f^\ast\in\cF_{\mathrm{HV}}}
  \E\norm{\hat f-f^\ast}^2
  \ \gtrsim\ N^{-s/(s+1)} .
\]
\end{proposition}

\begin{proof}
(i) Fix $v$ at depth $k=L-r$ with $r\ge1$. Its subtree contains
$2^{j-k}$ nodes at each depth $j\in(k,L]$, so
\[
  \sum_{w\succ v}c_w^2
  =2^{-k}\sum_{j=k+1}^{L}E_j
  =2^{-k}A2^{-sL}\frac{a^{r}-1}{a-1},
\]
while $c_v^2=2^{-k}A2^{-sL}a^r$. Their ratio is
$a^r(a-1)/(a^r-1)\ge a-1$. Nodes at depth $L$ have no descendant energy, and
nodes below depth $L$ have neither coefficient nor descendant energy. The
condition is uniform in $\omega$ because only magnitudes enter.

(ii) The constant function is an eigenfunction of $\bm K_\Pi$ with eigenvalue
$1$ and is orthogonal to every Haar mode. By Prop.~\ref{prop:spectrum}(i),
$\lambda_v^{-2\alpha}\asymp2^{s\dep(v)}$ with $s=2\alpha\beta$. Therefore
\[
  \norm{\bm K_\Pi^{-\alpha}f_\omega^\ast}^2
  =\frac14+\norm{\bm K_\Pi^{-\alpha}h_\omega}^2
  \lesssim \frac14+A\sum_{r=0}^{L}\Bigl(\frac{a}{2^{s}}\Bigr)^r
  \le \frac14+\frac{C A}{1-a2^{-s}}.
\]
Since $a<2^s$ and $\zeta>1/2$, choosing $A$ sufficiently small makes the
right-hand side at most $\zeta^2$.

(iii) At any $x$ exactly one node per depth has $x\in A_v$, and
$|c_v|\norm{\psi_v}_\infty=E_k^{1/2}$. Hence
\[
  \norm{h_\omega}_\infty
  \le \sum_{k\le L}E_k^{1/2}
  =A^{1/2}2^{-sL/2}\sum_{r=0}^{L}a^{r/2},
\]
which is bounded uniformly in $L$ because $a2^{-s}<1$, and can be made at most
$1/4$ by reducing $A$.

(iv) $a_L^2=E_L2^{-L}=A2^{-sL}2^{-L}=Am^{-(s+1)}$, and the separation identity
follows from orthonormality.

For the final display, take $Y\mid X=x\sim\mathrm{Bernoulli}(f_\omega^\ast(x))$.
Neighbouring vertices differ in one sign, and since all regression parameters
lie in $[1/4,3/4]$,
$\KL(P_\omega^{\otimes N}\Vert P_{\omega'}^{\otimes N})\lesssim Na_L^2$.
Choosing $m\asymp N^{1/(s+1)}$ makes this $\mathcal O(1)$, and the standard Assouad
hypercube argument gives a lower bound of order
$ma_L^2\asymp m^{-s}\asymp N^{-s/(s+1)}$.
\end{proof}

\begin{theorem}[Worked example]
\label{thm:worked}
Let $\bm K_\Pi$ be as in Def.~\ref{def:pruned} with $\beta>1$, let
$\alpha\in(0,1]$, $s=2\alpha\beta$, fix $1<a<2^s$ and a source radius
$\zeta>1/2$, and let $\cF_{\mathrm{HV}}$ be the intersection of the
corresponding source class with $\cF_{\mathrm{HV}}(a-1)$. Then population
dyadic CART forests grown to $M\asymp N^{1/(s+1)}$ leaves satisfy
\[
  \inf_{\hat f}~\sup_{f^\ast\in\cF_{\mathrm{HV}}}~
  \E\Big[\norm{\hat f-f^\ast}^2\Big]
  \ \asymp\
  \sup_{f^\ast\in\cF_{\mathrm{HV}}}~
  \E\Big[\norm{\hat f_{\mathrm{RF}}-f^\ast}^2\Big]
  \ \asymp\ N^{-s/(s+1)}.
\]
\end{theorem}

\begin{proof}
Uniform design and the response model $Y\mid \bm X=\bm x \sim \mathrm{Bernoulli}(f^\ast(\bm x))$ verify (A1) across $\cF_{\mathrm{HV}}$, while Prop.~\ref{prop:packing} constructs the class inside (A2).
Prop.~\ref{prop:spectrum} gives (A3)--(A4),
Prop.~\ref{prop:complexity} gives $\mathfrak C_m(\varepsilon)=\mathcal O(m)$, and
Prop.~\ref{prop:hv-asg} gives (A8) for hierarchy-restricted population CART.
The leaf budget in (A6) is built into Def.~\ref{def:dyadic-cart}; the
subsampling condition (A5) may be imposed independently when fitting leaf
means and does not alter the population split trajectory. Under HV a
state with unresolved energy has a strictly positive canonical gain by
\eqref{eq:hv-sid}, so the zero-gain convention in (A7) is immaterial before the
signal is exhausted. The upper bound of Thm.~\ref{thm:minimax} (Claim~1) therefore gives
$$\sup_{f^\ast\in\cF_{\mathrm{HV}}}\E[\norm{\hat f_{\mathrm{RF}}-f^\ast}^2]\lesssim N^{-s/(s+1)}.$$

For the matching minimax upper bound over ordinary data-based estimators $\hat f$, note that the fixed complete dyadic partition $T_L$ from Prop.~\ref{prop:complexity} has $|T_L|\asymp m$ leaves independently of $f^\ast$ and resolves all Haar modes through depth $L$, so its leaf-mean regressogram achieves risk $\lesssim m^{-s} + m/N \asymp N^{-s/(s+1)}$ at $m\asymp N^{1/(s+1)}$ by Lemma~\ref{lem:oracle-regressogram}.

Finally, Prop.~\ref{prop:packing} supplies the minimax lower bound $\gtrsim N^{-s/(s+1)}$ over ordinary estimators $\hat f$. Moreover, on the Assouad hypercube $\{f_\omega^\ast : \omega\in\{-1,1\}^m\}$ of Prop.~\ref{prop:packing}, all $f_\omega^\ast$ share identical squared Haar coefficients $c_v^2$ at every node $v\in\cH_d$. By Lemma~\ref{lem:gain-coeff}, the population dyadic CART gains $\Gamma_T(v; f_\omega^\ast) = c_v^2$ are invariant to $\omega$, so the oracle partition law of $\hat f_{\mathrm{RF}}$ is identical across the hypercube. Consequently, $\hat f_{\mathrm{RF}}$ is a fixed, $\omega$-independent estimator on this subexperiment, and the Assouad lower bound of Prop.~\ref{prop:packing} applies directly to $\hat f_{\mathrm{RF}}$ as well.
\end{proof}

\begin{remark}[Why the packing has no logarithmic loss]
\label{rem:log}
The compatibility window $1<a<2^s$ is doing two jobs at once. The inequality
$a>1$ gives a fixed positive visibility constant $\kappa=a-1$, while
$a<2^s$ makes the source-norm series in Prop.~\ref{prop:packing}(ii) summable.
At the boundary $a=2^s$ that series grows like $L\asymp\log m$, whereas any
strict interior choice keeps both the HV constant and the source radius fixed.
Since ASG only requires a fixed blockwise contraction, there is no need to push
$a$ to this boundary, and the minimax packing incurs no extra logarithmic
factor.
\end{remark}

\begin{remark}[Interactions]
For $d\ge2$, the class $\cF_{\mathrm{HV}}$ is not additive. A depth-$L$ contrast $\psi_v$ splits coordinate
$j(L)$ within a cell determined by the preceding $L$ splits, which cycle
through all $d$ coordinates. Free signs at depth $L$ therefore produce a
contrast in one coordinate whose sign depends on the position in the others,
i.e. a genuine multi-way interaction. The construction is not a sum of univariate functions
embedded in $\R^d$.
\end{remark}

\begin{remark}[An algorithmic alternative to HV]
\label{rem:algorithmic}
$\cF_{\mathrm{HV}}$ restricts the function class while leaving CART unchanged.
The opposite trade is also available. In Lemma~\ref{lem:gating} the distracting
gains satisfy $\delta^2\ll V(T;f)$, so a rule that explores whenever
$\max_a\Gamma(a)<\theta V(T)$ --- splitting the leaf of largest impurity at the
midpoint of a random offered coordinate --- opens the gate after $\mathcal O(d)$
explorations in expectation. Such a rule would preserve the full source ball at
the cost of a further deviation from CART; we do not pursue it here.
\end{remark}

\subsection{Structural impurity decrease and amortized spectral gain}
\label{sec:ex-sid}

Numerous authors have studied variants of the structural impurity decrease
(SID) assumption \citep{wager2015, syrg2020, chi2022, klusowski_cart},
which imposes a uniform lower bound on the population improvement available
from a single tree split.

For a measurable rectangle $A\subseteq[0,1]^d$ with $\mu_X(A)>0$, define
its probability-weighted conditional variance by
\[
    v_f(A)
    :=
    \mu_X(A)\Var(f(\bm x)\mid \bm x\in A).
\]
Given an admissible axis-aligned split $(j,\tau)$ of $A$, write
\[
    A_L(j,\tau)
    :=
    A\cap\{\bm x:x_j\le\tau\},
    \qquad
    A_R(j,\tau)
    :=
    A\cap\{\bm x:x_j>\tau\},
\]
and define the corresponding population impurity decrease
\[
    \Delta_f(A;j,\tau)
    :=
    v_f(A)
    -
    v_f(A_L(j,\tau))
    -
    v_f(A_R(j,\tau)).
\]
Equivalently, whenever both children have positive mass,
\[
    \Delta_f(A;j,\tau)
    =
    \frac{\mu_X(A_L)\mu_X(A_R)}{\mu_X(A)}
    \Big(
        \E\big[f(\bm x)\mid \bm x\in A_L\big]
        -
        \E\big[f(\bm x)\mid \bm x\in A_R\big]
    \Big)^2.
\]

\begin{definition}[Structural impurity decrease]
\label{def:SID}
A function $f\in L^2$ satisfies SID with constant
$\lambda\in(0,1]$ if, for every admissible rectangular cell $A$ with
$\mu_X(A)>0$,
\[
    \sup_{(j,\tau)\in\mathcal S(A)}
    \Delta_f(A;j,\tau)
    \ge
    \lambda\,v_f(A),
\]
where $\mathcal S(A)$ denotes the admissible axis-aligned splits of $A$.
\end{definition}

SID has proven useful in analyses of adaptive CART and random forests, though its requirement of a uniform local impurity decrease can be difficult to verify directly \citep{Klusowski2024}. A notable recent result of \citet{mazumder2023} gives tractable sufficient conditions for SID.
In particular, they show that additive regression functions satisfy SID when their component functions obey a local reverse Poincar\'e inequality, thereby connecting SID to more familiar nonparametric smoothness classes.

In this subsection, we show that SID supplies a simple pointwise sufficient condition for ASG.

\begin{proposition}[SID implies ASG pointwise]
\label{prop:SID}
Suppose $f\in L^2$ satisfies population SID with constant
$\lambda>0$. Suppose further that, conditional on any current leaf, the
tree randomization offers a SID-witnessing coordinate with probability at
least $\pi_0>0$. Then, on any reachability window
\[
    |T|\le C_{\rm reach}m,
\]
$f$ satisfies the conclusion of ASG with a block budget $h\le C_0m$,
where $C_0$ depends only on the desired contraction,
$C_{\rm reach}$, $\lambda$, and $\pi_0$.
\end{proposition}

\begin{proof}
For each leaf $A\in T$, write
\[
    v_f(A)
    :=
    \mu_X(A)\Var(f(\bm x)\mid \bm x\in A).
\]
Since $\bm P_Tf$ is the conditional expectation of $f$ on the leaf
$\sigma$-algebra,
\[
    V(T;f)
    =
    \norm{(\bm I-\bm P_T)f}_{L^2}^2
    =
    \sum_{A\in T}v_f(A).
\]
Let $t:=|T|$. By SID, a leaf $A$ admits a population split with gain at
least $\lambda v_f(A)$. Since a SID-witnessing coordinate is offered with
probability at least $\pi_0$, best-first CART has conditional expected gain
at least
\[
    \bar\lambda\max_{A\in T}v_f(A),
    \qquad
    \bar\lambda:=\pi_0\lambda.
\]
Hence
\[
    \E_\nu[V(T_{t+1};f)\mid T_t]
    \le
    V(T_t;f)
    -
    \bar\lambda\max_{A\in T_t}v_f(A)
    \le
    \left(1-\frac{\bar\lambda}{t}\right)V(T_t;f).
\]
Iterating for $H=\lceil C_0m\rceil$ further splits gives
\[
    \E_\nu[V(T_{t+H};f)\mid T_t]
    \le
    \prod_{r=0}^{H-1}
    \left(1-\frac{\bar\lambda}{t+r}\right)
    V(T_t;f)
    \le
    \left(\frac{t}{t+H}\right)^{\bar\lambda}
    V(T_t;f).
\]
If $t\le C_{\rm reach}m$, then
\[
    \E_\nu[V(T_{t+H};f)\mid T_t]
    \le
    \left(
        \frac{C_{\rm reach}}
             {C_{\rm reach}+C_0}
    \right)^{\bar\lambda}
    V(T_t;f).
\]
Choosing $C_0$ sufficiently large therefore yields a fixed contraction:
for some $\eta\in(0,1)$,
\[
    V(T_t;f)
    -
    \E_\nu[V(T_{t+H};f)\mid T_t]
    \ge
    \eta V(T_t;f).
\]
Finally,
\[
    V(T;f)
    \ge
    \frac12 U_m(T;f)
    -
    \norm{(\bm I-\bm G_m)f}_{L^2}^2,
\]
so
\[
    V(T;f)
    -
    \E_\nu[V(T_{t+H};f)\mid T]
    \ge
    c\,U_m(T;f)
    -
    C\norm{(\bm I-\bm G_m)f}_{L^2}^2
\]
with $c=\eta/2$ and $C=\eta$.
Since $H=\mathcal O(m)$, this is the ASG conclusion in the linear-complexity regime.
\end{proof}

\begin{remark}[Scope of the SID implication]
\label{rem:SIDscope}
The preceding implication is pointwise in the target $f$: a fixed SID
constant gives a simple sufficient condition for amortized progress. ASG is
strictly blockwise in character and can accommodate trajectories for which no
uniform one-step impurity-decrease constant is available. The quantifiers
should therefore be kept distinct when importing SID examples into the
spectral source-class setting. In particular, SID for each member of a class
does not establish uniform ASG unless the SID constant is itself uniform over
that class; spectral superpositions may otherwise create cancellations or
gating effects.
\end{remark}

\subsection{Superlinear complexity and the approximation-theoretic view}
\label{sec:ex-general}

\begin{proposition}[Rate under polynomial complexity]
\label{prop:gamma}
Suppose $\mathfrak C_m(\varepsilon)\lesssim m^{\gamma}$ for some $\gamma\ge1$ and,
for this subsection, strengthen ASG so that it holds uniformly over reachable
states with $|T|\lesssim m^{\gamma}$ and uses
$h\lesssim\mathfrak C_m(\varepsilon)$ further splits. Then the population-split forest
satisfies
\[
  \sup_{f^\ast\in\mathcal F}~\E\Big[\norm{\hat f_{\mathrm{RF}}-f^\ast}^2\Big]
  \ \lesssim\ \inf_m~\Bigl\{m^{-s}+\frac{m^{\gamma}}{N}\Bigr\}
  \ \asymp\ N^{-s/(s+\gamma)},
\]
attained at $m\asymp N^{1/(s+\gamma)}$. Under the corresponding generalized
empirical ASG condition, the empirical-split rate is
$\bigl(\log(Nd)/N\bigr)^{s/(s+\gamma)}$.
\end{proposition}

\begin{proof}
At dyadic spectral scales up to $m$, the cumulative multiscale budget is
\[
  \sum_{\ell:\,2^\ell\le m}\mathfrak C_{2^\ell}(\varepsilon)
  \lesssim\sum_{\ell:\,2^\ell\le m}2^{\gamma\ell}
  \lesssim m^\gamma,
\]
which is why the enlarged reachability window is required. The same block
contraction argument as in Thm.~\ref{thm:minimax} therefore gives approximation
error $\mathcal O(m^{-s})$ using $M\lesssim m^\gamma$ leaves. For population-selected
partitions, estimating the leaf means contributes $\mathcal O(M/N)$, giving
$m^{-s}+m^\gamma/N$; balancing the terms yields the stated choice of $m$ and
rate. Under generalized eASG, the adaptive-tree complexity term is
$\mathcal O(M\log(Nd)/N)$ and the same optimization gives the empirical rate.
\end{proof}

This is an upper bound. A matching \emph{estimator-class} lower bound requires
a separate argument: in the sparse additive case with $d_0$ active $C^1$ components
($s=2, \gamma=d_0$), the lower bound of \citet[Thm.~1]{tan2022} shows that
axis-aligned leaf-averaging tree estimators cannot improve on
$N^{-2/(d_0+2)}$ (see Remark~\ref{rem:classical-tree-rates} and Appx.~\ref{appx:related}).
The additive class is therefore a diagnostic rather than an
information-theoretic minimax example: additive estimators can achieve $N^{-2/3}$, whereas
$\gamma=d_0$ quantifies the representational overhead induced by single-tree geometry.

\begin{proposition}[Transfer from tree approximation]
\label{prop:transfer}
If
$\sup_{\norm{g}\le\zeta}\inf_{|T|\le L}\norm{\bm Q_T\bm K_\Pi^{\alpha}g}^2
\lesssim L^{-s}$, then $\mathfrak C_m(\varepsilon)\lesssim\varepsilon^{-1/s}m$.
\end{proposition}

\begin{proof}
For a unit $\bm u\in\mathrm{ran}(\bm G_m)$ one has
$\norm{\bm K_\Pi^{-\alpha}\bm u}\le\lambda_m^{-\alpha}$, so $\zeta\lambda_m^{\alpha}\bm u$
lies in the source ball. The hypothesis gives
$\inf_{|T|\le L}\norm{\bm Q_T \bm u}^2\lesssim\lambda_m^{-2\alpha}L^{-s}
\asymp(m/L)^{s}$, and $L=\varepsilon^{-1/s}m$ makes this $\lesssim\varepsilon$.
\end{proof}

Prop.~\ref{prop:transfer} shows that matching recursive-tree and spectral
approximation rates is a simple sufficient condition for linear spectral tree
complexity. It provides the bridge to the approximation theory for anisotropic
and spatially inhomogeneous classes
\citep{akakpo2012, xu2026}: results of that type supply exactly the hypothesis
above. We do not claim that unrestricted classes of that kind satisfy (A8) ---
Lemma~\ref{lem:gating} shows representation complexity alone cannot imply it ---
and regard $\cF_{\mathrm{HV}}$ as the greedy-accessible analogue.


\section{Related works}\label{appx:related}

In this section, we situate Thm.~\ref{thm:minimax} within the contemporary literature on convergence rates for decision trees and random forests (see Table~\ref{tab:related} for a summary).

\paragraph{Non-adaptive vs.\ greedy ensembles.}
As noted in the main text, several works establish minimax optimality for \emph{non-adaptive} tree ensembles, including Mondrian forests \citep{mondrian_forest, mourtada2017, mourtada2020} and random tessellation forests \citep{oreill2022, oreilly2024}. Because their partitions are drawn independently of the responses $Y$, these analyses reduce to kernel or local-averaging theory on fixed random meshes and do not explain how or why data-driven, greedy splitting succeeds in practice. Closer to the adaptive setting, \citet{Zhang2024} obtain minimax rates for modified RFs by enforcing adaptive split balancing (constraining candidate splits to remain geometrically balanced while adapting to smoothness) followed by leaf-wise regressions. By contrast, our focus is on unmodified, best-first CART splitting with ordinary leaf-constant predictions, isolating the spectral and algorithmic conditions under which standard greedy impurity reduction attains minimax risk.

\paragraph{Globally optimal trees vs.\ greedy CART.}
A second line of work studies \emph{globally optimal} decision trees---estimators that minimize empirical risk over the space of all binary trees with at most $M$ leaves \citep{tan2022, xu2026}. Notably, \citet{xu2026} establish minimax-optimal rates (up to logarithmic factors) for globally optimal decision trees over broad anisotropic and spatially inhomogeneous function classes. However, computing a globally optimal decision tree is NP-complete in general \citep{hyafil1976}. We distinguish between such global tree optimizers and greedy CART---both in its \emph{population} form (which accesses exact conditional expectations for candidate splits) and its \emph{empirical} form (which estimates split gains from $\mathcal{D}_N$). Thm.~\ref{thm:minimax} shows that greedy CART matches the exact minimax rate in the population regime and the minimax rate up to logarithmic factors in the empirical regime, trading intractable global search over partitions for the amortized spectral gain condition \eqref{eq:ASG}/\eqref{eq:eASG}.

\paragraph{Adaptive CART and random forests.}
Several authors have analyzed greedy CART and random forest variants under structural assumptions on the regression function. Notable examples include:
\begin{itemize}[nosep]
    \item \citet{Scornet2015} prove the first $L^2$-consistency guarantee for Breiman's original RF algorithm under an additive regression model, showing adaptation to sparse active coordinates without deriving explicit finite-sample rates.
    \item \citet{syrg2020} analyze greedy CART with binary features under a sparsity constraint and a structural impurity decrease (SID) condition, obtaining MSE bounds that scale logarithmically in the ambient dimension $d$ and establishing valid asymptotic inference for honest forests.
    \item \citet{klusowski_cart} studies a single tree with cost-complexity pruning under SID and shows that CART adapts to latent low-dimensional structure in sparse additive models at logarithmic depth.
    \item \citet{chi2022} extend the SID-based analysis to forests built from the empirical CART criterion, deriving high-dimensional convergence rates while allowing for discontinuous $f^\ast$ and dependent covariates.
    \item \citet{Klusowski2024} establish universal consistency and explicit high-dimensional rates for CART and RFs on additive models under an $\ell_1$-variation sparsity constraint---even when $d$ grows exponentially in $N$---without imposing a local-gain condition or covariate independence.
\end{itemize}

Our framework differs from these works in three respects. First, rather than restricting $f^\ast$ to an additive or binary-feature model, we work on a continuous domain $\mathcal{X} = [0,1]^d$ and define the target class $\mathcal{F}$ directly through the spectral source condition (A2) and eigenvalue decay (A3) of the reference forest operator $\bm K_\Pi$, which naturally accommodates multi-way interactions (Appx.~\ref{sec:ex-packing}). 
This inductive bias is arguably more appropriate to tree- and forest-based estimators.
Second, whereas \citet{syrg2020}, \citet{klusowski_cart}, and \citet{chi2022} rely on a one-step SID inequality, our assumption (A8) requires only \emph{blockwise} amortized progress across $\mathcal{O}(\mathfrak{C}_m(\varepsilon))$ splits; as shown in Prop.~\ref{prop:SID}, pointwise SID is a special case that implies ASG. Third, we establish matching upper and minimax lower bounds ($N^{-s/(s+1)}$) over $\mathcal{F}$.

\paragraph{When do random forests fail? Gating and \citet{tang2018}.}
\citet{tang2018} identify structural regimes in which random forests fail to achieve optimal performance, notably when signal is concentrated in high-order interactions whose marginal projections vanish, or when uninformative splits dilute leaf sample sizes. Our gating counterexample in Lemma~\ref{lem:gating} provides a sharp spectral formulation of the first phenomenon: even when $f^\ast$ lies inside the spectral source ball $\mathcal{F}$, a zero coarse contrast ($c_{v_0} = 0$) can act as a ``gate'' that hides large-scale descendant signal behind an arbitrarily small positive distraction on another branch. This explains why spectral smoothness (A2)--(A4) alone cannot guarantee greedy convergence and why an algorithmic visibility condition such as ASG (A8) or Hierarchical Visibility (Def.~\ref{def:hv}) is indispensable.

\begin{table}[t] 
  \centering
  \caption{Comparison of theoretical guarantees for decision tree and random forest variants.}
  \label{tab:related}
  \footnotesize
  \begin{tabular*}{\textwidth}{@{}lccll@{}}
    \toprule
    Work & CART basis & Local gain & Target class & Guarantee \\
    \midrule
    \citet{mourtada2020}       & \xmark & \xmark & H\"older                        & minimax \\
    \citet{oreilly2024}        & \xmark & \xmark & H\"older                        & minimax \\
    \citet{xu2026}             & \xmark & \xmark & Besov                           & minimax (up to $\log$) \\
    \citet{Zhang2024}          & \xmark & \xmark & H\"older                        & minimax \\
    \citet{Scornet2015}        & \cmark & \xmark & sparse, additive                & consistency \\
    \citet{syrg2020}           & \cmark & \cmark & sparse, binary $X$              & convergence rate \\
    \citet{chi2022}            & \cmark & \cmark & sparse, additive                & convergence rate \\
    \citet{klusowski_cart}     & \cmark & \cmark & sparse, additive (1 tree)       & rate / adaptivity \\
    \citet{Klusowski2024}      & \cmark & \xmark & sparse, additive ($\ell_1$ BV)  & universal consistency \\
    \citet{tan2022}            & \cmark & \xmark & sparse, additive ($C^1$)        & tree lower bound \\
    \textbf{This work (Thm.~\ref{thm:minimax})} & \cmark & \cmark (ASG) & spectral source class $\mathcal{F}$ & minimax (oracle; $\log$ emp.) \\
    \bottomrule
  \end{tabular*}
\end{table}

\paragraph{Reconciling with the lower bound of \citet{tan2022}.}
At first glance, Thm.~\ref{thm:minimax} might appear to be in tension with \citet{tan2022}, who prove a ``cautionary tale'' showing that axis-aligned, leaf-averaging trees \emph{cannot} attain the univariate minimax rate for sparse additive models. This apparent discrepancy dissolves once we distinguish between the ambient additive class and the tree-adapted spectral class $\mathcal{F}$.

Specifically, \citet{tan2022} consider a sparse additive target $f^\ast(\bm x) = \sum_{j \in S} \phi_j(x_j)$ with $|S| = d_0$ active coordinates and $C^1$ components. The information-theoretic minimax rate over this additive class is $N^{-2/3}$ (with smoothness $s = 2$ in squared-error scaling), reflecting an effective dimension of one \citep{raskutti2012}. However, \citet{tan2022} prove that no single axis-aligned leaf-averaging tree---even with an oracle partition---can improve upon $N^{-2/(d_0+2)}$: piecewise-constant leaf averages cannot exploit additive separability and must instead tensorize across all $d_0$ active coordinates, requiring $M \asymp r^{d_0}$ leaves to achieve a mesh of order $r^{-1}$ in each active direction.

In our framework, this geometric bottleneck is captured directly by the superlinear spectral tree-complexity regime of Remark~\ref{rem:classical-tree-rates} and Prop.~\ref{prop:gamma}. Achieving squared approximation error $m^{-s}$ ($s=2$) simultaneously across $d_0$ active coordinates requires $\mathfrak C_m(\varepsilon) \asymp m^\gamma$ leaves with $\gamma = d_0$, which yields the exact oracle tree rate of \citet{tan2022}:
\[
    N^{-\frac{s}{s + \gamma}} = N^{-\frac{2}{d_0 + 2}}.
\]
Viewed in terms of the leaf budget $M \asymp m^{d_0}$, a partition that resolves $d_0$ active coordinates to accuracy $M^{-2/d_0}$ also accommodates $d_0$-way hierarchical interactions at that scale. Thus the sparse additive class is strictly contained in a much larger space of $d_0$-dimensional tree-approximable functions, and enlarging the class slows the achievable rate from $N^{-2/3}$ to $N^{-2/(d_0+2)}$. Crucially, this bottleneck arises entirely from the single-tree projection bias $\norm{(\bm I - \bm P_T)f^\ast}_{L^2}^2$---the one term in our risk decomposition that does not shrink with ensemble size $B$ when all trees share the same partition geometry. Consequently, increasing $B$ reduces estimation variance but cannot overcome the representational barrier identified by \citet{tan2022}.


\section{Experimental Setup}\label{appx:exp}

RFs and GBMs were trained using the \texttt{scikit-learn} package. Exact matrix decompositions were computed with \texttt{numpy}, while approximate decompositions utilized \texttt{scipy}. The SCATE model was implemented within the \texttt{torch} framework. Finally, standard Python libraries were used to evaluate metrics such as model size and inference time. All experiments were conducted in one of two setups: locally, on a Macbook Pro with an M5 Pro chip and 24GB memory, or on high performance computing, with 12 AMD EPYC 7282 CPUs, 64GB RAM, and an NVIDIA A30 chip. 

While the complexity of our proposed method is independent of the base model's size, the majority of the pruning techniques we benchmark against extract information directly from the ensemble's internal rules and structure. Consequently, employing an overly large model would unfairly increase the compression difficulty for baseline methods, while offering us an advantage at no cost. To ensure a fair comparison, we standardized the hyperparameters across all RFs and GBMs, aligning them closely with the configurations used in \citep{devos2025}. Specifically, RFs were configured with \texttt{n\_estimators=250}, \texttt{max\_depth=15}, while GBMs were trained with \texttt{n\_estimators=100}, \texttt{max\_depth=6}. For GBMs, we additionally set \texttt{base\_score=0} to be correct to the definition of the smoother matrix $\bm{S}$. We detail the construction of each respective matrix below, in Algorithm \ref{alg:rf-kernel} and Algorithm \ref{alg:gbm-smoother}, for completeness.
\begin{algorithm}[h]
\caption{Construction of the RF kernel matrix}
\label{alg:rf-kernel}
\begin{algorithmic}[1]
\REQUIRE
    Fitted forest with $B_{\mathrm{RF}}$ trees;
    reference points $X=(x_1,\ldots,x_n)$;
    evaluation points $X_*=(x_1^*,\ldots,x_m^*)$.
\ENSURE
    Reference kernel $\mathbf K_{\mathrm{RF}}\in\mathbb R^{n\times n}$
    and evaluation kernel $\mathbf K_*\in\mathbb R^{m\times n}$.

\STATE Initialize
$\mathbf K_{\mathrm{RF}}\leftarrow\mathbf 0_{n\times n}$
and $\mathbf K_*\leftarrow\mathbf 0_{m\times n}$.

\FOR{$b=1,\ldots,B_{\mathrm{RF}}$}

    \STATE Let $\ell_b(x)$ denote the leaf containing $x$ in tree $b$.

    \STATE For each leaf $\ell$, compute its reference-sample count
    \[
        N_{b,\ell}
        =
        \sum_{j=1}^{n}
        \mathbf 1\{\ell_b(x_j)=\ell\}.
    \]

    \STATE Define the normalized leaf weight
    \[
        k_b(x,x_j)
        =
        \begin{cases}
        \displaystyle
        \frac{\mathbf 1\{\ell_b(x)=\ell_b(x_j)\}}
             {N_{b,\ell_b(x)}},
        & N_{b,\ell_b(x)}>0,\\[1ex]
        0, & N_{b,\ell_b(x)}=0.
        \end{cases}
    \]

    \STATE Construct
    $[\mathbf K_b]_{ij}=k_b(x_i,x_j)$ and
    $[\mathbf K_{b,*}]_{ij}=k_b(x_i^*,x_j)$.

    \STATE Set $w_b=1/B_{\mathrm{RF}}$ for an ordinary RF;
    for a weighted forest, use its stored tree weight $w_b$.

    \STATE Update
    \[
    \begin{aligned}
        \mathbf K_{\mathrm{RF}}
        &\leftarrow \mathbf K_{\mathrm{RF}}+w_b\mathbf K_b,\\
        \mathbf K_*
        &\leftarrow \mathbf K_*+w_b\mathbf K_{b,*}.
    \end{aligned}
    \]

\ENDFOR

\STATE \textbf{return} $\mathbf K_{\mathrm{RF}},\mathbf K_*$.
\end{algorithmic}
\end{algorithm}

\begin{algorithm}[h]
\caption{Construction of the GBM smoother matrix}
\label{alg:gbm-smoother}
\begin{algorithmic}[1]
\REQUIRE
    Fitted squared-error GBM with $B_{\mathrm{GBM}}$ ordered trees
    and learning rate $\eta$; reference points
    $X=(x_1,\ldots,x_n)$; evaluation points
    $X_*=(x_1^*,\ldots,x_m^*)$.
\ENSURE
    Reference smoother $\mathbf S_{\mathrm{GBM}}\in\mathbb R^{n\times n}$
    and evaluation smoother $\mathbf S_*\in\mathbb R^{m\times n}$.

\STATE Initialize
$\mathbf S_0\leftarrow\mathbf 0_{n\times n}$
and $\mathbf S_{*,0}\leftarrow\mathbf 0_{m\times n}$.
\STATE This initialization corresponds to $f_0(x)=0$.

\FOR{$b=1,\ldots,B_{\mathrm{GBM}}$}

    \STATE Construct the normalized leaf-membership matrices
    $\mathbf K_b\in\mathbb R^{n\times n}$ and
    $\mathbf K_{b,*}\in\mathbb R^{m\times n}$ for tree $b$,
    as in Algorithm~\ref{alg:rf-kernel}.

    \STATE Compute the residual operator
    \[
        \mathbf R_{b-1}
        \leftarrow
        \mathbf I_n-\mathbf S_{b-1}.
    \]

    \STATE Update both smoothers using the same residual operator:
    \[
    \begin{aligned}
        \mathbf S_b
        &\leftarrow
        \mathbf S_{b-1}+\eta\mathbf K_b\mathbf R_{b-1},\\
        \mathbf S_{*,b}
        &\leftarrow
        \mathbf S_{*,b-1}+\eta\mathbf K_{b,*}\mathbf R_{b-1}.
    \end{aligned}
    \]

\ENDFOR

\STATE Set
$\mathbf S_{\mathrm{GBM}}\leftarrow\mathbf S_{B_{\mathrm{GBM}}}$
and $\mathbf S_*\leftarrow\mathbf S_{*,B_{\mathrm{GBM}}}$.

\STATE \textbf{return} $\mathbf S_{\mathrm{GBM}},\mathbf S_*$.
\end{algorithmic}
\end{algorithm}
For our experiments, we assembled small-to-medium datasets from three sources: the UCI Machine Learning Repository \citep{Dua:2019}, OpenML \citep{OpenML2013}, and Kaggle. To equitably demonstrate the performance of all methods, these datasets were chosen to provide a balanced mix of classification and regression tasks. During preprocessing, we remove all missing values and exclude temporal features, which tree-based ensembles cannot handle natively. Detailed dataset characteristics---including the number of samples, feature types, and task categories---are summarized in Table \ref{tab:datasets}. For all experiments, we took a train-test split of 75-25; in cases where a validation set was needed, we took an additional train-validation split of 90-10 (in this case, the train-validation-test split would be 67.5-7.5-25).

\begin{table}[htbp] \small
    \centering
    \caption{Summary of the local experiment datasets. Feature counts exclude the response and describe the input columns before one-hot encoding. Numerical and categorical counts follow the data types used by the experiment loaders; integer-coded predictors treated numerically are counted as numerical.}
    \begin{adjustbox}{max width=\textwidth}
    \begin{tabular}{l l r r r r l}
        \toprule
        \textbf{Dataset} & \textbf{Code} & \textbf{\#Samples} & \textbf{\#Numerical} & \textbf{\#Categorical} & \textbf{\# Total} & \textbf{Task} \\
        \midrule
        \href{https://archive.ics.uci.edu/dataset/1/abalone}{Abalone} & \texttt{abalone} & 4177 & 7 & 1 & 8 & Regression \\
        \href{https://archive.ics.uci.edu/dataset/9/auto+mpg}{Auto MPG} & \texttt{auto\_mpg} & 392 & 7 & 0 & 7 & Regression \\
        \href{https://archive.ics.uci.edu/dataset/275/bike+sharing+dataset}{Bike Sharing} & \texttt{bike} & 17379 & 12 & 0 & 12 & Regression \\
        \href{https://www.cs.toronto.edu/\%7Edelve/data/boston/bostonDetail.html}{Boston Housing} & \texttt{boston} & 506 & 12 & 0 & 12 & Regression \\
        \href{https://scikit-learn.org/stable/modules/generated/sklearn.datasets.fetch_california_housing.html}{California Housing} & \texttt{california} & 20640 & 8 & 0 & 8 & Regression \\
        \href{https://archive.ics.uci.edu/dataset/165/concrete+compressive+strength}{Concrete} & \texttt{concrete} & 1030 & 8 & 0 & 8 & Regression \\
        \href{https://ggplot2.tidyverse.org/reference/diamonds.html}{Diamonds} & \texttt{diamonds} & 53940 & 6 & 3 & 9 & Regression \\
        \href{https://scikit-learn.org/stable/api/sklearn.datasets.html}{Friedman} & \texttt{friedman\_1} & 1000 & 10 & 0 & 10 & Regression \\
        \href{https://archive.ics.uci.edu/dataset/320/student+performance}{Student Performance} & \texttt{student} & 649 & 15 & 17 & 32 & Regression \\
        \href{https://archive.ics.uci.edu/dataset/186/wine+quality}{Wine Quality} & \texttt{wq} & 4898 & 11 & 0 & 11 & Regression \\
        \midrule
        \href{https://archive.ics.uci.edu/dataset/2/adult}{Adult} & \texttt{adult} & 45222 & 5 & 8 & 13 & Classification \\
        \href{https://www.openml.org/search?type=data\&sort=runs\&id=1462\&status=active}{Banknote Auth.} & \texttt{banknote} & 1372 & 4 & 0 & 4 & Classification \\
        \href{https://archive.ics.uci.edu/dataset/17/breast+cancer+wisconsin+diagnostic}{Breast Cancer} & \texttt{bc} & 569 & 30 & 0 & 30 & Classification \\
        \href{https://www.openml.org/search?type=data\&status=active\&id=46911}{Bank Customer Churn} & \texttt{churn} & 10000 & 8 & 2 & 10 & Classification \\
        \href{https://archive.ics.uci.edu/dataset/34/diabetes}{Diabetes} & \texttt{diabetes} & 768 & 8 & 0 & 8 & Classification \\
        \href{https://archive.ics.uci.edu/dataset/159/magic+gamma+telescope}{MAGIC Gamma Telescope} & \texttt{magic} & 19020 & 10 & 0 & 10 & Classification \\
        \href{https://archive.ics.uci.edu/dataset/94/spambase}{Spambase} & \texttt{spambase} & 4601 & 57 & 0 & 57 & Classification \\
        \href{https://www.kaggle.com/datasets/blastchar/telco-customer-churn}{Telco Churn} & \texttt{telco} & 7032 & 4 & 15 & 19 & Classification \\
        \bottomrule
    \end{tabular}
    \end{adjustbox}
    \label{tab:datasets}
\end{table}

We describe the setup used for the experiments in Sect. \ref{sec:exp}. All experiments are run with 10 repetitions each (for every method/dataset combination, with different data splitting and model seeds).

\subsection{Spectra and Generalization} To evaluate the general performance of SCATE, we use all datasets except \texttt{adult} and \texttt{california}. Because our theoretical analysis is limited to RFs, this experiment focuses exclusively on them. For each dataset, we performed the following steps:

\begin{itemize}[noitemsep]
    \item After splitting the data (train, validation, test) and fitting an RF to the training set, we constructed the RF kernel $\bm{K}$ and computed its exact eigendecomposition. The resulting eigenvalues were sorted to generate a decay curve. We also computed the cross-kernel matrices $\bm{K}_{train, test}$ and $\bm{K}_{train, validation}$.
    \item We evaluate varying numbers of learned eigenfunctions, $P$, using the sequence \texttt{[2, 5, 10, 20, 30, 40, 50, 60, 70, 80, 90, 100, 150, 175, 200]}. We capped $P$ at 200, as subsequent eigenfunctions correspond to very small eigenvalues and contain minimal information.
    \item \textbf{Oracle Baseline:} For each value of $P$, we calculate the Oracle baseline (Sect.~\ref{sec:compression}) by computing its approximation of $\bm{K}_{train, test}$ and its resulting predictions on the test data.
    \item For each $P$, we train two separate SCATE models (Sect.~\ref{sec:compression}) for 200 epochs to learn the eigenfunction mappings: a small model (\texttt{width=16}, \texttt{depth=2}) and a large model (\texttt{width=128}, \texttt{depth=4}).
    \item The test data is passed through each model to generate output eigenfunctions. These are used to compute the approximated kernel $\hat{\bm{K}}_{train\_test}$, which was then multiplied by the training labels $y$ to produce predictions. Performance is measured using accuracy for classification and $R^2$ for regression.
\end{itemize}

The resulting visualizations from this process are presented in Fig.~\ref{fig:visualisation}.

\subsection{Na\"ive Benchmark}In this experiment, we evaluate SCATE against two simple distillation baselines for RFs:
\begin{itemize}[noitemsep]
    \item \textbf{Na\"ive MLP:} This approach uses the identical network architecture as SCATE. However, instead of learning eigenfunctions, it is trained to directly predict the output of each tree in the base model. This direct distillation scheme is inspired by \citep{hinton_distilling2}.
    \item \textbf{Na\"ive RF:} A naturally smaller RF created by restricting the \texttt{max\_depth} and \texttt{n\_estimators} hyperparameters. Because default \texttt{scikit-learn} objects store metadata beyond what is needed for inference, we extract only the essential matrices into a custom \texttt{minimal\_RF} object to ensure an accurate model size comparison.
\end{itemize}
We then benchmark SCATE against the Na\"ive MLP, the Na\"ive RF, and the Base RF, on all datasets excluding \texttt{adult} and \texttt{california}, using the following pipeline:
\begin{itemize}[noitemsep]
    \item We split the data, train the baseline RF, and construct the kernel exactly as described in the visualization experiment. We additionally record the performance of the baseline RF, as a measure of the performance we expect to converge to for all models.
    \item Based on earlier visualization results (Fig.~\ref{fig:visualisation}) showing that SCATE learns most effectively up to 50 eigenfunctions, we fix $P=50$. Instead of varying $P$, we sweep the network architectures for SCATE and the Na\"ive MLP using combinations of \texttt{widths=[4, 8, 16, 32, 64, 128]} and \texttt{depths=[1, 2, 3, 4, 5]}. 
    \item Concurrently, we evaluate the Na\"ive RF baselines across combinations of \texttt{n\_estimators=[10, 50, 100, 200, 500]} and \texttt{max\_depth=[3, 4, 5, 6, 7, 8, 9, None]}.
    \item We train the models on their respective targets (eigenfunctions for SCATE, tree outputs for the MLP, and raw data for the Na\"ive RF). We then record both the predictive performance and the final model size for every configuration.
    \item For each dataset and method, we isolate the best-performing configurations at different model sizes to construct a convex Pareto frontier. Because model sizes scale differently across methods, some curves may contain very few data points if a majority of the configurations fall below the optimal frontier.  
\end{itemize}

The resulting Pareto curves are plotted in Fig.~\ref{fig:naive}.

\subsection{Compression Benchmark}
For the compression benchmark, we test how well different methods perform compression. We focus on SCATE and FP-SCATE (which is SCATE applied to a model already shrunk by ForestPrune) under two strict memory limits: 10KB and 100KB. Unlike the earlier experiments, we also include GBMs here to compare performance.

Given all datasets in the pipeline, our benchmark proceeds as follows:
\begin{itemize}[noitemsep]
    \item Similar to the earlier tests, we start by splitting the data, training a standard Random Forest, and creating the kernel matrix $\bm{K}$ for each dataset.
    \item We test a wide range of settings (hyperparameters) for each method to see how they balance model size against accuracy. For some models, this amounts to an explicit accuracy loss bound --- for others, tuning occurs by varying some regularisation hyperparameter, which would implicitly affect the loss in accuracy. For every combination of settings and dataset, we record the compressed model's size and performance. Since different algorithms output different types of models, we measure size fairly by only counting size of the bare minimum objects needed to make predictions.
    \item For each method and dataset, we find the best-performing model that fits under the 10KB and 100KB limits. The score of this best model becomes our final result for that budget, and we summarize these results  in Table~\ref{tab:rf}.
    \item Finally, we repeat this exact same process for the GBM models, constructing a smoother matrix $\bm{S}$ instead of an RF kernel. The results for the GBM tests are shown in Table~\ref{tab:gbm}.
\end{itemize}

Each method ends up being run around $50$ times for each dataset and seed, and the runs are doubled, since we evaluate on both RFs and GBMs, resulting in $50 \times 18 \times 10 \times 2 = 15000$ runs for each method, and we benchmark $7$ methods in total, resulting in $\sim 10^5$ runs in total. To address these computational demands, we run the compression benchmark from a high performance computing partition. Other experiments were run from a Macbook Pro, with an M5 Pro chip with 18-core CPU and 20-core GPU, and 24GB RAM. In order to accurately compare training and inference time between methods, we conduct a small runtime experiment on this machine, which can be found in Appx. \ref{appx:runtime}

This section details the implementation of each method included in our compression benchmark. For all baseline methods, we outline any structural modifications required to adapt them to our experimental framework, alongside the hyperparameter grids searched to identify the optimal model for each dataset.
\begin{itemize}
    \item \textbf{SCATE:} Following the setup used in our na\"ive benchmark, we fix $P=50$ and evaluate configurations across \texttt{widths} $\in \{4, 8, 16, 32, 64, 128\}$ and \texttt{depths} $\in \{1, 2, 3, 4, 5\}$. For GBM compression, because our smoother matrix $\textbf{S}$ is strictly defined for MSE loss, we train a regressor for binary classification tasks and interpret the continuous outputs as linear probabilities.
    \item \textbf{FP:} The ForestPrune algorithm \citep{liu2023} lacks native support for classification tasks. To address this, we apply the same workaround used in our own models by fitting a regressor for binary classification. We identify the optimal model by sweeping a hyperparameter grid generated via \texttt{np.flip(np.logspace(-2, 1.5, 50))}, mirroring the configuration in the algorithm's official implementation.
    \item \textbf{FP-SCATE:} This pipeline sequentially applies FP followed by SCATE. We first isolate the best-performing FP model using the hyperparameter sweep described above. We then compress that specific model using SCATE, evaluating it over the same hyperparameter grid as base SCATE.
    \item \textbf{LOP:} For the LOP method \citep{devos2025}, we search for the best model by sweeping tolerance values continuously from 0.01 to 1. To manage computational overhead, we impose a strict 10-minute timeout per run. In practice, this limit was only triggered a handful of times on the most computationally demanding datasets.
    \item \textbf{FIRE and TE:} Similar to FP, neither the FIRE \citep{liu2023fire} nor TreeExtract \citep{liu2025extracting} algorithms natively handle compression for classification. Consequently, we employ the same regressor workaround for binary classification tasks. For FIRE, we sweep the $\alpha$ hyperparameter across values from 10 to 100. For TE, we evaluate models across a $\lambda$ range defined by \texttt{np.flip(np.logspace(1, 4, 50))}.
    \item \textbf{SSF:} Because the SSF algorithm \citep{alkhoury2024splitting} is designed exclusively for compressing Random Forests, we exclude it from our GBM benchmarks. We determine the optimal SSF models by sweeping threshold values between 0 and 0.5, which covers the entire hyperparameter range specified in the original paper.
\end{itemize}
\subsection{SCATE Size Specification}\label{appx:scate_size}
\paragraph{SCATE model size and architecture selection.}
Let $d$ denote the input dimension after preprocessing, $w$ the
hidden-layer width, $L\geq 1$ the number of hidden layers, and
$m$ the number of spectral outputs. All hidden layers have width $w$.
The network contains
\begin{equation}
\label{eq:scate-parameter-count}
    P(d,w,L,m)
    =
    \underbrace{w(d+1)}_{\text{input layer}}
    +
    \underbrace{(L-1)w(w+1)}_{\text{hidden layers}}
    +
    \underbrace{m(w+1)}_{\text{output layer}}
\end{equation}
parameters, including a bias for each hidden and output unit.

For the scalar regression and binary classification experiments, the
prediction also uses a vector of $m$ spectral coefficients. The
network parameters are stored as 32-bit floating-point tensors; the
coefficient vector is a 64-bit floating-point NumPy array in the
reported implementation. Thus, their numerical payload requires
\begin{equation}
\label{eq:scate-payload-size}
    B_{\mathrm{payload}}
    =
    4P(d,w,L,m)+8m
    \qquad\text{bytes}.
\end{equation}

The size recorded in our benchmark additionally includes
implementation-dependent overhead. Specifically, it is measured as
\begin{equation}
\label{eq:scate-measured-size}
    B_{\mathrm{measured}}
    =
    \left|
    \operatorname{torch.save}(\operatorname{state\_dict}(\theta))
    \right|
    +
    \operatorname{getsizeof}(\mathbf c),
\end{equation}
where the first term is the number of bytes in the serialized network
state dictionary and $\mathbf c\in\mathbb R^m$ is the coefficient array.
Equivalently, for the data types above,
\[
    B_{\mathrm{measured}}
    =
    4P(d,w,L,m)+8m
    +H_{\mathrm{torch}}+H_{\mathrm{array}},
\]
where $H_{\mathrm{torch}}$ is the state-dictionary serialization
overhead and $H_{\mathrm{array}}$ is the additional array-object
overhead counted by \texttt{getsizeof}. These overheads depend on the
software environment and representation.

Given a byte budget $B$, a user can enumerate candidate values of
$(w,L,m)$, construct the corresponding network and coefficient array,
and retain configurations satisfying
$B_{\mathrm{measured}}\leq B$. A configuration can then be selected
among the feasible candidates using validation performance. The
parameter-count formula provides an initial screening rule, while
direct measurement accounts for serialization overhead.
\section{Additional Results}\label{appx:extra_results}
We present additional results from the visualization experiment and the na\"ive benchmark, experiments comparing different methods of learning eigenfunctions, and runtime experiments.
\subsection{Evaluating Different Methods for Learning Eigenfunctions}\label{appx:neural_exp} 
We evaluate and compare the performance of two existing neural eigenfunction estimation approaches—NeuralEF \citep{deng2022neuralef} and NeuralSVD \citep{ryu2024operator}—against our proposed method, SCATE. Our evaluation follows two experimental setups: first, a capacity-constrained setting similar to the na\"ive benchmark, where we fix the model size and vary the number of target eigenfunctions $P$; second, a budget-constrained setting similar to the compression benchmark, where we fix $P$ and vary model sizes to identify the optimal configuration. Due to the relatively high computational demands of NeuralEF and NeuralSVD, we restrict this comparative analysis to five small datasets: \texttt{auto\_mpg}, \texttt{student}, \texttt{boston}, \texttt{bc}, and \texttt{diabetes}, serving as an illustrative side-by-side comparison. For SCATE, we maintain the experimental settings detailed in Sect. \ref{sec:compression} and Appx. \ref{appx:exp}. 

Because the primary contributions of NeuralEF and NeuralSVD are novel learning objectives rather than specific model architectures, their objectives can theoretically be adapted to any neural network. However, to ensure a fair comparison, we utilize the network designs specified in each method's official implementation. We briefly summarize their methodologies below:

\begin{itemize}
    \item \textbf{NeuralEF:} NeuralEF formulates kernel eigendecomposition as a series of asymmetric maximization problems. Because this formulation relies on the deflation of previously computed modes, the optimization exhibits a strict hierarchical dependency where latter eigenfunctions depend on the convergence of earlier ones. While their original objective is a constrained optimization problem, in practice, NeuralEF relaxes orthogonality into explicit penalty terms within the loss function and enforces unit-norm requirements via a $L^2$ batch normalization layer. This structure requires training a neural network for each of the $P$ eigenfunctions it aims to learn. Furthermore, because its objective is fundamentally rooted in maximizing the Rayleigh quotient, NeuralEF is strictly confined to self-adjoint operators. As a result, it cannot compute the singular value decomposition (SVD) of kernels, limiting its applicability in our context exclusively to Random Forest (RF) compression. 
    
    \item \textbf{NeuralSVD:} In contrast, NeuralSVD frames spectral decomposition as a fully unconstrained optimization task based on Schmidt’s low-rank approximation theorem. By introducing a "NestedLoRA" objective, it recovers the ordered singular functions by breaking the symmetry of the low-rank approximation. This approach provides several improvements over NeuralEF. First, NeuralSVD eliminates the need for explicit orthogonality penalties and normalization layers; instead, it learns unnormalized singular functions that naturally absorb the scale of their corresponding singular values. Second, rather than using a sequentially trained network, NeuralSVD employs a joint nesting technique that unifies the learning of all $P$ modes into a single global objective function. This allows the model to simultaneously and optimize all functions within a single, shared neural network, similar to SCATE. Finally, because its objective natively supports SVD, NeuralSVD is capable of compression of both RFs and Gradient Boosting Machines (GBMs).
\end{itemize}

\begin{table}[htbp]
\centering
\caption{RF compression benchmark for the three eigenfunction learning methods, averaged over $10$ trials. Best result per row in bold. NA indicates that the model could not reach the target size.}
\label{tab:neural_rf}
\setlength{\tabcolsep}{3pt}
\renewcommand{\arraystretch}{1.1}
\resizebox{\textwidth}{!}{
\begin{tabular}{lc|ccc|ccc}
\toprule
 & & \multicolumn{3}{c|}{\textbf{Size = 10KB}} & \multicolumn{3}{c}{\textbf{Size = 100KB}} \\
\cmidrule{3-8}
\textbf{Dataset} & \textbf{RF} & \textbf{NeuralEF} & \textbf{NeuralSVD} & \textbf{SCATE} & \textbf{NeuralEF} & \textbf{NeuralSVD} & \textbf{SCATE} \\
\midrule
\texttt{auto\_mpg} & 0.69 & 0.086 $\pm$ 0.011 & 0.237 $\pm$ 0.020 & \textbf{0.453 $\pm$ 0.012} & 0.192 $\pm$ 0.005 & 0.320 $\pm$ 0.009 & \textbf{0.497 $\pm$ 0.009} \\
\texttt{bc} & 0.96 & 0.732 $\pm$ 0.018 & 0.804 $\pm$ 0.028 & \textbf{0.984 $\pm$ 0.003} & 0.881 $\pm$ 0.016 & 0.949 $\pm$ 0.008 & \textbf{0.987 $\pm$ 0.003} \\
\texttt{boston} & 0.85 & 0.149 $\pm$ 0.011 & 0.220 $\pm$ 0.015 & \textbf{0.775 $\pm$ 0.013} & 0.149 $\pm$ 0.011 & 0.252 $\pm$ 0.013 & \textbf{0.814 $\pm$ 0.012} \\
\texttt{diabetes} & 0.76 & 0.654 $\pm$ 0.011 & 0.658 $\pm$ 0.007 & \textbf{0.786 $\pm$ 0.010} & 0.654 $\pm$ 0.011 & 0.660 $\pm$ 0.007 & \textbf{0.794 $\pm$ 0.010} \\
\texttt{student} & 0.85 & NA & 0.116 $\pm$ 0.013 & \textbf{0.854 $\pm$ 0.014} & 0.134 $\pm$ 0.007 & 0.337 $\pm$ 0.018 & \textbf{0.857 $\pm$ 0.013} \\
\midrule
\textbf{Average Rank} &    & 3.00& 2.00& \textbf{1.00}    & 3.00 & 2.00 & \textbf{1.00}  \\
\bottomrule
\end{tabular}
}
\end{table}

\begin{figure}[htbp]
    \centering
    \includegraphics[width=0.9\linewidth]{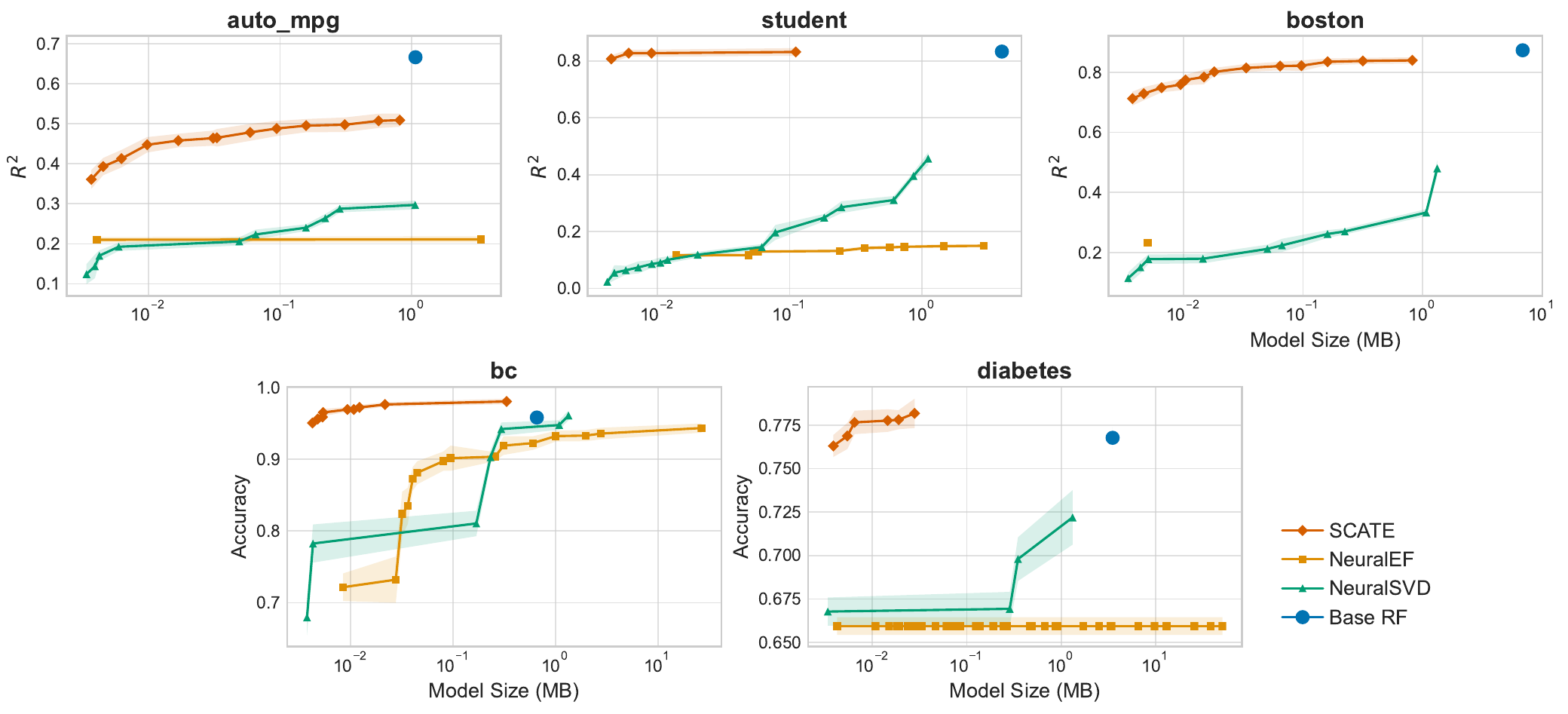}
    \caption{Rate-performance Pareto curves across five benchmarkd datasets for three eigenfunction learning methods. Shaded regions indicate standard errors.}
    \label{fig:naive_neural}
\end{figure}

From these results, we can see SCATE dominating compression performance. There are several reasons that may have caused subpar performance from the two competing methods, despite their theoretical complexity. 

First, there is a fundamental difference in their optimization objectives: while NeuralEF and NeuralSVD are designed for mathematically exact spectral decomposition, SCATE is designed for kernel reconstruction. This can make a big difference in compression settings, where a network's learning capacity is heavily limited. In particular, both NeuralEF and NeuralSVD dedicate a substantial proportion of their capacity towards learning eigenvectors that are strictly orthogonal. This objective may be too complex for the limited learning capacity of small networks and is better suited for more expressive models. 

In contrast, SCATE treats orthogonality as a soft penalty. Since top eigencomponents have strong signals, we do not necessarily need the orthogonality property to learn the top eigenfunctions effectively. Ultimately, this performance gap does not imply that NeuralEF or NeuralSVD are inferior methods; rather, their objectives and motivations are more appropriate for different settings than our own. As a consequence, we select SCATE as the primary method for learning the eigenfunctions in our RF compression experiments.

\subsection{Additional Visualization Results}\label{appx:extra_vis}
Complementing the two-dataset visualization in Fig. \ref{fig:visualisation}, Fig. \ref{fig:vis_full_reg} (regression) and Fig. \ref{fig:vis_rull_cl} (classification) show the eigenvalue decay, kernel approximation error, and predictive performance for the thirteen remaining datasets in the benchmark. 

\begin{figure}[htbp]
    \centering
    \includegraphics[width=\linewidth]{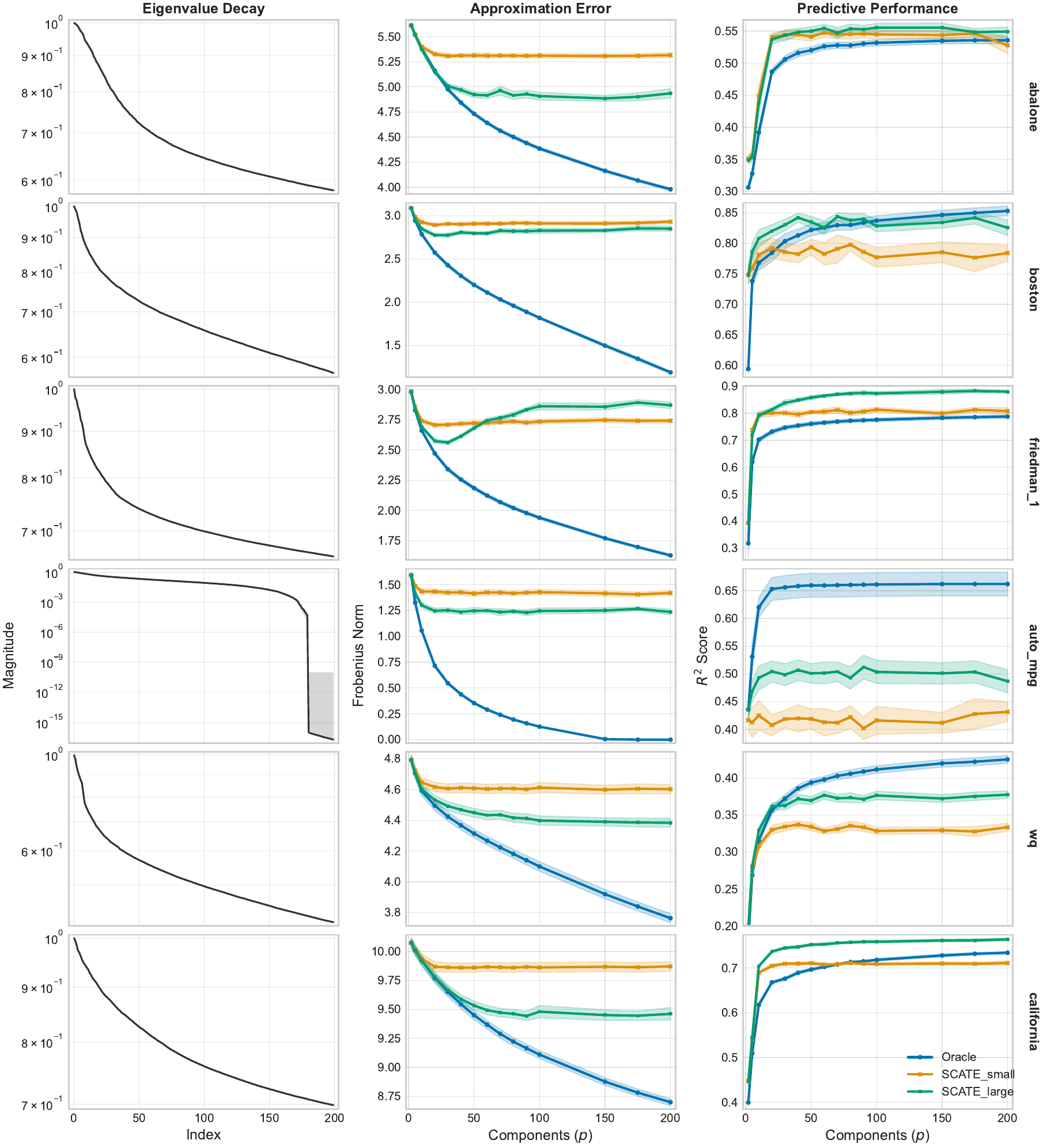}
    \caption{Spectral decay, kernel approximation error, and predictive performance across the remaining six regression datasets in the benchmark. Shaded regions indicate standard errors.}
    \label{fig:vis_full_reg}
\end{figure}

\begin{figure}[htbp]
    \centering
    \includegraphics[width=\linewidth]{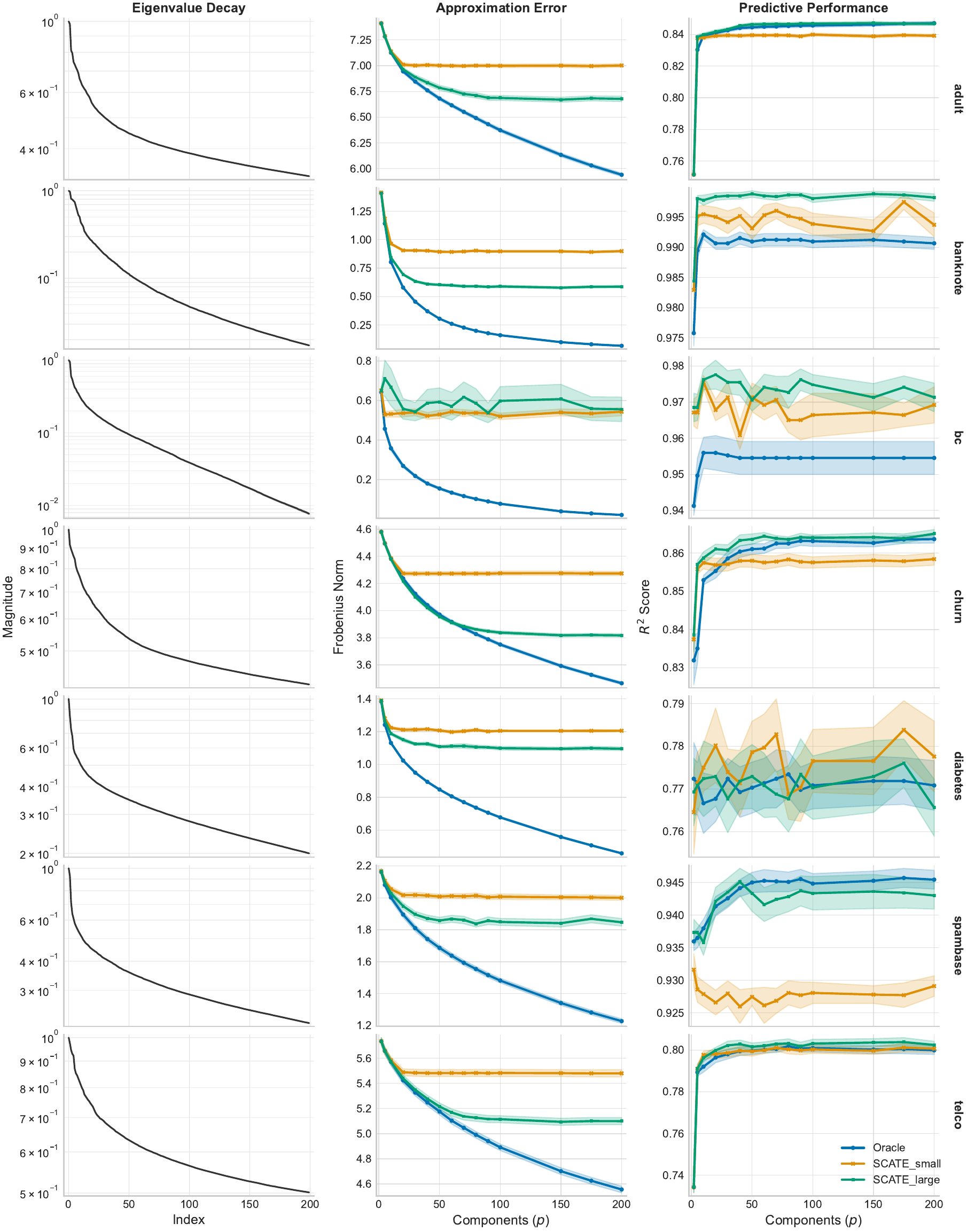}
    \caption{Spectral decay, kernel approximation error, and predictive performance across the remaining seven classification datasets in the benchmark. Shaded regions indicate standard errors.}
    \label{fig:vis_rull_cl}
\end{figure}
These findings further illustrate that substantial improvements in compression performance are closely associated with a steep decay in the kernel's eigenvalues. Moreover, SCATE consistently matches Oracle performance, demonstrating robustness even when the kernel approximation error stops improving due to noise in the trailing eigencomponents.

\subsection{Additional Na\"ive Benchmark Results}\label{appx:extra_naive}

In Figure \ref{fig:naive_full}, we present the Pareto frontiers for all remaining datasets evaluated in the na\"ive benchmark, supplementing the initial results shown in Figure \ref{fig:naive}. These plots demonstrate that SCATE consistently serves as the dominant method for RF compression compared to na\"ive alternatives. 

While there are isolated instances where baseline methods achieve higher scores, these occurrences typically coincide with anomalous scenarios where the compressed models outperform the uncompressed Base RF. While these performance gains are possible, they are neither predictable nor broadly replicable, and thus cannot be relied on in practice. In contrast, SCATE maintains stable performance across diverse scenarios, validating its robustness as a highly effective methodology for RF compression.

\begin{figure}[htbp]
    \centering
    \includegraphics[width=\linewidth]{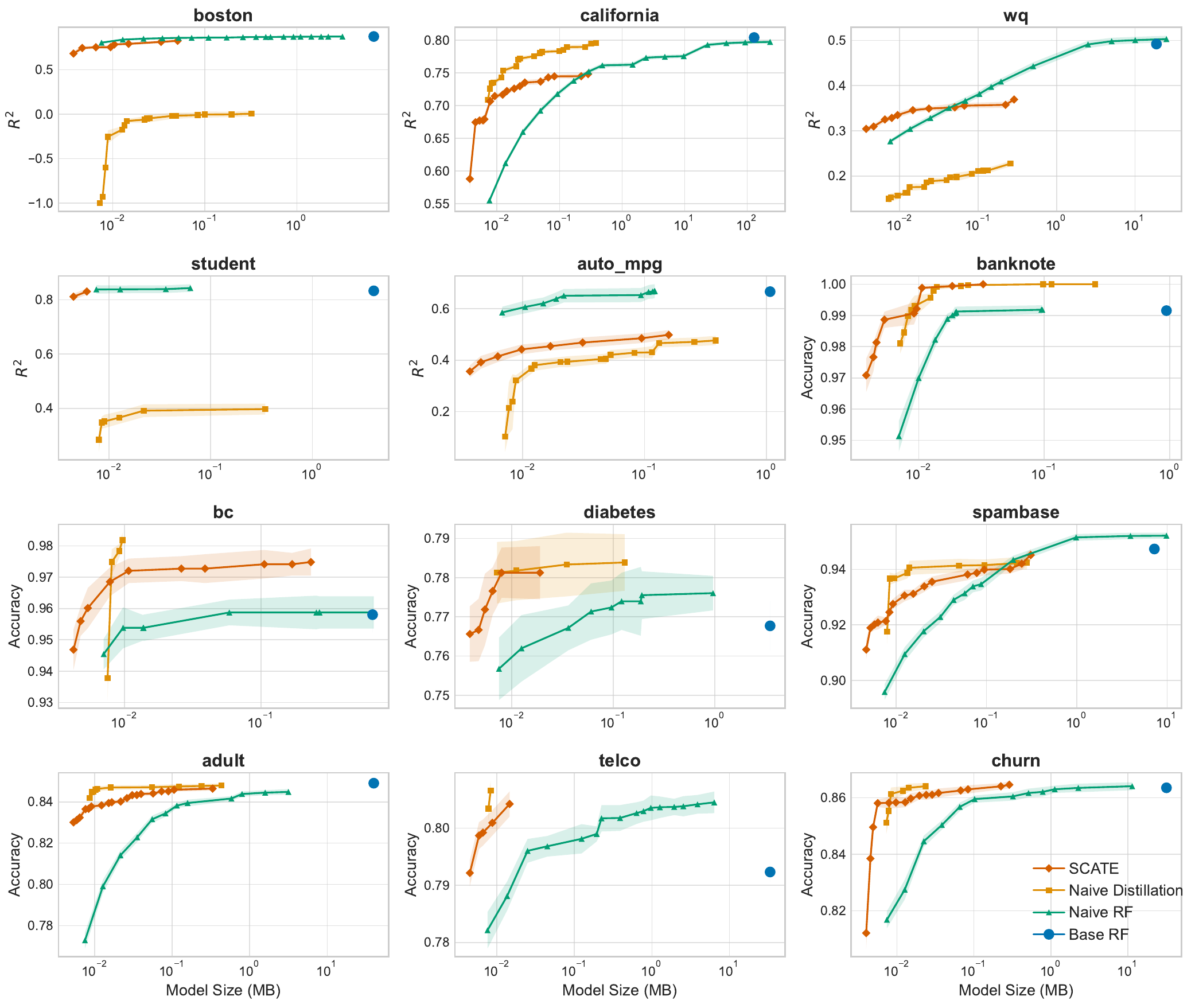} 
    \caption{Rate-performance Pareto curves across the remaining twelve benchmark datasets. Shaded regions indicate standard errors.}
    \label{fig:naive_full}
\end{figure}
\newpage
\subsection{Comprehensive Compression Benchmark Results}\label{appx:big_table}
As a supplement to Tables \ref{tab:rf} and \ref{tab:gbm}, we present the complete compression benchmark results in Tables \ref{tab:rf_results} and \ref{tab:gbm_results}. This expanded table reports the standard errors for all evaluated methods across each dataset. To identify the most competitive approaches, methods achieving performance within one standard error of the top-performing model are jointly highlighted in bold, though they remain ordered by their absolute mean performance. Notably, while the mean performance of several baseline models may initially appear comparable to our own, a more granular analysis incorporating these standard errors reveals that our method is superior, in the vast majority of scenarios, and validates our empirical rankings. We additionally report the size of the base RF and GBM models, which for 10KB show a compression factor around $10^3$ for RFs, and $10^2$ for GBMs, which is relatively inline with the values often tested for within the literature \citep{devos2025, liu2023}. To be fair with previous works, which reported leaf count as a measure of size, we strip down each base model to simply the matrices needed to perform inference (features, thresholds, left, right, and values) before computing its size.
\clearpage
\begin{sidewaystable}[htbp]
\centering
\caption{RF compression benchmark. Scores and base sizes are mean $\pm$ standard error across available runs. Base sizes use MiB. Bold entries lie within one standard error of the highest mean. Superscripts $[n]$ indicate fewer than ten contributing runs. NA denotes no within-budget result; TO denotes no candidate result within time limit.}
\label{tab:rf_results}
\setlength{\tabcolsep}{3pt}
\renewcommand{\arraystretch}{1.1}
\resizebox{\linewidth}{!}{
\begin{tabular}{lcc|ccccccc|ccccccc}
\toprule
~ & & & \multicolumn{7}{c|}{\textbf{Size = 10KB}} & \multicolumn{7}{c}{\textbf{Size = 100KB}} \\
\cmidrule{4-17}
\textbf{Dataset} & \textbf{RF} & \textbf{Base Size} & \textbf{ForestPrune} & \textbf{LOP} & \textbf{FIRE} & \textbf{TreeExtract} & \textbf{SSF} & \textbf{SCATE} & \textbf{FP-SCATE} & \textbf{ForestPrune} & \textbf{LOP} & \textbf{FIRE} & \textbf{TreeExtract} & \textbf{SSF} & \textbf{SCATE} & \textbf{FP-SCATE} \\
\midrule
\texttt{abalone} & 0.548 & 31.180 $\pm$ 0.110 & 0.347 $\pm$ 0.010 & 0.504 $\pm$ 0.009 & 0.282 $\pm$ 0.016 & NA & NA & \textbf{0.561 $\pm$ 0.004} & 0.543 $\pm$ 0.010 & 0.533 $\pm$ 0.004 & 0.511 $\pm$ 0.010 & 0.283 $\pm$ 0.016 & 0.548 $\pm$ 0.004 & NA & \textbf{0.568 $\pm$ 0.003} & 0.545 $\pm$ 0.010 \\
\texttt{auto\_mpg} & 0.874 & 4.864 $\pm$ 0.015 & 0.630 $\pm$ 0.025 & 0.813 $\pm$ 0.015 & 0.676 $\pm$ 0.015 & NA & NA & 0.852 $\pm$ 0.010 & \textbf{0.867 $\pm$ 0.008} & 0.820 $\pm$ 0.016 & 0.820 $\pm$ 0.014 & 0.677 $\pm$ 0.015 & 0.840 $\pm$ 0.014 & 0.812 $\pm$ 0.010 & 0.869 $\pm$ 0.008 & \textbf{0.878 $\pm$ 0.008} \\
\texttt{bike} & 0.94 & 151.507 $\pm$ 0.311 & 0.434 $\pm$ 0.025 & NA & TO & TO & \textbf{0.699 $\pm$ 0.001} & 0.678 $\pm$ 0.007 & 0.684 $\pm$ 0.004 & 0.721 $\pm$ 0.018 & NA & TO & TO & 0.699 $\pm$ 0.001 & \textbf{0.855 $\pm$ 0.003} & 0.843 $\pm$ 0.003 \\
\texttt{boston} & 0.855 & 6.867 $\pm$ 0.020 & 0.600 $\pm$ 0.016 & \textbf{0.805 $\pm$ 0.023} & 0.684 $\pm$ 0.016 & NA & NA & 0.775 $\pm$ 0.013 & \textbf{0.790 $\pm$ 0.012} & 0.800 $\pm$ 0.015 & 0.818 $\pm$ 0.025 & 0.690 $\pm$ 0.016 & 0.810 $\pm$ 0.014 & 0.717 $\pm$ 0.013 & 0.814 $\pm$ 0.012 & \textbf{0.833 $\pm$ 0.011} \\
\texttt{california} & 0.804 & 129.371 $\pm$ 0.309 & 0.449 $\pm$ 0.027\textsuperscript{[9]} & 0.684 $\pm$ 0.006 & TO & TO & NA & 0.700 $\pm$ 0.003 & \textbf{0.707 $\pm$ 0.006}\textsuperscript{[4]} & 0.676 $\pm$ 0.010 & \textbf{0.744 $\pm$ 0.005} & TO & TO & NA & \textbf{0.746 $\pm$ 0.004} & \textbf{0.743 $\pm$ 0.004}\textsuperscript{[4]} \\
\texttt{concrete} & 0.907 & 14.422 $\pm$ 0.025 & 0.504 $\pm$ 0.015 & 0.741 $\pm$ 0.021 & 0.730 $\pm$ 0.014 & NA & NA & 0.801 $\pm$ 0.008 & \textbf{0.820 $\pm$ 0.005} & 0.758 $\pm$ 0.014 & 0.855 $\pm$ 0.005 & 0.746 $\pm$ 0.014 & 0.894 $\pm$ 0.003 & \textbf{0.901 $\pm$ 0.004} & 0.862 $\pm$ 0.003 & 0.861 $\pm$ 0.003 \\
\texttt{diamonds} & 0.981 & 200.838 $\pm$ 0.774 & 0.778 $\pm$ 0.053\textsuperscript{[4]} & NA & TO & TO & NA & \textbf{0.937 $\pm$ 0.005} & 0.594 $\pm$ 0.141\textsuperscript{[4]} & 0.932 $\pm$ 0.003 & NA & TO & TO & 0.930 $\pm$ 0.000 & \textbf{0.959 $\pm$ 0.001} & 0.599 $\pm$ 0.142\textsuperscript{[4]} \\
\texttt{friedman\_1} & 0.851 & 14.328 $\pm$ 0.016 & 0.474 $\pm$ 0.023 & 0.759 $\pm$ 0.012 & 0.482 $\pm$ 0.025 & NA & NA & \textbf{0.822 $\pm$ 0.006} & 0.811 $\pm$ 0.009 & 0.736 $\pm$ 0.014 & 0.825 $\pm$ 0.013 & 0.494 $\pm$ 0.024 & 0.794 $\pm$ 0.007 & NA & \textbf{0.863 $\pm$ 0.004} & 0.855 $\pm$ 0.006 \\
\texttt{student} & 0.851 & 4.077 $\pm$ 0.014 & 0.644 $\pm$ 0.030 & 0.804 $\pm$ 0.015 & 0.710 $\pm$ 0.013 & NA & 0.818 $\pm$ 0.015 & \textbf{0.854 $\pm$ 0.014} & 0.836 $\pm$ 0.019 & 0.838 $\pm$ 0.013 & 0.805 $\pm$ 0.015 & 0.716 $\pm$ 0.013 & 0.791 $\pm$ 0.011 & 0.818 $\pm$ 0.015 & \textbf{0.857 $\pm$ 0.013} & 0.840 $\pm$ 0.019 \\
\texttt{wq} & 0.496 & 18.393 $\pm$ 0.070 & 0.236 $\pm$ 0.012 & 0.312 $\pm$ 0.006 & 0.307 $\pm$ 0.010 & NA & NA & \textbf{0.336 $\pm$ 0.008} & \textbf{0.332 $\pm$ 0.006} & 0.350 $\pm$ 0.010 & \textbf{0.358 $\pm$ 0.008} & 0.340 $\pm$ 0.008 & 0.281 $\pm$ 0.005 & 0.243 $\pm$ 0.010 & \textbf{0.362 $\pm$ 0.007} & \textbf{0.360 $\pm$ 0.007} \\
\midrule
\texttt{adult} & 0.851 & 39.593 $\pm$ 0.200 & 0.792 $\pm$ 0.003 & 0.827 $\pm$ 0.003 & TO & TO & NA & \textbf{0.839 $\pm$ 0.001} & 0.834 $\pm$ 0.001\textsuperscript{[6]} & 0.837 $\pm$ 0.001 & 0.835 $\pm$ 0.001 & TO & TO & \textbf{0.848 $\pm$ 0.001} & \textbf{0.847 $\pm$ 0.001} & 0.846 $\pm$ 0.001\textsuperscript{[6]} \\
\texttt{banknote} & 0.993 & 0.964 $\pm$ 0.010 & 0.874 $\pm$ 0.009 & 0.986 $\pm$ 0.004 & 0.992 $\pm$ 0.002 & NA & 0.999 $\pm$ 0.001 & \textbf{1.000 $\pm$ 0.000} & \textbf{1.000 $\pm$ 0.000} & 0.919 $\pm$ 0.011 & 0.986 $\pm$ 0.004 & 0.992 $\pm$ 0.002 & 0.892 $\pm$ 0.003 & \textbf{1.000 $\pm$ 0.000} & \textbf{1.000 $\pm$ 0.000} & \textbf{1.000 $\pm$ 0.000} \\
\texttt{bc} & 0.962 & 0.677 $\pm$ 0.014 & 0.950 $\pm$ 0.007 & 0.922 $\pm$ 0.027 & 0.955 $\pm$ 0.005 & NA & 0.976 $\pm$ 0.004 & \textbf{0.984 $\pm$ 0.003} & 0.980 $\pm$ 0.003 & 0.951 $\pm$ 0.007 & 0.922 $\pm$ 0.027 & 0.955 $\pm$ 0.005 & 0.920 $\pm$ 0.007 & 0.980 $\pm$ 0.003 & \textbf{0.987 $\pm$ 0.003} & \textbf{0.985 $\pm$ 0.003} \\
\texttt{churn} & 0.862 & 31.671 $\pm$ 0.177 & 0.823 $\pm$ 0.004 & 0.844 $\pm$ 0.006 & 0.850 $\pm$ 0.002 & NA & NA & \textbf{0.859 $\pm$ 0.002} & 0.852 $\pm$ 0.007 & 0.853 $\pm$ 0.002 & 0.850 $\pm$ 0.004 & 0.850 $\pm$ 0.002 & 0.851 $\pm$ 0.002 & 0.828 $\pm$ 0.004\textsuperscript{[3]} & \textbf{0.863 $\pm$ 0.002} & 0.855 $\pm$ 0.008 \\
\texttt{diabetes} & 0.759 & 3.493 $\pm$ 0.040 & 0.757 $\pm$ 0.008 & 0.741 $\pm$ 0.014 & 0.765 $\pm$ 0.009 & NA & NA & \textbf{0.786 $\pm$ 0.010} & \textbf{0.794 $\pm$ 0.009} & 0.764 $\pm$ 0.009 & 0.741 $\pm$ 0.014 & 0.765 $\pm$ 0.009 & 0.661 $\pm$ 0.010 & 0.733 $\pm$ 0.007 & \textbf{0.794 $\pm$ 0.010} & \textbf{0.801 $\pm$ 0.009} \\
\texttt{magic} & 0.88 & 27.142 $\pm$ 0.120 & 0.793 $\pm$ 0.006 & 0.818 $\pm$ 0.004 & \textbf{0.860 $\pm$ 0.001} & NA & NA & \textbf{0.860 $\pm$ 0.002} & \textbf{0.860 $\pm$ 0.002} & 0.846 $\pm$ 0.002 & 0.859 $\pm$ 0.002 & 0.868 $\pm$ 0.001 & 0.841 $\pm$ 0.001\textsuperscript{[9]} & NA & \textbf{0.872 $\pm$ 0.001} & 0.869 $\pm$ 0.001 \\
\texttt{spambase} & 0.949 & 7.219 $\pm$ 0.037 & 0.868 $\pm$ 0.005 & 0.915 $\pm$ 0.003 & \textbf{0.937 $\pm$ 0.002} & NA & NA & 0.929 $\pm$ 0.001 & 0.926 $\pm$ 0.002 & 0.913 $\pm$ 0.002 & 0.933 $\pm$ 0.002 & \textbf{0.944 $\pm$ 0.002} & 0.894 $\pm$ 0.003 & \textbf{0.944 $\pm$ 0.002} & 0.939 $\pm$ 0.001 & 0.936 $\pm$ 0.001 \\
\texttt{telco} & 0.789 & 33.395 $\pm$ 0.142 & 0.747 $\pm$ 0.007 & 0.773 $\pm$ 0.008 & 0.789 $\pm$ 0.003 & NA & NA & \textbf{0.804 $\pm$ 0.003} & 0.800 $\pm$ 0.002 & 0.797 $\pm$ 0.002 & 0.775 $\pm$ 0.008 & 0.789 $\pm$ 0.003 & 0.789 $\pm$ 0.002 & 0.778 $\pm$ 0.004 & \textbf{0.808 $\pm$ 0.003} & 0.804 $\pm$ 0.002 \\
\midrule
\textbf{Average Rank} &  &  & 4.67 & 3.86 & 4.08 & 6.44 & 5.50 & \textbf{1.53} & 1.92 & 4.22 & 4.72 & 5.31 & 5.19 & 4.67 & \textbf{1.56} & 2.33 \\
\bottomrule
\end{tabular}
}
\vspace{1cm}

\caption{GBM compression benchmark. Scores and base sizes are mean $\pm$ standard error across available runs. Base sizes use MiB. Bold entries lie within one standard error of the highest mean. Superscripts $[n]$ indicate fewer than ten contributing runs. NA denotes no within-budget result; TO denotes no candidate result within time limit.}
\label{tab:gbm_results}
\setlength{\tabcolsep}{3pt}
\renewcommand{\arraystretch}{1.1}
\resizebox{\linewidth}{!}{
\begin{tabular}{lcc|cccccc|cccccc}
\toprule
~ & & & \multicolumn{6}{c|}{\textbf{Size = 10KB}} & \multicolumn{6}{c}{\textbf{Size = 100KB}} \\
\cmidrule{4-15}
\textbf{Dataset} & \textbf{GBM} & \textbf{Base Size} & \textbf{FIRE} & \textbf{ForestPrune} & \textbf{FP-SCATE} & \textbf{LOP} & \textbf{SCATE} & \textbf{TreeExtract} & \textbf{FIRE} & \textbf{ForestPrune} & \textbf{FP-SCATE} & \textbf{LOP} & \textbf{SCATE} & \textbf{TreeExtract} \\
\midrule
\texttt{abalone} & 0.529 & 0.690 $\pm$ 0.004 & 0.465 $\pm$ 0.009 & 0.445 $\pm$ 0.004 & 0.547 $\pm$ 0.007 & 0.505 $\pm$ 0.006 & \textbf{0.562 $\pm$ 0.002} & 0.540 $\pm$ 0.003 & 0.465 $\pm$ 0.009 & 0.528 $\pm$ 0.004 & 0.552 $\pm$ 0.008 & 0.520 $\pm$ 0.007 & \textbf{0.571 $\pm$ 0.002} & 0.547 $\pm$ 0.004 \\
\texttt{auto\_mpg} & 0.858 & 0.585 $\pm$ 0.004 & 0.820 $\pm$ 0.012 & 0.799 $\pm$ 0.016 & 0.835 $\pm$ 0.016 & 0.820 $\pm$ 0.011 & \textbf{0.852 $\pm$ 0.008} & \textbf{0.854 $\pm$ 0.008} & 0.820 $\pm$ 0.012 & 0.861 $\pm$ 0.011 & 0.847 $\pm$ 0.018 & 0.826 $\pm$ 0.010 & \textbf{0.872 $\pm$ 0.007} & 0.854 $\pm$ 0.008 \\
\texttt{bike} & 0.95 & 0.850 $\pm$ 0.002 & \textbf{0.880 $\pm$ 0.003} & 0.571 $\pm$ 0.010 & 0.624 $\pm$ 0.009 & NA & 0.618 $\pm$ 0.007 & TO & 0.880 $\pm$ 0.003 & 0.922 $\pm$ 0.001 & 0.868 $\pm$ 0.010 & \textbf{0.944 $\pm$ 0.001} & 0.919 $\pm$ 0.001 & TO \\
\texttt{boston} & 0.862 & 0.581 $\pm$ 0.005 & 0.777 $\pm$ 0.027 & 0.775 $\pm$ 0.026 & 0.710 $\pm$ 0.067 & \textbf{0.807 $\pm$ 0.026} & \textbf{0.790 $\pm$ 0.012} & \textbf{0.786 $\pm$ 0.015} & 0.777 $\pm$ 0.027 & \textbf{0.866 $\pm$ 0.014} & 0.757 $\pm$ 0.069 & 0.826 $\pm$ 0.026 & 0.842 $\pm$ 0.009 & 0.817 $\pm$ 0.013 \\
\texttt{california} & 0.829 & 0.803 $\pm$ 0.004 & 0.750 $\pm$ 0.004 & 0.583 $\pm$ 0.004 & 0.704 $\pm$ 0.003 & 0.758 $\pm$ 0.006 & 0.708 $\pm$ 0.004 & \textbf{0.792 $\pm$ 0.003} & 0.750 $\pm$ 0.004 & 0.800 $\pm$ 0.003 & 0.763 $\pm$ 0.006 & \textbf{0.823 $\pm$ 0.002} & 0.792 $\pm$ 0.002 & 0.792 $\pm$ 0.003 \\
\texttt{concrete} & 0.921 & 0.687 $\pm$ 0.005 & 0.845 $\pm$ 0.007 & 0.711 $\pm$ 0.008 & 0.814 $\pm$ 0.006 & 0.834 $\pm$ 0.007 & 0.789 $\pm$ 0.013 & \textbf{0.870 $\pm$ 0.004} & 0.845 $\pm$ 0.007 & 0.901 $\pm$ 0.004 & 0.868 $\pm$ 0.005 & 0.908 $\pm$ 0.003 & 0.879 $\pm$ 0.005 & \textbf{0.915 $\pm$ 0.003} \\
\texttt{diamonds} & 0.981 & 0.835 $\pm$ 0.001 & \textbf{0.939 $\pm$ 0.001} & 0.878 $\pm$ 0.003 & 0.935 $\pm$ 0.009 & TO & \textbf{0.942 $\pm$ 0.005} & NA & 0.939 $\pm$ 0.001 & 0.975 $\pm$ 0.000 & 0.960 $\pm$ 0.003 & TO & 0.970 $\pm$ 0.001 & \textbf{0.981 $\pm$ 0.000}\textsuperscript{[6]} \\
\texttt{friedman\_1} & 0.910 & 0.717 $\pm$ 0.005 & 0.790 $\pm$ 0.010 & 0.680 $\pm$ 0.015 & 0.867 $\pm$ 0.016 & 0.874 $\pm$ 0.017 & 0.893 $\pm$ 0.007 & \textbf{0.917 $\pm$ 0.003} & 0.790 $\pm$ 0.010 & 0.899 $\pm$ 0.005 & 0.926 $\pm$ 0.015 & 0.905 $\pm$ 0.004 & \textbf{0.972 $\pm$ 0.001} & 0.917 $\pm$ 0.003 \\
\texttt{student} & 0.832 & 0.689 $\pm$ 0.006 & 0.752 $\pm$ 0.031 & 0.811 $\pm$ 0.014 & 0.699 $\pm$ 0.041 & 0.789 $\pm$ 0.033 & \textbf{0.837 $\pm$ 0.012} & 0.785 $\pm$ 0.014 & 0.752 $\pm$ 0.031 & \textbf{0.857 $\pm$ 0.009} & 0.701 $\pm$ 0.041 & 0.789 $\pm$ 0.033 & 0.838 $\pm$ 0.012 & 0.785 $\pm$ 0.014 \\
\texttt{wq} & 0.454 & 0.683 $\pm$ 0.005 & \textbf{0.346 $\pm$ 0.007} & 0.287 $\pm$ 0.005 & 0.331 $\pm$ 0.007 & 0.317 $\pm$ 0.007 & 0.333 $\pm$ 0.007 & 0.300 $\pm$ 0.006 & 0.346 $\pm$ 0.007 & \textbf{0.385 $\pm$ 0.007} & 0.356 $\pm$ 0.007 & \textbf{0.381 $\pm$ 0.006} & 0.358 $\pm$ 0.008 & 0.300 $\pm$ 0.006 \\
\midrule
\texttt{adult} & 0.853 & 0.705 $\pm$ 0.003 & 0.814 $\pm$ 0.002 & 0.819 $\pm$ 0.002 & 0.841 $\pm$ 0.001 & 0.844 $\pm$ 0.002 & 0.846 $\pm$ 0.001 & \textbf{0.853 $\pm$ 0.001} & 0.814 $\pm$ 0.002 & 0.853 $\pm$ 0.001 & 0.849 $\pm$ 0.001 & 0.852 $\pm$ 0.001 & 0.853 $\pm$ 0.001 & \textbf{0.856 $\pm$ 0.001} \\
\texttt{banknote} & 0.985 & 0.471 $\pm$ 0.023 & 0.982 $\pm$ 0.002 & 0.947 $\pm$ 0.008 & 0.998 $\pm$ 0.002 & 0.977 $\pm$ 0.005 & \textbf{1.000 $\pm$ 0.000} & 0.991 $\pm$ 0.002 & 0.982 $\pm$ 0.002 & 0.986 $\pm$ 0.002 & 0.998 $\pm$ 0.002 & 0.977 $\pm$ 0.005 & \textbf{1.000 $\pm$ 0.000} & 0.995 $\pm$ 0.001 \\
\texttt{bc} & 0.938 & 0.310 $\pm$ 0.034 & 0.946 $\pm$ 0.004 & 0.936 $\pm$ 0.008 & 0.947 $\pm$ 0.036 & 0.937 $\pm$ 0.007 & \textbf{0.984 $\pm$ 0.003} & 0.945 $\pm$ 0.006 & 0.946 $\pm$ 0.004 & 0.949 $\pm$ 0.005 & 0.949 $\pm$ 0.036 & 0.937 $\pm$ 0.007 & \textbf{0.990 $\pm$ 0.002} & 0.952 $\pm$ 0.005 \\
\texttt{churn} & 0.861 & 0.711 $\pm$ 0.004 & 0.852 $\pm$ 0.002 & 0.847 $\pm$ 0.002 & 0.853 $\pm$ 0.006 & 0.862 $\pm$ 0.002 & 0.860 $\pm$ 0.002 & \textbf{0.864 $\pm$ 0.002} & 0.852 $\pm$ 0.002 & \textbf{0.864 $\pm$ 0.002} & 0.855 $\pm$ 0.006 & 0.862 $\pm$ 0.002 & \textbf{0.865 $\pm$ 0.002} & \textbf{0.866 $\pm$ 0.002} \\
\texttt{diabetes} & 0.744 & 0.592 $\pm$ 0.004 & 0.744 $\pm$ 0.011 & 0.768 $\pm$ 0.009 & 0.728 $\pm$ 0.026 & 0.755 $\pm$ 0.012 & \textbf{0.791 $\pm$ 0.011} & 0.777 $\pm$ 0.008 & 0.744 $\pm$ 0.011 & 0.779 $\pm$ 0.009 & 0.730 $\pm$ 0.027 & 0.755 $\pm$ 0.012 & \textbf{0.797 $\pm$ 0.010} & 0.783 $\pm$ 0.009 \\
\texttt{magic} & 0.88 & 0.731 $\pm$ 0.003 & 0.851 $\pm$ 0.002 & 0.811 $\pm$ 0.002 & 0.862 $\pm$ 0.002 & 0.854 $\pm$ 0.002 & 0.862 $\pm$ 0.001 & \textbf{0.866 $\pm$ 0.001} & 0.851 $\pm$ 0.002 & 0.873 $\pm$ 0.002 & 0.874 $\pm$ 0.002 & 0.871 $\pm$ 0.003 & 0.875 $\pm$ 0.001 & \textbf{0.881 $\pm$ 0.002} \\
\texttt{spambase} & 0.945 & 0.507 $\pm$ 0.003 & 0.916 $\pm$ 0.002 & 0.904 $\pm$ 0.002 & 0.926 $\pm$ 0.002 & 0.922 $\pm$ 0.003 & 0.928 $\pm$ 0.002 & \textbf{0.944 $\pm$ 0.001} & 0.916 $\pm$ 0.002 & 0.946 $\pm$ 0.001 & 0.935 $\pm$ 0.001 & 0.934 $\pm$ 0.003 & 0.938 $\pm$ 0.001 & \textbf{0.951 $\pm$ 0.001} \\
\texttt{telco} & 0.794 & 0.691 $\pm$ 0.007 & 0.793 $\pm$ 0.003 & 0.791 $\pm$ 0.003 & 0.793 $\pm$ 0.009 & 0.794 $\pm$ 0.004 & \textbf{0.806 $\pm$ 0.003} & \textbf{0.804 $\pm$ 0.003} & 0.793 $\pm$ 0.003 & 0.803 $\pm$ 0.003 & 0.798 $\pm$ 0.009 & 0.794 $\pm$ 0.004 & \textbf{0.809 $\pm$ 0.003} & 0.805 $\pm$ 0.003 \\
\midrule
\textbf{Average Rank} &  &  & 3.83 & 5.28 & 3.67 & 3.67 & \textbf{2.06} & 2.50 & 5.50 & 2.64 & 4.19 & 3.94 & \textbf{2.06} & 2.67 \\
\bottomrule
\end{tabular}
}

\label{tab:full}
\end{sidewaystable}

\clearpage

\subsection{Visualizing Compression Frontiers}\label{appx:extra_compression}
While Tables \ref{tab:rf_results} and \ref{tab:gbm_results} provides a granular analysis of performance at discrete memory constraints, we further evaluate SCATE across a continuous capacity spectrum ranging from 1KB to 100KB. Plotting traditional Pareto frontiers for all individual methods often results in severe visual clutter, as disparate hyperparameter configurations across techniques lead to irregular and non-uniform model sizes. Instead, to ensure a clear comparison, we isntead distill these results into two curves each dataset. The first curve tracks the peak performance of SCATE evaluated at 1KB intervals. The second curve is a composition of the state-of-the-art pruning methods: at each 1KB threshold, it adopts the maximum score achieved by \textit{any} competing baseline method (excluding the generally weaker FP-SCATE variant). This baseline then represents a best-in-class curve against which we can evaluate SCATE. 
\begin{figure}[htbp]
    \centering
    \includegraphics[width=0.9\linewidth]{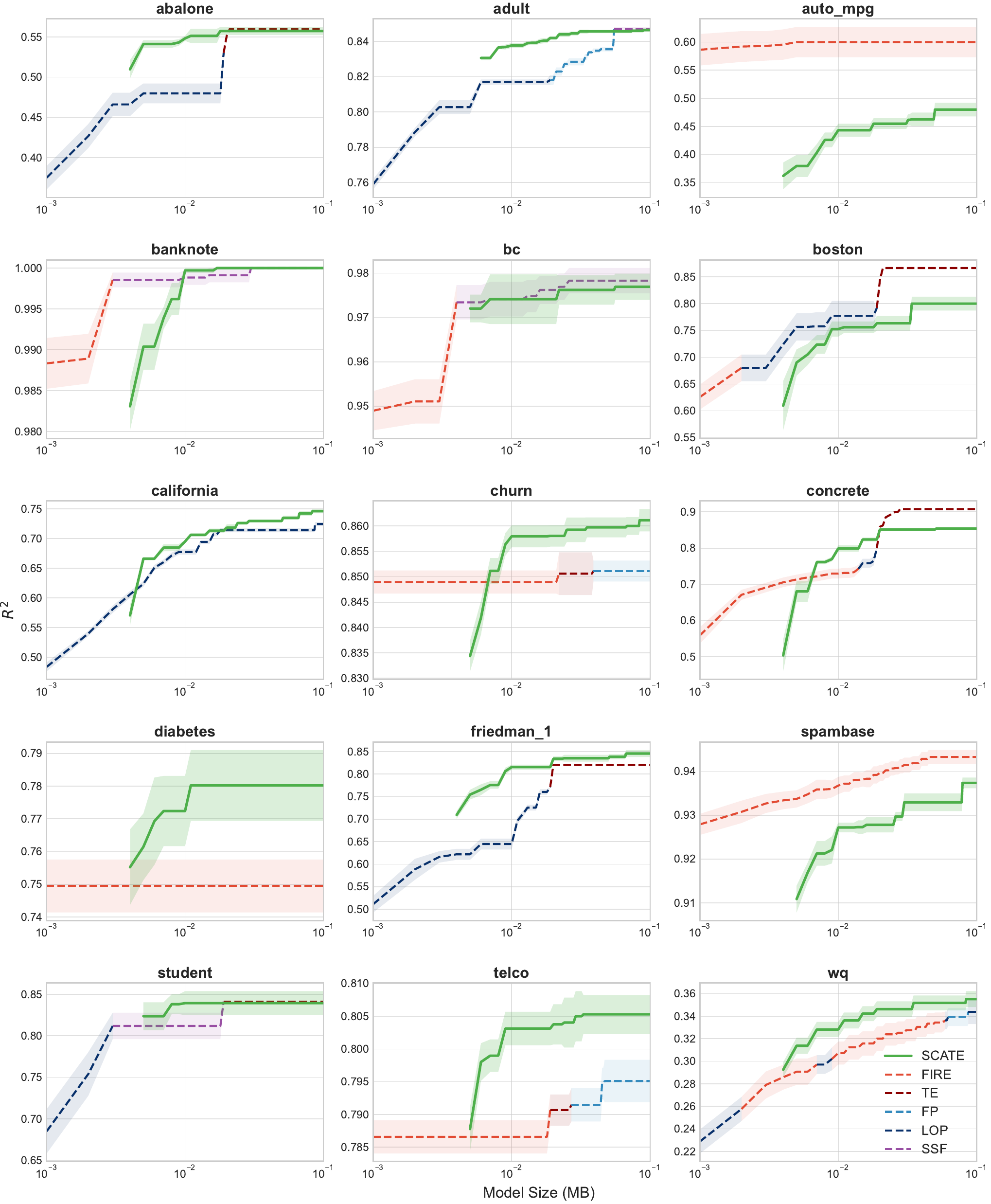}
    \caption{Visualization curves for the RF compression benchmark. The solid line represents SCATE, and the dashed line represents the composite method.}
    \label{fig:rf_pareto}
\end{figure}

\begin{figure}[htbp]
    \centering
    \includegraphics[width=0.9\linewidth]{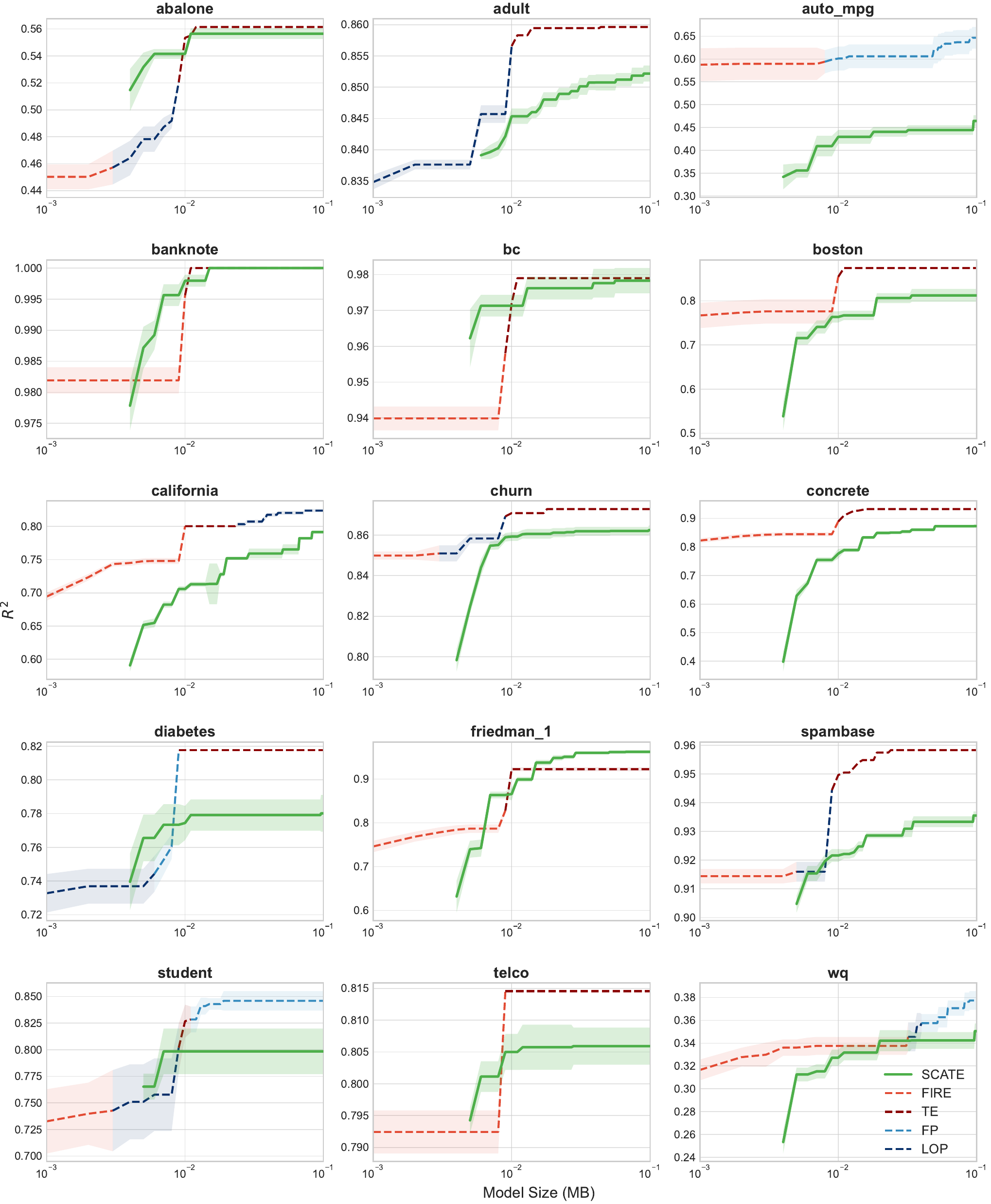}
    \caption{Visualization curves for the GBM compression benchmark. The solid line represents SCATE, and the dashed line represents the composite method.}
    \label{fig:gbm_pareto}
\end{figure}

The resulting performance curves are illustrated in Figure \ref{fig:rf_pareto} for RFs and Figure \ref{fig:gbm_pareto} for GBMs. An examination of these figures reveals SCATE’s robustness, especially around the 10KB point. In Fig. \ref{fig:rf_pareto}, SCATE frequently beats the composite curve or performs highly competitively, with the exception of the \texttt{auto\_mpg} dataset. In the GBM domain, the results are more tempered; SCATE slightly underperforms, though the differences are generally narrow. As discussed in Sect. \ref{sec:discussion}, this dynamic occurs because SCATE's compression footprint is bounded by the complexity of the kernel rather than the size of the original model. Consequently, we can attribute these results to competing methods inherently struggling more to compress the larger parameter counts of RFs.

\subsection{Runtime Experiments}\label{appx:runtime}

We evaluate the computational efficiency of our method and others through two experiments. First, we measure benchmark both training and inference times across three representative datasets. Then, we evaluate the trade-off between compression time and predictive performance under some fixed memory budgets. The second experiment is intended to give a more realistic view on the cost of searching for configurations, as opposed to just the runtime of the selected model.

\paragraph {Training and Inference Benchmark}

We assess the runtime of our experiments on three small-to-medium datasets: \texttt{student}, \texttt{wq}, and \texttt{telco}. We deliberately exclude the two largest datasets (\texttt{adult} and \texttt{california}) from this specific analysis; the prohibitive computational scaling of baseline rule extractors, specifically FIRE \citep{liu2023fire} and TreeExtract \citep{liu2025extracting}, resulted in frequent timeouts on these datasets, preventing a complete comparison.

Because these compression methods are primarily motivated by strict deployment constraints, runtime is most meaningfully analyzed as a function of the final model's memory footprint - as a result, we evaluate runtime across a continuous capacity spectrum. We discretize the target memory footprint into 5KB intervals, between 0KB to 100KB. For each interval, we compute the average training and inference times across all hyperparameter configurations that yield a compressed model falling within that specific 5KB memory bucket. This procedure is conducted independently for both RF and GBM base models, generating independent training and inference time profiles for each dataset (resulting in 12 comparative plots in total). The resulting plots are produced in 
\begin{figure}[htbp]
    \centering
    \includegraphics[width=0.9\linewidth]{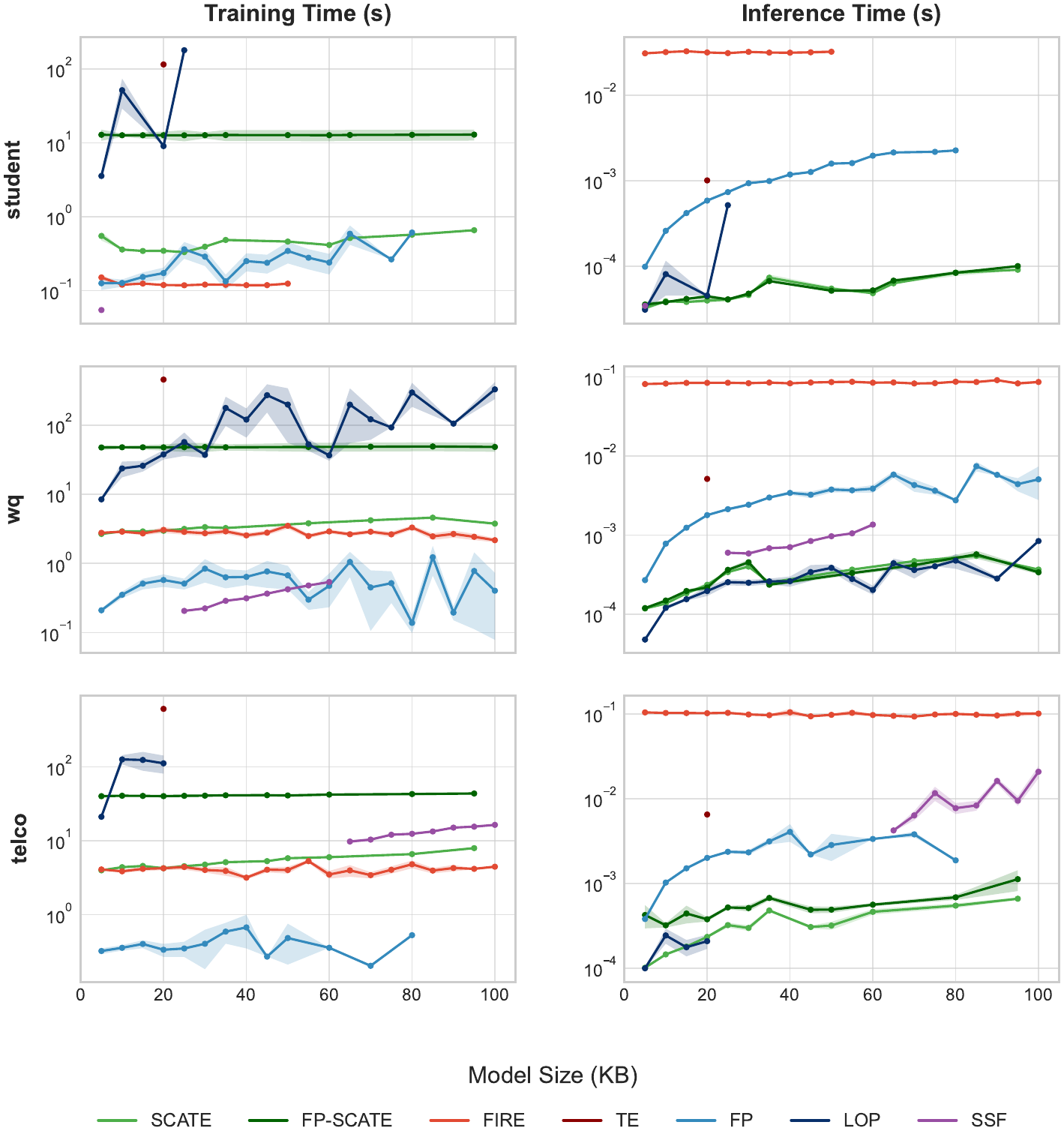}
    \caption{Training and inference time against model size plots for RFs, over three datasets}
    \label{fig:rf_runtime}
\end{figure}

\begin{figure}[htbp]
    \centering
    \includegraphics[width=0.9\linewidth]{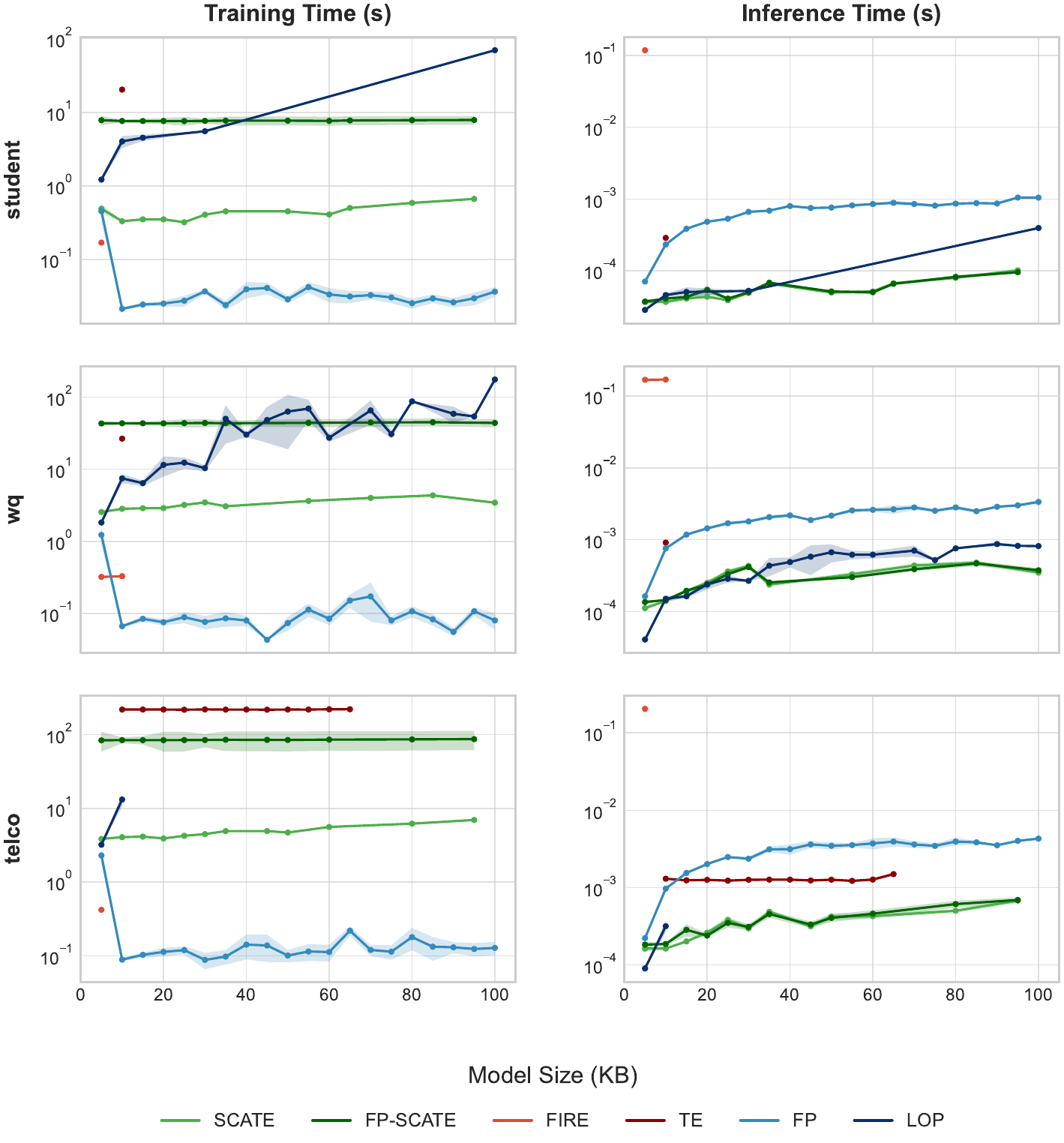}
    \caption{Training and inference time against model size plots for GBMs, over three datasets}
    \label{fig:gbm_runtime}
\end{figure}
Interestingly, our evaluations indicate that runtime remains largely invariant across different model complexities for the majority of the evaluated algorithms. Therefore, an additional comprehensive summary of these computational costs, reported as means and standard errors, is provided in Tab. \ref{tab:runtime_table}.

\begin{table}[htbp]
    \centering
    \caption{Average training and inference runtimes over the entire evaluated hyperparameter set. All values are reported in seconds as mean $\pm$ standard error. The fastest methods per column (including those within 1 standard error of the best) are bolded. SSF results in GBM are blank, as SSF is an RF-exclusive method.}
    \label{tab:runtime_table}
    \renewcommand{\arraystretch}{1.2}
    \resizebox{\textwidth}{!}{
    \begin{tabular}{l cc cc}
        \toprule
        & \multicolumn{2}{c}{\textbf{Random Forest}} & \multicolumn{2}{c}{\textbf{GBM}} \\
        \cmidrule(lr){2-3} \cmidrule(lr){4-5}
        \textbf{Method} & \textbf{Avg Training (s)} & \textbf{Avg Inference (s)} & \textbf{Avg Training (s)} & \textbf{Avg Inference (s)} \\
        \midrule
        \textbf{SCATE}    & $3.0402 \pm 0.3668$    & $\mathbf{2.32 \times 10^{-4} \pm 3.04 \times 10^{-5}}$ & $2.8506 \pm 0.3340$    & $\mathbf{2.37 \times 10^{-4} \pm 2.96 \times 10^{-5}}$ \\
        \textbf{FP-SCATE} & $33.7156 \pm 2.6515$   & $3.06 \times 10^{-4} \pm 4.26 \times 10^{-5}$          & $45.5539 \pm 5.5014$   & $\mathbf{2.41 \times 10^{-4} \pm 2.99 \times 10^{-5}}$ \\
        \textbf{LOP}      & $108.8504 \pm 18.2936$ & $2.69 \times 10^{-4} \pm 3.49 \times 10^{-5}$          & $36.3631 \pm 7.8968$   & $4.01 \times 10^{-4} \pm 5.67 \times 10^{-5}$          \\
        \textbf{FP}       & $\mathbf{0.4319 \pm 0.0347}$ & $2.46 \times 10^{-3} \pm 2.34 \times 10^{-4}$    & $\mathbf{0.1445 \pm 0.0422}$ & $1.97 \times 10^{-3} \pm 1.61 \times 10^{-4}$          \\
        \textbf{TE}       & $393.9479 \pm 146.0017$& $4.23 \times 10^{-3} \pm 1.66 \times 10^{-3}$          & $191.7248 \pm 19.0575$ & $1.18 \times 10^{-3} \pm 7.56 \times 10^{-5}$          \\
        \textbf{SSF}      & $6.3441 \pm 1.6434$    & $5.41 \times 10^{-3} \pm 1.52 \times 10^{-3}$          & --                     & --                                                     \\
        \textbf{FIRE}     & $2.7730 \pm 0.2128$    & $8.02 \times 10^{-2} \pm 3.61 \times 10^{-3}$          & $0.3109 \pm 0.0519$    & $1.65 \times 10^{-1} \pm 1.78 \times 10^{-2}$          \\
        \bottomrule
    \end{tabular}
    }
\end{table}
SCATE has the lowest mean inference time, with FP-SCATE also remaining competitive. This is relevant for constrained deployment scenarios, though it is possible that competing tree methods can be made more efficient by making greater use of tree parallelization, in the RF case. However, these measurements concern individual configurations - the cost of searching for an optimal configuration on the validation set is often a concern in practice, which we also explore below.

\paragraph{Additional runtime experiments}
We additionally examine training and inference on the \texttt{churn} and \texttt{california} datasets for the faster SSF, FIRE, FP, and SCATE methods, and report the results in Tables~\ref{tab:runtime-additional-rf}
and~\ref{tab:runtime-additional-gbm}. Values are means and standard
errors across occupied model-size intervals in the nominal
5--100\,KB range.

\begin{table}[t]
\centering
\caption{Additional RF runtime results on Churn and California Housing. Training and inference times are reported in seconds. NA indicates that SSF did not reach the 100\,KB size threshold; TO indicates the 24-hour timeout. For SSF on Churn, only one size interval contributed.}
\label{tab:runtime-additional-rf}
\setlength{\tabcolsep}{5pt}
\renewcommand{\arraystretch}{1.1}
\resizebox{\textwidth}{!}{%
\begin{tabular}{lcccc}
\toprule
& \multicolumn{2}{c}{\textbf{Churn}}
& \multicolumn{2}{c}{\textbf{California}} \\
 \cmidrule(lr){2-3} \cmidrule(lr){4-5}
\textbf{Method}
& \textbf{Avg Training (s)} & \textbf{Avg Inference (s)}
& \textbf{Avg Training (s)} & \textbf{Avg Inference (s)} \\
\midrule
SSF
& $27.5816 \pm 0.0000$
& $1.23{\times}10^{-1} \pm 0$
& NA & NA \\
FIRE
& $5.8907 \pm 0.1828$
& $1.24{\times}10^{-1} \pm 7.03{\times}10^{-4}$
& TO & TO \\
FP
& $2.0442 \pm 0.3313$
& $9.00{\times}10^{-3} \pm 1.04{\times}10^{-3}$
& $7.5740 \pm 1.0152$
& $2.68{\times}10^{-2} \pm 2.62{\times}10^{-3}$ \\
SCATE
& $8.0607 \pm 0.3809$
& $4.66{\times}10^{-4} \pm 6.05{\times}10^{-5}$
& $11.1591 \pm 0.5400$
& $8.85{\times}10^{-4} \pm 1.04{\times}10^{-4}$ \\
\bottomrule
\end{tabular}%
}
\end{table}

\begin{table}[t]
\centering
\caption{Additional GBM runtime results on Churn and California
Housing. Training and inference times are reported in seconds. For FIRE on Churn, only one size interval contributed, so the reported zero standard error does not quantify uncertainty. SSF is omitted because it is an RF-only method.}
\label{tab:runtime-additional-gbm}
\setlength{\tabcolsep}{5pt}
\renewcommand{\arraystretch}{1.1}
\resizebox{\textwidth}{!}{%
\begin{tabular}{lcccc}
\toprule
& \multicolumn{2}{c}{\textbf{Churn}}
& \multicolumn{2}{c}{\textbf{California}} \\
 \cmidrule(lr){2-3} \cmidrule(lr){4-5}
\textbf{Method}
& \textbf{Training (s)} & \textbf{Inference (s)}
& \textbf{Training (s)} & \textbf{Inference (s)} \\
\midrule
FIRE
& $0.5860 \pm 0.0000$
& $2.68{\times}10^{-1} \pm 0$
& $1.4372 \pm 0.0158$
& $5.14{\times}10^{-1} \pm 6.74{\times}10^{-3}$ \\
FP
& $0.6671 \pm 0.2921$
& $5.46{\times}10^{-3} \pm 4.63{\times}10^{-4}$
& $1.5546 \pm 0.6215$
& $1.28{\times}10^{-2} \pm 1.10{\times}10^{-3}$ \\
SCATE
& $16.5309 \pm 0.9629$
& $2.31{\times}10^{-3} \pm 2.17{\times}10^{-4}$
& $9.8198 \pm 0.6122$
& $6.54{\times}10^{-4} \pm 6.71{\times}10^{-5}$ \\
\bottomrule
\end{tabular}%
}
\end{table}

\paragraph{Compression-configuration Benchmark}

To compare runtime and predictive performance more jointly, we conduct an additional benchmark on eight datasets: \texttt{student}, \texttt{diabetes}, \texttt{spambase},
\texttt{wq}, \texttt{telco}, \texttt{california},
\texttt{magic}, and \texttt{diamonds}. Experiments use AMD EPYC 7702 CPUs on cluster computing, with five allocated CPU cores each seed and GPUs disabled. For each method, dataset, and ensemble type, we evaluate up to ten pre-specified configurations over up to ten data-split and model seeds. Each seed has a three-hour time limit. If the first two seeds both time out without completing a configuration, the remaining sessions for that
combination are considered to have timed out as well.

For each completed seed, the reported compression-search time is the sum of shared preprocessing and training times of the evaluated configurations. For SCATE and FP-SCATE, this involves the construction and decomposition of the corresponding smoother matrices (with an additional FP step for FP-SCATE). A seed that reaches the
time limit after completing at least one configuration is assigned a runtime of three hours, and its completed configurations are kept. Failed seeds and seeds with no completed configurations are marked as timeouts.

For each memory budget (10kB and 100kB), we select the configuration with the highest test score, and scale it by the uncompressed model's test score. The resulting ratio is a relative summary, since regression and classification problems are combined. Selecting models this way describes the best observed result within the search, as opposed to the performance of a configuration selected independently with validation data. We first average the normalized scores and search times across available seeds for each dataset. Figure~\ref{fig:runtime_pareto} then reports the median and interquartile range across datasets
with an eligible result.

\begin{figure}[htbp]
    \centering
    \includegraphics[width=\linewidth]{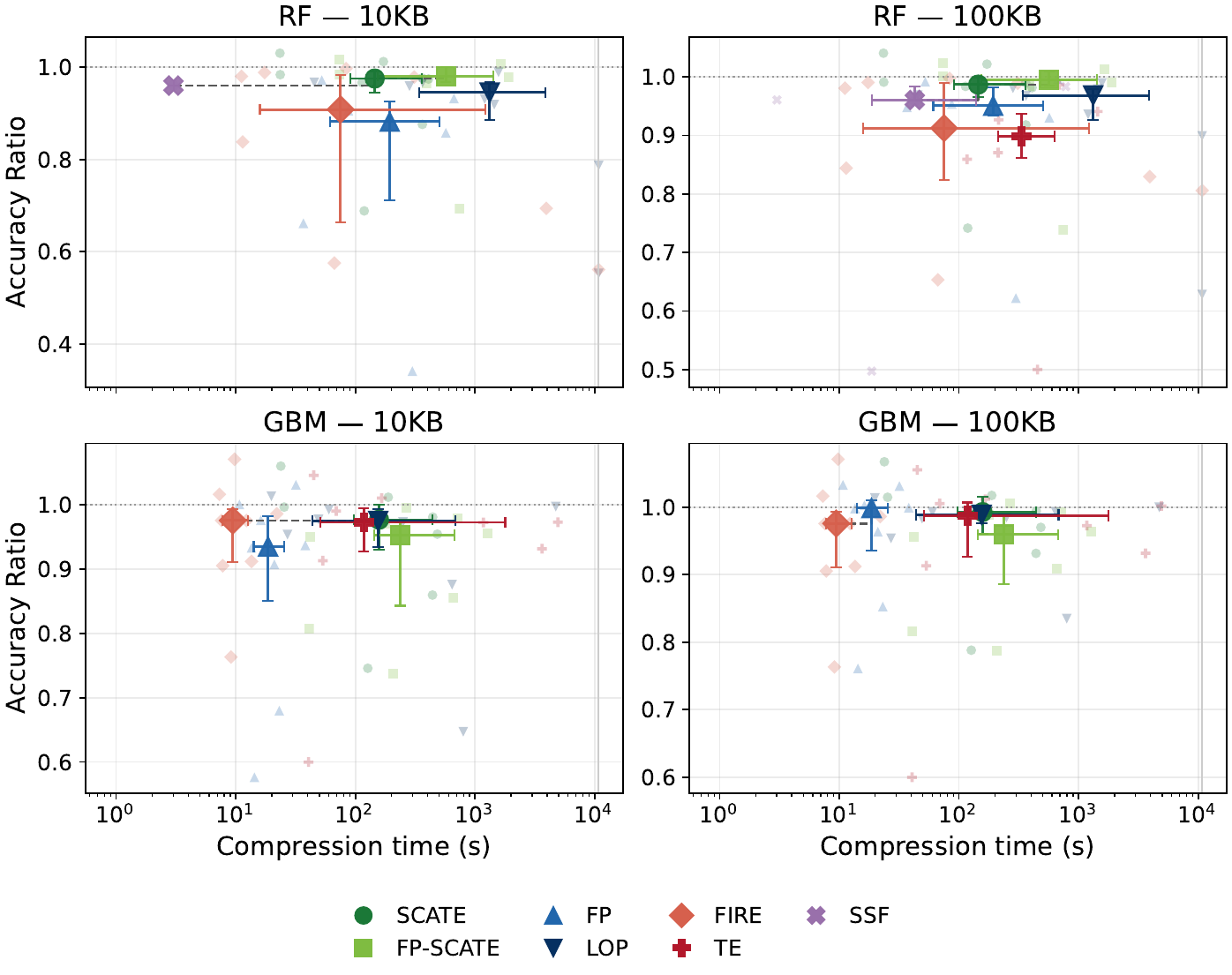}
    \caption{Compression-search time and normalized predictive performance  for RF and GBM ensembles.
    Large markers show medians across datasets, after averaging over available seeds within each dataset; error bars show interquartile ranges, and faint markers show dataset-level means. The horizontal dotted line indicates base model performance, and the vertical line indicates the time limit.}
    \label{fig:runtime_pareto}
\end{figure}

For RF compression, SCATE attains a high median normalized score, at 0.976 and 0.987 for the 10 and 100KB budgets, respectively. It lies on the frontier of the displayed method medians at both budgets, where no stronger method is also faster (and in general it is also the strongest method). For GBMs, the trade-off differs. While SCATE remains competitive in terms of relative strength, the FIRE and FP methods are similarly powerful but provide an advantage in terms of runtime. This reflects our earlier results and claims, where the smaller GBM model base size allows some methods to be competitive with SCATE, but not for RF, where we generally dominate.
\clearpage

%% file: checklist.tex
\section*{NeurIPS Paper Checklist}

\begin{enumerate}

\item {\bf Claims}
    \item[] Question: Do the main claims made in the abstract and introduction accurately reflect the paper's contributions and scope?
    \item[] Answer: \answerYes{}
    \item[] Justification: Our abstract claims that we derive minimax-optimal convergence for RF regression under certain regularity conditions, which we define formally in Sect. 4, and prove in Appx. A. We also claim to derive a compression method for RFs and GBMs (Sect. 5), and that it performs competitively on this task (Sect. 6, and Appx. B).
    \item[] Guidelines:
    \begin{itemize}
        \item The answer \answerNA{} means that the abstract and introduction do not include the claims made in the paper.
        \item The abstract and/or introduction should clearly state the claims made, including the contributions made in the paper and important assumptions and limitations. A \answerNo{} or \answerNA{} answer to this question will not be perceived well by the reviewers. 
        \item The claims made should match theoretical and experimental results, and reflect how much the results can be expected to generalize to other settings. 
        \item It is fine to include aspirational goals as motivation as long as it is clear that these goals are not attained by the paper. 
    \end{itemize}

\item {\bf Limitations}
    \item[] Question: Does the paper discuss the limitations of the work performed by the authors?
    \item[] Answer: \answerYes{}
    \item[] Justification: We discuss limitations at length in Sect. 7, as well as in the appendix, whenever it is relevant.
    \item[] Guidelines:
    \begin{itemize}
        \item The answer \answerNA{} means that the paper has no limitation while the answer \answerNo{} means that the paper has limitations, but those are not discussed in the paper. 
        \item The authors are encouraged to create a separate ``Limitations'' section in their paper.
        \item The paper should point out any strong assumptions and how robust the results are to violations of these assumptions (e.g., independence assumptions, noiseless settings, model well-specification, asymptotic approximations only holding locally). The authors should reflect on how these assumptions might be violated in practice and what the implications would be.
        \item The authors should reflect on the scope of the claims made, e.g., if the approach was only tested on a few datasets or with a few runs. In general, empirical results often depend on implicit assumptions, which should be articulated.
        \item The authors should reflect on the factors that influence the performance of the approach. For example, a facial recognition algorithm may perform poorly when image resolution is low or images are taken in low lighting. Or a speech-to-text system might not be used reliably to provide closed captions for online lectures because it fails to handle technical jargon.
        \item The authors should discuss the computational efficiency of the proposed algorithms and how they scale with dataset size.
        \item If applicable, the authors should discuss possible limitations of their approach to address problems of privacy and fairness.
        \item While the authors might fear that complete honesty about limitations might be used by reviewers as grounds for rejection, a worse outcome might be that reviewers discover limitations that aren't acknowledged in the paper. The authors should use their best judgment and recognize that individual actions in favor of transparency play an important role in developing norms that preserve the integrity of the community. Reviewers will be specifically instructed to not penalize honesty concerning limitations.
    \end{itemize}

\item {\bf Theory assumptions and proofs}
    \item[] Question: For each theoretical result, does the paper provide the full set of assumptions and a complete (and correct) proof?
    \item[] Answer: \answerYes{} 
    \item[] Justification: Each theorem and claim made in the main text (primarily Sect. 4) is proven in Appx. A. The assumptions that we make for these theorems are clearly stated, and some of them are evaluated in Appx. E. Theorems that the proof relies upon are referenced whenever appropriate.
    \item[] Guidelines:
    \begin{itemize}
        \item The answer \answerNA{} means that the paper does not include theoretical results. 
        \item All the theorems, formulas, and proofs in the paper should be numbered and cross-referenced.
        \item All assumptions should be clearly stated or referenced in the statement of any theorems.
        \item The proofs can either appear in the main paper or the supplemental material, but if they appear in the supplemental material, the authors are encouraged to provide a short proof sketch to provide intuition. 
        \item Inversely, any informal proof provided in the core of the paper should be complemented by formal proofs provided in appendix or supplemental material.
        \item Theorems and Lemmas that the proof relies upon should be properly referenced. 
    \end{itemize}

    \item {\bf Experimental result reproducibility}
    \item[] Question: Does the paper fully disclose all the information needed to reproduce the main experimental results of the paper to the extent that it affects the main claims and/or conclusions of the paper (regardless of whether the code and data are provided or not)?
    \item[] Answer: \answerYes{} 
    \item[] Justification: We include the model, code and data required to run our experiments within our supplemental material, as well as any necessary specifications to run them.
    \item[] Guidelines:
    \begin{itemize}
        \item The answer \answerNA{} means that the paper does not include experiments.
        \item If the paper includes experiments, a \answerNo{} answer to this question will not be perceived well by the reviewers: Making the paper reproducible is important, regardless of whether the code and data are provided or not.
        \item If the contribution is a dataset and\slash or model, the authors should describe the steps taken to make their results reproducible or verifiable. 
        \item Depending on the contribution, reproducibility can be accomplished in various ways. For example, if the contribution is a novel architecture, describing the architecture fully might suffice, or if the contribution is a specific model and empirical evaluation, it may be necessary to either make it possible for others to replicate the model with the same dataset, or provide access to the model. In general. releasing code and data is often one good way to accomplish this, but reproducibility can also be provided via detailed instructions for how to replicate the results, access to a hosted model (e.g., in the case of a large language model), releasing of a model checkpoint, or other means that are appropriate to the research performed.
        \item While NeurIPS does not require releasing code, the conference does require all submissions to provide some reasonable avenue for reproducibility, which may depend on the nature of the contribution. For example
        \begin{enumerate}
            \item If the contribution is primarily a new algorithm, the paper should make it clear how to reproduce that algorithm.
            \item If the contribution is primarily a new model architecture, the paper should describe the architecture clearly and fully.
            \item If the contribution is a new model (e.g., a large language model), then there should either be a way to access this model for reproducing the results or a way to reproduce the model (e.g., with an open-source dataset or instructions for how to construct the dataset).
            \item We recognize that reproducibility may be tricky in some cases, in which case authors are welcome to describe the particular way they provide for reproducibility. In the case of closed-source models, it may be that access to the model is limited in some way (e.g., to registered users), but it should be possible for other researchers to have some path to reproducing or verifying the results.
        \end{enumerate}
    \end{itemize}

\item {\bf Open access to data and code}
    \item[] Question: Does the paper provide open access to the data and code, with sufficient instructions to faithfully reproduce the main experimental results, as described in supplemental material?
    \item[] Answer: \answerYes 
    \item[] Justification: Our code, model and data is available via the supplementary material, and upon conclusion of the review process we intend to release the code publicly to facilitate further research.
    \item[] Guidelines:
    \begin{itemize}
        \item The answer \answerNA{} means that paper does not include experiments requiring code.
        \item Please see the NeurIPS code and data submission guidelines (\url{https://neurips.cc/public/guides/CodeSubmissionPolicy}) for more details.
        \item While we encourage the release of code and data, we understand that this might not be possible, so \answerNo{} is an acceptable answer. Papers cannot be rejected simply for not including code, unless this is central to the contribution (e.g., for a new open-source benchmark).
        \item The instructions should contain the exact command and environment needed to run to reproduce the results. See the NeurIPS code and data submission guidelines (\url{https://neurips.cc/public/guides/CodeSubmissionPolicy}) for more details.
        \item The authors should provide instructions on data access and preparation, including how to access the raw data, preprocessed data, intermediate data, and generated data, etc.
        \item The authors should provide scripts to reproduce all experimental results for the new proposed method and baselines. If only a subset of experiments are reproducible, they should state which ones are omitted from the script and why.
        \item At submission time, to preserve anonymity, the authors should release anonymized versions (if applicable).
        \item Providing as much information as possible in supplemental material (appended to the paper) is recommended, but including URLs to data and code is permitted.
    \end{itemize}

\item {\bf Experimental setting/details}
    \item[] Question: Does the paper specify all the training and test details (e.g., data splits, hyperparameters, how they were chosen, type of optimizer) necessary to understand the results?
    \item[] Answer: \answerYes{} 
    \item[] Justification: The experimental setup is summarized in the main text, and all substantial experimental design details are given in Appx. D, with results in Appx. E.
    \item[] Guidelines:
    \begin{itemize}
        \item The answer \answerNA{} means that the paper does not include experiments.
        \item The experimental setting should be presented in the core of the paper to a level of detail that is necessary to appreciate the results and make sense of them.
        \item The full details can be provided either with the code, in appendix, or as supplemental material.
    \end{itemize}

\item {\bf Experiment statistical significance}
    \item[] Question: Does the paper report error bars suitably and correctly defined or other appropriate information about the statistical significance of the experiments?
    \item[] Answer: \answerYes{} 
    \item[] Justification: Error bars are provided for all our experiments, and they were all run with repetition.
    \item[] Guidelines:
    \begin{itemize}
        \item The answer \answerNA{} means that the paper does not include experiments.
        \item The authors should answer \answerYes{} if the results are accompanied by error bars, confidence intervals, or statistical significance tests, at least for the experiments that support the main claims of the paper.
        \item The factors of variability that the error bars are capturing should be clearly stated (for example, train/test split, initialization, random drawing of some parameter, or overall run with given experimental conditions).
        \item The method for calculating the error bars should be explained (closed form formula, call to a library function, bootstrap, etc.)
        \item The assumptions made should be given (e.g., Normally distributed errors).
        \item It should be clear whether the error bar is the standard deviation or the standard error of the mean.
        \item It is OK to report 1-sigma error bars, but one should state it. The authors should preferably report a 2-sigma error bar than state that they have a 96\% CI, if the hypothesis of Normality of errors is not verified.
        \item For asymmetric distributions, the authors should be careful not to show in tables or figures symmetric error bars that would yield results that are out of range (e.g., negative error rates).
        \item If error bars are reported in tables or plots, the authors should explain in the text how they were calculated and reference the corresponding figures or tables in the text.
    \end{itemize}

\item {\bf Experiments compute resources}
    \item[] Question: For each experiment, does the paper provide sufficient information on the computer resources (type of compute workers, memory, time of execution) needed to reproduce the experiments?
    \item[] Answer: \answerYes{} 
    \item[] Justification: We specify some the computer resources that we use in our experiments in Appx. D. However, given the nature of the experiment and inconsistency in the amount of compute available to us at any given time, we are unable to report a `total' compute count. Some information relating to our compute setup is also hidden to protect anonymity; we intend to provide this information upon the paper's public release.
    \item[] Guidelines:
    \begin{itemize}
        \item The answer \answerNA{} means that the paper does not include experiments.
        \item The paper should indicate the type of compute workers CPU or GPU, internal cluster, or cloud provider, including relevant memory and storage.
        \item The paper should provide the amount of compute required for each of the individual experimental runs as well as estimate the total compute. 
        \item The paper should disclose whether the full research project required more compute than the experiments reported in the paper (e.g., preliminary or failed experiments that didn't make it into the paper). 
    \end{itemize}
    
\item {\bf Code of ethics}
    \item[] Question: Does the research conducted in the paper conform, in every respect, with the NeurIPS Code of Ethics \url{https://neurips.cc/public/EthicsGuidelines}?
    \item[] Answer: \answerYes{} 
    \item[] Justification: We have reviewed the NeurIPS Code of Ethics and conformed to it, in every respect.
    \item[] Guidelines:
    \begin{itemize}
        \item The answer \answerNA{} means that the authors have not reviewed the NeurIPS Code of Ethics.
        \item If the authors answer \answerNo, they should explain the special circumstances that require a deviation from the Code of Ethics.
        \item The authors should make sure to preserve anonymity (e.g., if there is a special consideration due to laws or regulations in their jurisdiction).
    \end{itemize}

\item {\bf Broader impacts}
    \item[] Question: Does the paper discuss both potential positive societal impacts and negative societal impacts of the work performed?
    \item[] Answer: \answerYes{} 
    \item[] Justification: Our contribution is foundational, meaning that there are no direct social impacts of our work. However, our paper partly revolves around knowledge distillation, which may carry potential positive impacts via optimizing resource consumption.
    \item[] Guidelines:
    \begin{itemize}
        \item The answer \answerNA{} means that there is no societal impact of the work performed.
        \item If the authors answer \answerNA{} or \answerNo, they should explain why their work has no societal impact or why the paper does not address societal impact.
        \item Examples of negative societal impacts include potential malicious or unintended uses (e.g., disinformation, generating fake profiles, surveillance), fairness considerations (e.g., deployment of technologies that could make decisions that unfairly impact specific groups), privacy considerations, and security considerations.
        \item The conference expects that many papers will be foundational research and not tied to particular applications, let alone deployments. However, if there is a direct path to any negative applications, the authors should point it out. For example, it is legitimate to point out that an improvement in the quality of generative models could be used to generate Deepfakes for disinformation. On the other hand, it is not needed to point out that a generic algorithm for optimizing neural networks could enable people to train models that generate Deepfakes faster.
        \item The authors should consider possible harms that could arise when the technology is being used as intended and functioning correctly, harms that could arise when the technology is being used as intended but gives incorrect results, and harms following from (intentional or unintentional) misuse of the technology.
        \item If there are negative societal impacts, the authors could also discuss possible mitigation strategies (e.g., gated release of models, providing defenses in addition to attacks, mechanisms for monitoring misuse, mechanisms to monitor how a system learns from feedback over time, improving the efficiency and accessibility of ML).
    \end{itemize}
    
\item {\bf Safeguards}
    \item[] Question: Does the paper describe safeguards that have been put in place for responsible release of data or models that have a high risk for misuse (e.g., pre-trained language models, image generators, or scraped datasets)?
    \item[] Answer: \answerNA{}{} 
    \item[] Justification: Our model makes use of publicly available data, models and packages, with a general purpose use case. Therefore, there is no high risk for misuse.
    \item[] Guidelines:
    \begin{itemize}
        \item The answer \answerNA{} means that the paper poses no such risks.
        \item Released models that have a high risk for misuse or dual-use should be released with necessary safeguards to allow for controlled use of the model, for example by requiring that users adhere to usage guidelines or restrictions to access the model or implementing safety filters. 
        \item Datasets that have been scraped from the Internet could pose safety risks. The authors should describe how they avoided releasing unsafe images.
        \item We recognize that providing effective safeguards is challenging, and many papers do not require this, but we encourage authors to take this into account and make a best faith effort.
    \end{itemize}

\item {\bf Licenses for existing assets}
    \item[] Question: Are the creators or original owners of assets (e.g., code, data, models), used in the paper, properly credited and are the license and terms of use explicitly mentioned and properly respected?
    \item[] Answer: \answerYes{}{} 
    \item[] Justification: All external methods and datasets used in, or are relevant to, our paper are appropriately referenced. Any modifications on these methods and datasets are also stated.
    \item[] Guidelines:
    \begin{itemize}
        \item The answer \answerNA{} means that the paper does not use existing assets.
        \item The authors should cite the original paper that produced the code package or dataset.
        \item The authors should state which version of the asset is used and, if possible, include a URL.
        \item The name of the license (e.g., CC-BY 4.0) should be included for each asset.
        \item For scraped data from a particular source (e.g., website), the copyright and terms of service of that source should be provided.
        \item If assets are released, the license, copyright information, and terms of use in the package should be provided. For popular datasets, \url{paperswithcode.com/datasets} has curated licenses for some datasets. Their licensing guide can help determine the license of a dataset.
        \item For existing datasets that are re-packaged, both the original license and the license of the derived asset (if it has changed) should be provided.
        \item If this information is not available online, the authors are encouraged to reach out to the asset's creators.
    \end{itemize}

\item {\bf New assets}
    \item[] Question: Are new assets introduced in the paper well documented and is the documentation provided alongside the assets?
    \item[] Answer: \answerYes{} 
    \item[] Justification: Our asset contribution is a codebase containing our model, the scripts we use for our experiments, and the datasets used. General documentation is provided via both comments in the codebase and in the paper itself.
    \item[] Guidelines:
    \begin{itemize}
        \item The answer \answerNA{} means that the paper does not release new assets.
        \item Researchers should communicate the details of the dataset\slash code\slash model as part of their submissions via structured templates. This includes details about training, license, limitations, etc. 
        \item The paper should discuss whether and how consent was obtained from people whose asset is used.
        \item At submission time, remember to anonymize your assets (if applicable). You can either create an anonymized URL or include an anonymized zip file.
    \end{itemize}

\item {\bf Crowdsourcing and research with human subjects}
    \item[] Question: For crowdsourcing experiments and research with human subjects, does the paper include the full text of instructions given to participants and screenshots, if applicable, as well as details about compensation (if any)? 
    \item[] Answer: \answerNA{} 
    \item[] Justification: Our paper does not involve crowdsourcing nor research with human subjects.
    \item[] Guidelines:
    \begin{itemize}
        \item The answer \answerNA{} means that the paper does not involve crowdsourcing nor research with human subjects.
        \item Including this information in the supplemental material is fine, but if the main contribution of the paper involves human subjects, then as much detail as possible should be included in the main paper. 
        \item According to the NeurIPS Code of Ethics, workers involved in data collection, curation, or other labor should be paid at least the minimum wage in the country of the data collector. 
    \end{itemize}

\item {\bf Institutional review board (IRB) approvals or equivalent for research with human subjects}
    \item[] Question: Does the paper describe potential risks incurred by study participants, whether such risks were disclosed to the subjects, and whether Institutional Review Board (IRB) approvals (or an equivalent approval/review based on the requirements of your country or institution) were obtained?
    \item[] Answer: \answerNA{} 
    \item[] Justification: Our paper does involve research with any human subjects.
    \item[] Guidelines:
    \begin{itemize}
        \item The answer \answerNA{} means that the paper does not involve crowdsourcing nor research with human subjects.
        \item Depending on the country in which research is conducted, IRB approval (or equivalent) may be required for any human subjects research. If you obtained IRB approval, you should clearly state this in the paper. 
        \item We recognize that the procedures for this may vary significantly between institutions and locations, and we expect authors to adhere to the NeurIPS Code of Ethics and the guidelines for their institution. 
        \item For initial submissions, do not include any information that would break anonymity (if applicable), such as the institution conducting the review.
    \end{itemize}

\item {\bf Declaration of LLM usage}
    \item[] Question: Does the paper describe the usage of LLMs if it is an important, original, or non-standard component of the core methods in this research? Note that if the LLM is used only for writing, editing, or formatting purposes and does \emph{not} impact the core methodology, scientific rigor, or originality of the research, declaration is not required.
    \item[] Answer: \answerNA{} 
    \item[] Justification: We did not use LLMs in the core method development stages of our paper.
    \item[] Guidelines:
    \begin{itemize}
        \item The answer \answerNA{} means that the core method development in this research does not involve LLMs as any important, original, or non-standard components.
        \item Please refer to our LLM policy in the NeurIPS handbook for what should or should not be described.
    \end{itemize}

\end{enumerate}